\documentclass[11pt]{article}
\usepackage[margin=1in]{geometry}
\usepackage{times,natbib}
\usepackage[T1]{fontenc}
\usepackage[utf8]{inputenc}
\usepackage{amsmath,amssymb,amsthm,mathtools,bm}
\usepackage{graphicx,booktabs,array,multirow,tabularx,longtable}
\usepackage[table]{xcolor}
\usepackage{microtype}
\usepackage{algorithm,algorithmic}
\usepackage{subcaption}
\usepackage{enumitem}
\usepackage{url}
\usepackage[colorlinks=true,linkcolor=teal!65!black,citecolor=teal!65!black,urlcolor=teal!65!black]{hyperref}
\usepackage[nameinlink,noabbrev]{cleveref}
\newtheorem{theorem}{Theorem}[section]
\newtheorem{proposition}[theorem]{Proposition}
\newtheorem{lemma}[theorem]{Lemma}
\newtheorem{corollary}[theorem]{Corollary}
\theoremstyle{definition}

\theoremstyle{remark}

\newcommand{\E}{\mathbb E}

\newcommand{\Var}{\operatorname{Var}}
\newcommand{\Cov}{\operatorname{Cov}}
\newcommand{\KL}{\operatorname{KL}}
\newcommand{\TV}{\operatorname{TV}}
\newcommand{\F}{\mathcal F}
\newcommand{\Pn}{\mathbb P_n}
\newcommand{\PP}{\mathcal P}
\newcommand{\ind}{\mathbf 1}
\newcommand{\defect}{\mathfrak D_n}
\newcommand{\risk}{\mathcal R}
\newcommand{\norm}[1]{\left\lVert #1\right\rVert}

\newcommand{\smallhead}[1]{\paragraph{#1}}

\renewcommand{\arraystretch}{1.12}
\setlist{nosep,leftmargin=*}
\newcolumntype{Y}{>{\raggedright\arraybackslash}X}

\title{Learning to Fluctuate: Statistical Foundations for Causal Tabular Pretraining}
\author{Zhiheng Zhang\\School of Statistics and Data Science\\Shanghai University of Finance and Economics}
\date{}
\hypersetup{
  pdftitle={Learning to Fluctuate: Statistical Foundations for Causal Tabular Pretraining},
  pdfauthor={Zhiheng Zhang},
  pdfsubject={Causal tabular pretraining},
  pdfkeywords={causal inference, tabular foundation models, influence functions, pretraining}
}
\begin{document}
\maketitle
\begin{abstract}
Causal tabular foundation models amortize effect estimation across synthetic mechanisms, but supervision by the mechanism-level effect $\theta(P)$---which we call \emph{latent-effect supervision}---rewards posterior shrinkage instead of directly encoding the repeated-sample response needed in a fixed deployment population. We introduce fluctuation-supervised pretraining (FSP): each synthetic table is labeled by its average treatment effect plus its efficient influence-function fluctuation, while deployment remains a single frozen forward pass. Along the path $T_{\lambda,P}=\theta(P)+\lambda P_n\psi_P$, we prove an endpoint transition: every fixed $\lambda<1$ retains label ambiguity of order $(1-\lambda)^2/n$, whereas full fluctuation makes the Gaussian label observable and reduces optimal finite-stratum causal label-prediction risk to order $n^{-2}$. One finite-pretraining bound combines label, network, episode-sampling, and optimization errors; its resulting sampling defect controls fixed-mechanism bias, mean squared error, variance, Gaussian approximation, and, with variance-head accuracy, studentized coverage. Complementary lower bounds separate the local $n^{-1}$ ATE risk that deployment observations cannot erase from the $\log N/M$ excess risk of a generic finite-dictionary episode-learning problem. Experiments trace the learned sampling response. Across 24 nonlinear continuous-covariate cells at trained context lengths, a continuous-row FSP model lowers checkpoint-mean macro RMSE by 7.0\% versus S-learner and wins all 12 weak-overlap cells; validation-selected Summary FSP deploys $11.6\times$ faster per table than S-learner in our warm one-thread benchmark. Under effect shift, matched Raw FSP lowers mean-checkpoint RMSE by 54.2\% and teacher defect by 99.0\% versus latent-effect supervision, and RMSE by 10.2\% versus the released CausalPFN-S checkpoint. Known-effect semisynthesis tests coverage; two randomized-study evaluations show that lower RMSE can coexist with residual attenuation.
\end{abstract}
\section{Introduction}
Causal tabular foundation models promise to amortize an analysis that is otherwise repeated cohort by cohort: pretrain on synthetic causal mechanisms, freeze the network, and estimate an intervention effect from a new table in one forward pass. The TabPFN paradigm established the data set, rather than an individual row, as the unit of in-context prediction \citep{muller2022,hollmann2025}; CausalPFN, Do-PFN, and CausalFM extend synthetic table pretraining to causal targets and identifying settings \citep{causalpfn,dopfn,causalfm}. The question is: \emph{Can causal foundation models amortize an efficient estimator's sampling law?}

Consider a hypothetical hospital where smoking lowers mean birthweight by \(150\) g. A model pretrained mostly on near-zero effects may systematically attenuate estimates from this population toward the prior center, for example around \(-75\) g rather than the true \(-150\) g. Such shrinkage can improve average prediction across synthetic populations. Yet repeated cohorts from this same hospital can produce estimates concentrated around the wrong effect; even an accurate standard error then leaves confidence intervals centered incorrectly. Good prediction over virtual worlds and correct inference in one real world are different objectives. The issue persists even with exact labels and exact population-risk optimization: squared-loss prediction of the mechanism-level effect $\theta(P)$, which we call \emph{latent-effect supervision}, rewards prior-conditioned prediction, whereas inference requires the correct sampling response at this fixed population (Figure~\ref{fig:intro-overview}).

\begin{figure}[t]
\centering
\includegraphics[width=.94\linewidth]{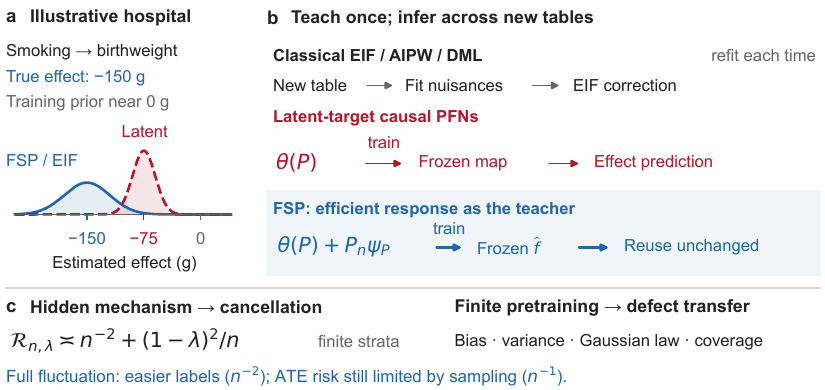}
\caption{\textbf{Teach the sampling response once, reuse it across new tables.} \textbf{Panel a} shows why latent-effect supervision can have the wrong repeated-sampling center under prior shift; \textbf{Panel b} shows that FSP moves EIF structure from per-dataset correction to reusable synthetic supervision; \textbf{Panel c} shows why the full-fluctuation endpoint is statistically special and how finite pretraining error propagates into inferential guarantees.}\label{fig:intro-overview}
\end{figure}

At this intersection, existing methods place the statistical work at three different stages. \textbf{The first} route pretrains a frozen model to infer causal effects or interventional objects directly from a new table \citep{causalpfn,dopfn,causalfm}. \textbf{A second} constructs an orthogonal or targeted estimate anew on each analyzed data set, through nuisance fitting, cross-fitting, or a targeted neural objective \citep{dml,dragonnet}. \textbf{A third} acts downstream of learned objects: MP-OSPC forms an EIF-based one-step ATE posterior from PFN-derived nuisance posteriors, while WALDO uses simulation-calibrated Neyman inversion to form confidence regions from a predictor or posterior estimator \citep{ospc,waldo}. \textbf{FSP trains a reusable inferential rule.} AIPW/DML typically constructs the nuisance-adjusted estimator afresh on each new data set \citep{dml}; FSP teaches its sampling response across synthetic tables and reuses the frozen map without per-table nuisance fitting or test-time correction. Our theory quantifies the statistical cost of that reuse through label ambiguity, network approximation, finite pretraining over $M$ episodes, and optimization error. Appendix Table~\ref{tab:routes} compares these representative, non-exclusive workflows.

We propose \textbf{fluctuation-supervised pretraining (FSP)}. For a synthetic mechanism $P$, a generated table $D_n$, its ATE $\theta(P)$, and efficient influence function $\psi_P$, define the supervision path
\begin{equation}\label{eq:opening-label}
 T_{\lambda,P}(D_n)=\theta(P)+\lambda\frac1n\sum_{i=1}^n\psi_P(O_i),
 \qquad \lambda\in[0,1].
\end{equation}
The conventional, table-invariant target $D_n\mapsto\theta(P)$ is latent-effect supervision, the $\lambda=0$ endpoint; FSP is the full-fluctuation endpoint $\lambda=1$. We use $\lambda$ as an analytical interpolation, not a deployment hyperparameter, to ask whether any fixed partial injection suffices, or whether full fluctuation is required to change the statistical order of the supervised problem. The simulator thus teaches the network both where the virtual world's causal effect lies and how an efficient estimator should move for this particular synthetic table. Simulator-known nuisances construct each EIF teacher; the student learns a table-to-estimate rule from observed inputs, and fixed teacher-nuisance errors enter its sampling defect (Proposition~\ref{prop:teachererror}). At deployment, the frozen rule needs no nuisance refitting or test-time correction. We call this uncorrected prediction the \emph{native frozen output}. A separate head learns the efficient variance coefficient $V(P)=\E_P\psi_P^2$, used as $V(P)/n$ for studentization; conditional uncertainty of the label is not the sampling uncertainty of the ATE.

Three results connect this target to inference. \textbf{First}, every fixed $\lambda<1$ retains label ambiguity proportional to $(1-\lambda)^2/n$ in Gaussian and nondegenerate finite-stratum submodels. Full fluctuation cancels this component: label-prediction risk becomes zero in the Gaussian experiment and has sharp order $n^{-2}$ in the finite-stratum model. \textbf{Second}, an oracle bound combines label approximation, network approximation, finite pretraining over \(M\) synthetic episodes, and optimization error into one native sampling defect. Our theorem transfers this defect to bias, mean squared error, sampling variance, and Gaussian approximation; variance-head accuracy additionally controls studentization. \textbf{Third}, complementary lower bounds distinguish deployment risk of order $n^{-1}$ from the $\Omega(\min\{1,\log N/M\})$ excess risk of a separate finite-dictionary episode-learning experiment. The latter matches the generic dictionary-size dependence in the upper bound, without asserting a causal-FSP pretraining necessity. A two-world construction further gives a $1/(nM)$ obstruction for partially fluctuated labels. Our contributions are summarized as follows:
\begin{itemize}
\item \textbf{Supervision for inference.} We turn efficient-influence-function theory into a synthetic-label design principle and show that full fluctuation is a singular endpoint: it removes the first-order label ambiguity that persists under every fixed partial fluctuation (Section~\ref{sec:results}).
\item \textbf{A reusable statistical procedure.} We quantify when finite pretraining yields a frozen model with repeated-sample guarantees across covered mechanisms, under explicit approximation, transfer, and variance conditions (Section~\ref{sec:results}).
\item \textbf{Two distinct roles for data.} We separate learning a reusable estimator from estimating a population's effect, clarifying which gains synthetic episodes provide and which still require deployment observations (Sections~\ref{sec:results}--\ref{sec:discussion}).
\end{itemize}

The experiments test this chain through controlled seven-point $\lambda$ comparisons, raw-table and continuous-covariate backbones, finite-dictionary scaling, known-effect semisynthesis, and two randomized studies. The matched-backbone test shows the predicted reversal: effect shift erases the prior-center point-risk advantage of latent-effect supervision (Figure~\ref{fig:paired-baseline-profiles}). A 48-cell nonlinear expansion and measured CPU frontier then show where amortization pays. Appendix~\ref{sec:related} positions FSP, Sections~\ref{sec:setup}--\ref{sec:results} develop the framework and theory, Section~\ref{sec:experiments} tests their predicted regimes, and Section~\ref{sec:discussion} develops the implications.

\section{Training experiment and deployed estimator}\label{sec:setup}
\smallhead{Causal model and estimand.}
Fix positive integers $K$ and $n$. Let $P$ be the observed-data law of $O=(X,A,Y)$, where $X\in\{1,\ldots,K\}$ is a pretreatment stratum, $A\in\{0,1\}$ is treatment, and $Y\in[0,1]$ is the outcome. A table $D_n=(O_i)_{i=1}^n$ consists of independent copies $O_i=(X_i,A_i,Y_i)$ of $O$, with law $P^n$. Write $\E_P$ for expectation of one observation and $\E_{P^n}$ for expectation over a table; use the same convention for variances. For a measurable real-valued row function $h$, define $\Pn h=n^{-1}\sum_{i=1}^n h(O_i)$. For $s\in\{1,\ldots,K\}$ and $a\in\{0,1\}$, set
\[
 p_s=P(X=s),\qquad e_s=P(A=1\mid X=s),\qquad
 m_{as}=\E_P(Y\mid A=a,X=s).
\]
Write $p=(p_s)_s$, $e=(e_s)_s$, and $m=(m_{as})_{a,s}$, suppressing dependence on $P$. Let $Y(a)$ be the potential outcome under treatment $a$; expectation involving potential outcomes refers to a compatible causal law with observed margin $P$. Under consistency, conditional exchangeability given $X$, and positive treatment probabilities, the average treatment effect is the observed-law functional
\begin{equation}\label{eq:target}
 \theta(P):=\sum_{s=1}^K p_s(m_{1s}-m_{0s})=\E\{Y(1)-Y(0)\}.
\end{equation}
For the finite-stratum analysis, fix $p_*\in(0,1/K]$ and $\epsilon\in(0,1/2]$. The class $\PP_{K,p_*,\epsilon}$ consists of observed laws with $p_s\ge p_*$ and $e_s\in[\epsilon,1-\epsilon]$ for every $s$; $\PP^{\rm bin}_{K,p_*,\epsilon}$ additionally requires $Y\in\{0,1\}$. Put $q_*=p_*\epsilon$ and $H=1+\epsilon^{-1}$. Then $P(X=s,A=a)\ge q_*$. The principal mean-label upper bounds allow bounded outcomes; the sharp finite-stratum lower bound, explicit variance-head construction, and uniform mechanism covering use the binary-outcome subclass.
Here $m_{aX}$ and $e_X$ are the preceding arrays evaluated at the observed stratum. Define the uncentered ATE score
\begin{equation}\label{eq:phi}
\phi_P(O)=m_{1X}-m_{0X}+\frac{A(Y-m_{1X})}{e_X}
-\frac{(1-A)(Y-m_{0X})}{1-e_X}.
\end{equation}
Lemma~\ref{lem:score} gives $|\phi_P|\le H$ under the stated bounds.
Its centered influence function, full-fluctuation label, and variance coefficient per observation are
\begin{equation}\label{eq:teacher}
 \psi_P(O)=\phi_P(O)-\theta(P),\qquad
 T_P(D_n)=\Pn\phi_P,\qquad V(P)=\E_P\{\psi_P(O)^2\}.
\end{equation}
The same realized table determines the network input and the label fluctuation: $T_P(D_n)=\theta(P)+\Pn\psi_P$. Lemma~\ref{lem:score} gives its exact repeated-sampling moments,
\[
 \E_{P^n}\{T_P(D_n)\}=\theta(P),\qquad
 \Var_{P^n}\{T_P(D_n)\}=V(P)/n.
\]
For $\lambda\in[0,1]$, define
\begin{equation}\label{eq:lambda-label}
 T_{\lambda,P}(D_n)=\theta(P)+\lambda\Pn\psi_P
 =(1-\lambda)\theta(P)+\lambda T_P(D_n).
\end{equation}
We refer to $T_{\lambda, P}\left(D_n\right)$ as the oracle \textbf{teacher label}. It may use simulator-known mechanism quantities during synthetic pretraining, but it is never evaluated at deployment. At the FSP endpoint, $T_P\left(D_n\right)=T_{1, P}\left(D_n\right)=P_n \phi_P$. Every oracle label on this path has expectation $\theta(P)$ and variance $\lambda^2V(P)/n$ at fixed $P$. FSP uses $\lambda=1$, latent-effect supervision uses $\lambda=0$, and $0<\lambda<1$ gives partial-fluctuation supervision. The simulator uses $(p,e,m)$ to construct the effect labels. For general bounded outcomes, its variance target also uses $\sigma_{as}^2:=\Var_P(Y\mid A=a,X=s)$ through~\eqref{eq:Vformula}; for binary outcomes, $\sigma_{as}^2=m_{as}(1-m_{as})$. The student network receives the observed table or deterministic features of it, while simulator quantities determine only supervision targets.

\smallhead{Pretraining and frozen deployment.}
Let $\Pi$ be a distribution over $\PP_{K,p_*,\epsilon}$. Pretraining uses $M\ge1$ independent episodes: for $j=1,\ldots,M$, draw $P_j\sim\Pi$ and then $D_n^{(j)}\mid P_j\sim P_j^n$. Thus $n$ counts rows per table and $M$ counts independently generated training tables. Let $\F$ be a nonempty class, fixed before training, of measurable table-to-scalar rules $f=f_\omega$ with values in $[-H,H]$, where $\omega$ denotes mean-map weights. For each fixed $\lambda$, the empirical squared-label loss is
\begin{equation}\label{eq:loss}
 \widehat\risk_{\lambda,\Pi}(f)=\frac1M\sum_{j=1}^M
 \{f(D_n^{(j)})-T_{\lambda,P_j}(D_n^{(j)})\}^2.
\end{equation}
For FSP, write $\widehat f=f_{\widehat\omega_M}\in\F$ for the mean map fitted at $\lambda=1$. Its achieved optimization tolerance $\eta_M\ge0$ satisfies
\[
 \widehat\risk_{1,\Pi}(\widehat f)
 \le\inf_{f\in\F}\widehat\risk_{1,\Pi}(f)+\eta_M.
\]
A separate positive map $\widehat v$ learns $V(P)$, the coefficient of fixed-mechanism sampling variance $V(P)/n$, which differs from conditional label spread across simulator mechanisms. The complete map includes any shared table representation. In the implementation, its loss is detached from the mean backbone; Appendix~\ref{app:experiments} specifies the log-variance loss and target floor. We suppress the fixed $n$ in $\widehat\omega_M$ and the $(M,n)$ dependence of the fitted-map notation.

Freeze both maps and evaluate them on a fresh $D_n\sim P^n$, independent of training, at a fixed deployment law $P$. Fix a nominal error probability $\alpha\in(0,1)$, and let $z_{1-\alpha/2}$ be the $(1-\alpha/2)$ standard-normal quantile. The deployed outputs and nominal Wald interval are
\begin{equation}\label{eq:deployed}
 \begin{gathered}
 C_{M,n}=\left[
 \widehat\theta_{M,n}-z_{1-\alpha/2}\sqrt{\widehat V_{M,n}/n},\;
 \widehat\theta_{M,n}+z_{1-\alpha/2}\sqrt{\widehat V_{M,n}/n}
 \right],
 \end{gathered}
\end{equation}
$\text{where}~ \widehat\theta_{M,n}=\widehat f(D_n),~
 \widehat V_{M,n}=\widehat v(D_n)>0.$ These outputs use the table and frozen weights. Theorem~\ref{thm:main} establishes repeated-sampling coverage under its mechanism-transfer, mean-approximation, variance-accuracy, and nondegeneracy conditions.

\smallhead{Probability levels and the sampling defect.}
Deployment expectations condition on the fitted maps and hold $P$ fixed while new tables vary under $P^n$. Averaging over an independent mechanism $P\sim\Pi$ defines a distinct task-average risk. For a fixed measurable table rule $f\in\F$, define
\begin{equation}\label{eq:defect}
 \begin{aligned}
 \defect(P;f)&=n\E_{P^n}\bigl[\{f(D_n)-T_P(D_n)\}^2\bigr],~
 \risk_\Pi(f)&=\E_{P\sim\Pi}\E_{P^n}\bigl[\{f(D_n)-T_P(D_n)\}^2\bigr].
 \end{aligned}
\end{equation}
Thus $\E_{P\sim\Pi}\defect(P;f)=n\risk_\Pi(f)$; the mechanism-transfer conditions in Theorem~\ref{thm:main} connect this average to fixed-$P$ control. The factor $n$ compares the learned error with the $n^{-1/2}$ sampling scale:
\[
 \sqrt n\{f(D_n)-\theta(P)\}
 =\frac1{\sqrt n}\sum_{i=1}^n\psi_P(O_i)
 +\sqrt n\{f(D_n)-T_P(D_n)\}.
\]
The second term has second moment $\defect(P;f)$; a vanishing defect makes the learned remainder negligible in mean square at the first-order inferential scale. Appendix~\ref{app:scope} gives assumption diagnostics; Appendix~\ref{app:variance} analyzes variance-head learning on its stated subclass.

\section{From synthetic labels to sampling laws}\label{sec:results}
This section establishes the chain from synthetic supervision to native inference.
Proposition~\ref{prop:gauss} and Theorem~\ref{thm:learnability} show why full fluctuation is statistically special, changing label-prediction error from first to second order.
Theorem~\ref{thm:main} then connects finite pretraining error to fixed-mechanism bias, variance, Gaussian approximation, and coverage. Finally, Theorems~\ref{thm:meta-lower}--\ref{thm:minimax} separate the information limits of pretraining from those of the fresh deployment sample. Together, these results show what inferential structure can be amortized by pretraining and what population information must still come from the new table.

\subsection{Full fluctuation is a singular supervision endpoint (Figure~\ref{fig:theory-map}-b)}
The Gaussian experiment isolates supervision from the bounded causal class; its predictors range over $L_2(Z)$. At fixed $\theta$, bias and MSE refer to repeated draws of $Z\mid\theta$.
\begin{proposition}[Exact Gaussian $\lambda$-path: an intuitive example]\label{prop:gauss}
Let $Z\mid\theta\sim N(\theta,v/n)$, $\theta\sim N(0,s^2)$ for fixed $s^2,v>0$, and $T_\lambda=(1-\lambda)\theta+\lambda Z$. With $\kappa_n=ns^2/(ns^2+v)$ and $a_{\lambda,n}=\lambda+(1-\lambda)\kappa_n$, the minimizer of population squared label loss, unique up to almost-sure equality, is $f_\lambda^*(Z)=\E(T_\lambda\mid Z)=a_{\lambda,n}Z$, and
$\Var(T_\lambda\mid Z)=(1-\lambda)^2\frac{s^2v}{ns^2+v}$,~$
 \operatorname{Bias}_\theta(f_\lambda^*)=-(1-\lambda)(1-\kappa_n)\theta$,~$
 \Var_\theta(f_\lambda^*)=a_{\lambda,n}^2v/n$,~$
 \operatorname{MSE}_\theta(f_\lambda^*)=(1-a_{\lambda,n})^2\theta^2+a_{\lambda,n}^2v/n.$

Hence, every fixed $\lambda<1$ leaves $\Theta(n^{-1})$ label ambiguity; $\lambda=1$ makes the label observable and removes shrinkage while retaining sampling variance $v/n$.
\end{proposition}

Thus the observable label at $\lambda=1$ has zero conditional prediction variance but fixed-$\theta$ sampling variance $v/n$. For fixed $\lambda<1$, shrinkage vanishes asymptotically; the singularity concerns label-prediction order. For finite strata, let $N_s,N_{as}$ be stratum and treatment-cell counts and $Z_{as}$ their outcome sums. The observable comparator is
\[
 S_n(D_n)=\sum_{s=1}^K\frac{N_s}{n}\left\{
 \frac{Z_{1s}}{\max(N_{1s},nq_*/2)}-
 \frac{Z_{0s}}{\max(N_{0s},nq_*/2)}\right\}.
\]
Gaussian observability is illustrative; with hidden finite-stratum nuisances, the next theorem establishes second-order label learnability from observed tables (Figure~\ref{fig:theory-map}-b):
\begin{theorem}[Finite-stratum supervision phase transition]\label{thm:learnability}
Uniformly over $\PP_{K,p_*,\epsilon}$, $S_n$ satisfies
\begin{equation}\label{eq:comparator}
 \E_{P^n}(S_n-T_P)^2\le A_n := (n^2p_*\epsilon^2)^{-1}+2K(H+1)^2e^{-nq_*/8} = O(n^{-2}).
\end{equation}
For $\mathcal R_{n,\lambda}:=\inf_g\sup_{P\in\PP_{K,p_*,\epsilon}}\E_{P^n}\{g(D_n)-T_{\lambda,P}(D_n)\}^2$, where $g$ observes only the table,
\begin{equation}\label{eq:lambda-upper}
 \mathcal R_{n,\lambda}\le\{2\lambda^2+4(1-\lambda)^2\}A_n
 +16H^2(1-\lambda)^2/n.
\end{equation}
On a binary submodel with the same law in every stratum and treated count $N_{\rm tr}\sim\mathrm{Bin}(n,1/2)$, set $r_{n,\lambda}=\E_{N_{\rm tr}}\{(1-2\lambda N_{\rm tr}/n)^2/[6(N_{\rm tr}+2)]\}$,
then
\begin{equation}\label{eq:label-lower}
 \mathcal R_{n,\lambda}\ge\max\!\left\{\frac{2\lambda^2}{3n(n+2)},r_{n,\lambda}\right\},\quad
 r_{n,\lambda}\ge\frac{(1-\lambda)^2+\lambda^2/n}{6(n+2)},\quad
 nr_{n,\lambda}\to\frac{(1-\lambda)^2}{3}.
\end{equation}
Thus $\mathcal R_{n,\lambda}\asymp n^{-2}+(1-\lambda)^2/n$; the second-order window is $1-\lambda_n=O(n^{-1/2})$. The limit for $nr_{n,\lambda}$ holds at fixed $\lambda$.
\end{theorem}

Allocation imbalance multiplies centered cell-mean error, while rare-cell failures have exponentially small cost (Appendix~\ref{app:learnability}; Figure~\ref{fig:theory-map}). Deployment ATE risk remains $n^{-1}$.

\begin{figure}[t]
\centering
\includegraphics[width=.94\linewidth]{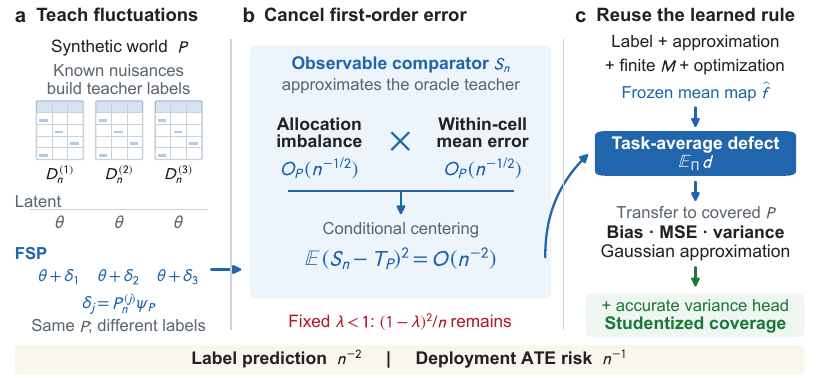}
\caption{\textbf{Second-order label learnability supports first-order inference.} (a) FSP labels follow each realized synthetic table. (b) Allocation imbalance multiplies centered cell-mean error; Theorem~\ref{thm:learnability} also controls rare cells. (c) Training controls the task average of $d(P)=\defect(P;\widehat f)$; mechanism transfer and variance-head accuracy yield fixed-population inference. This schematic fixes $K,p_*,\epsilon$ and distinguishes label-prediction risk from deployment ATE risk.}\label{fig:theory-map}
\end{figure}

\subsection{Finite pretraining implies native inference (Figure~\ref{fig:theory-map}-c)}
To turn this comparator into inference, finite pretraining must make the defect vanish (Figure~\ref{fig:theory-map}-c). Write $N_\xi=\mathcal N_\infty(\xi,\F)$ for the uniform $\xi$-cover of $\F$ over observed tables (Appendix~\ref{app:training}).
\begin{theorem}[Finite synthetic pretraining to causal inference]\label{thm:main}\label{thm:inference}
Suppose $\Pi(\PP_{K,p_*,\epsilon})=1$, $\F$ maps tables to $[-H,H]$, $N_\xi<\infty$, some $f_0\in\F$ has $\|f_0-S_n\|_\infty\le a_n$, and $\widehat f$ minimizes~\eqref{eq:loss} at $\lambda=1$ within $\eta_M$. Fix $\delta\in(0,1)$.

\emph{(i) Training.} With probability at least $1-\delta$,
\begin{equation}\label{eq:mainbound}
 \begin{aligned}
 \E_\Pi\defect(P;\widehat f)\le B_{M,n}(\xi):={}
 \underbrace{6nA_n}_{\rm label}
 +\underbrace{6na_n^2+24nH\xi}_{\rm approx./covering}
  +\underbrace{32nH^2\log(2N_\xi/\delta)/M}_{\rm pretraining}
 +\underbrace{2n\eta_M}_{\rm optimization}.
 \end{aligned}
\end{equation}
\emph{(ii) Transfer.} On this event, for each $\gamma\in(0,1)$ and task law $\Pi'\ll\Pi$ with $d\Pi'/d\Pi\le C_{\rm sh}$,
\[
 \Pi\{P:\defect(P;\widehat f)\le B_{M,n}(\xi)/\gamma\}\ge1-\gamma,
 \quad
 \Pi'\{P:\defect(P;\widehat f)\le C_{\rm sh}B_{M,n}(\xi)/\gamma\}\ge1-\gamma
\]
and an atom $P^\circ$ of mass $\pi^\circ>0$ satisfies $\defect(P^\circ;\widehat f)\le B_{M,n}(\xi)/\pi^\circ$.

\emph{(iii) Native sampling.} Fix $P$ and the trained maps; put $d=\defect(P;\widehat f)$, $\theta=\theta(P)$, $V=V(P)$, $\widehat\theta=\widehat f(D_n)$, and $\widehat V=\widehat v(D_n)$. Then
\begin{align}
 |\E_{P^n}\widehat\theta-\theta|\le\sqrt{d/n},~
 \left|\sqrt{n\E_{P^n}(\widehat\theta-\theta)^2}-\sqrt V\right|&\le\sqrt d,~
 \left|\sqrt{n\Var_{P^n}(\widehat\theta)}-\sqrt V\right|\le\sqrt d.
\end{align}
\emph{(iv) Distribution and coverage.} Let $d_{\rm K}$ be Kolmogorov distance to $N(0,1)$, $\rho(P)=\E_P|\psi_P|^3/V^{3/2}$, and $C_{\rm BE}$ a universal Berry--Esseen constant. If $V>0$, for $t>0$,
\begin{equation}\label{eq:be}
 d_{\rm K}\!\left(\frac{\sqrt n(\widehat\theta-\theta)}{\sqrt V},N(0,1)\right)
 \le C_{\rm BE}\rho(P)/\sqrt n+t/\sqrt{2\pi}+d/(Vt^2).
\end{equation}
If $\widehat V>0$ almost surely, set $e_V=\E_{P^n}(\widehat V/V-1)^2$. For $r\in(0,1/2)$,
\begin{equation}\label{eq:studentized-decomp}
 d_{\rm K}\!\left(\frac{\sqrt n(\widehat\theta-\theta)}{\sqrt{\widehat V}},N(0,1)\right)
 \le C_{\rm BE}\rho(P)/\sqrt n+t/\sqrt{2\pi}+d/(Vt^2)+r+e_V/r^2.
\end{equation}
Twice this bound controls the coverage error of~\eqref{eq:deployed}, without assuming independence of the heads.
\end{theorem}
The balanced-panel alternative in Appendix~\ref{app:uniform} gives uniform binary-outcome control with $M=Jm$ episodes, distinct from pooled training. For a Transformer with $W$ parameters in $[-B,B]^W$ and $\sup_D|f_w(D)-f_{w'}(D)|\le L_{\rm tr}\|w-w'\|_\infty$, we achieve $\log N_\xi\le W\log(1+2BL_{\rm tr}/\xi)$.
The sufficient schedule $a_n=o(n^{-1/2})$, $n\xi_n\to0$, $n\eta_{M_n}\to0$ in probability, and $M_n\gg n\{\log N_{\xi_n}+\log(1/\delta_n)\}$ makes~\eqref{eq:mainbound} vanish in probability. With $V$ bounded away from zero, bounded $\rho(P)$, and $e_V\to0$, it yields native coverage on transferred mechanisms. Architecture, variance, and unconditional-training conditions are in Appendices~\ref{app:architecture}--\ref{app:variance}. For binary outcomes, Corollary~\ref{cor:schedule} gives an explicit quantized-attention construction with $M=\lceil n^{3/2}\rceil$; Corollary~\ref{cor:deviation} separates sampling noise from learned error at finite $n$.

\subsection{Two information budgets}

Two pretraining lower bounds and one deployment lower bound separate what synthetic episodes can teach from what only a fresh deployment sample can reveal. The first two isolate finite-$M$ limits from supervision geometry and generic episode-learning complexity, while the third shows that even perfect pretraining cannot remove the $n^{-1}$ information limit of estimating a new population's effect.

\begin{theorem}[Pretraining lower bound under partial fluctuation]\label{thm:meta-lower}
For every fixed $\lambda<1$, the two Gaussian episode laws $Q_{\lambda,\sigma}$ of $(Z,T_\lambda)$ constructed in Appendix~\ref{app:meta-lower}, $\sigma\in\{-1,1\}$, share $P_Z=N(0,1/n)$. A learner $\mathcal A$ maps $M$ independent episodes to $\widehat f_{\mathcal A}$. Then
\begin{equation}\label{eq:meta-lower}
 \inf_{\mathcal A}\sup_{\sigma\in\{-1,1\}}\E_{Q_{\lambda,\sigma}^M}
 \{\risk_{\lambda,\sigma}(\widehat f_{\mathcal A})-\inf_{f\in L_2(P_Z)}\risk_{\lambda,\sigma}(f)\}
 \ge (1-\lambda)^2/(100nM).
\end{equation}
Here $\risk_{\lambda,\sigma}(f)=\E_{Q_{\lambda,\sigma}}\{f(Z)-T_\lambda\}^2$; the outer expectation integrates training randomness. At $\lambda=1$, both optimal maps equal $Z$.
\end{theorem}

Theorem~\ref{thm:meta-lower} isolates the geometry-specific cost of partial fluctuation; we next separate this effect from the generic complexity cost of learning among many candidate episode predictors.

\begin{theorem}[Finite-dictionary pretraining lower bound]\label{thm:dictionary-lower}
For $N\ge2$, $M\ge1$, and $d_N=\lfloor\log_2N\rfloor$, there are a dictionary $\F_N$ of at most $N$ bounded episode predictors and laws $\{Q_\sigma:\sigma\in\{-1,1\}^{d_N}\}$ for $(J,L)$ such that every possibly randomized learner trained on $M$ episodes obeys
\begin{equation}\label{eq:dictionary-lower}
 \inf_{\mathcal A}\sup_{\sigma\in\{-1,1\}^{d_N}}
 \E_{Q_\sigma^M}\!\left[R_\sigma(\widehat f_{\mathcal A})-\inf_fR_\sigma(f)\right]
 \ge\frac1{32}\min\!\left\{1,\frac{d_N}{M}\right\}.
\end{equation}
Here $R_\sigma(f)=\E_{Q_\sigma}\{f(J)-L\}^2$ and the inner infimum ranges over measurable predictors. This generic $\min\{1,\log N/M\}$ episode rate matches the dictionary-size dependence in~\eqref{eq:mainbound}; causal sampling defect additionally depends on teacher geometry.
\end{theorem}

The first two results delimit what finite synthetic pretraining can learn; we now ask what uncertainty remains even if the reusable rule were learned perfectly.

\begin{theorem}[Deployment-sample lower bound]\label{thm:minimax}
For every $n\ge1$, even with auxiliary pretraining randomness $U\perp D_n$ whose law is identical under every deployment mechanism,
\begin{equation}\label{eq:ate-lower}
 \inf_{\widetilde\theta(U,D_n)}\sup_{P\in\PP_{K,p_*,\epsilon}}
 \E_{P^n,U}(\widetilde\theta-\theta(P))^2\ge3/(512n).
\end{equation}
The infimum includes randomized estimators. On the least-favorable path $(P_t)$ through any interior binary-outcome $P$ with $V(P)>0$ (Appendix~\ref{app:minimax}),
\[
 \lim_{c\to\infty}\liminf_{n\to\infty}\inf_{\widetilde\theta_n(U_n,D_n)}
 \sup_{|u|\le c}n\E_{P_{u/\sqrt n}^n,U_n}(\widetilde\theta_n-\theta(P_{u/\sqrt n}))^2\ge V(P).
\]
Here $U_n\perp D_n$ has a law independent of $u$. Under Theorem~\ref{thm:panels}'s coverage schedule, a panel estimator chosen independently of $c$ satisfies, for every fixed $c<\infty$,
\[
 \sup_{|u|\le c}\left|n\E_{P_{u/\sqrt n}^n,\mathrm{train}}
 (\widehat f_n(D_n)-\theta(P_{u/\sqrt n}))^2-V(P)\right|\longrightarrow0,
\]
attaining the local bound.
\end{theorem}

Together, these lower bounds separate the cost of learning the reusable rule from the irreducible cost of learning the deployment population itself. Pretraining learns the rule; the fresh table supplies population information. Continuous confounders, uniform panels, and approximate simulator nuisances are treated in Theorem~\ref{thm:continuous}, Theorem~\ref{thm:panels}, and Proposition~\ref{prop:teachererror}.

\section{Experiments: from supervision to sampling behavior}\label{sec:experiments}
We evaluate FSP on common deployment tables against matched latent-effect-supervision ablations, the released CausalPFN-S checkpoint~\citep{causalpfn}, S/T/X/DR learners~\citep{kunzel2019,kennedy2023dr}, cross-fitted AIPW/DML~\citep{dml}, and an out-of-domain Do-PFN check~\citep{dopfn}. \emph{Raw latent-effect} denotes the Raw-FSP network trained instead on $T_{0,P}(D_n)=\theta(P)$, with the same architecture, $M=32768$ training tables, optimizer and five seeds. Each route otherwise enters in its native form: CausalPFN-S contributes its official checkpoint and classical estimators refit per table. Thus cross-method figures measure end-to-end quality and deployment cost, while the matched raw pair isolates the supervision target. Appendix~\ref{app:experiments} gives protocols and the claim-to-evidence map.

\smallhead{The label is the intervention.}
Holding generator, network, optimizer and budget fixed, the seven-point $\lambda$ sweep at the typical mechanism and $n=256$ increased the sampling-response coefficient from $0.105$ to $0.895$ and, across five $M=8192$ checkpoints, reduced mean full-label defect \textbf{$8.72$-fold}; the $\lambda=1$ checkpoints attained mean native coverage $.952$ (Appendix~\ref{app:lambda-experiment}; Figure~\ref{fig:lambda-path}). The Gaussian panels locate the mechanism: shrinkage improves point risk near the prior center but attenuates shifted effects. As predicted, \textbf{effect shift most strongly amplifies the FSP--latent-effect contrast}: on the matched raw backbone, it reverses the prior-center RMSE ordering, with FSP lowering shifted RMSE by $54.2\%$ and teacher defect by $99.0\%$, while its response slope remains $.972$ rather than $-.058$ (Figure~\ref{fig:paired-baseline-profiles}). Appendix Figure~\ref{fig:lambda-scaling} reports risk, bias and variance across the full path.

\begin{figure}[t]\centering
\includegraphics[width=.94\linewidth]{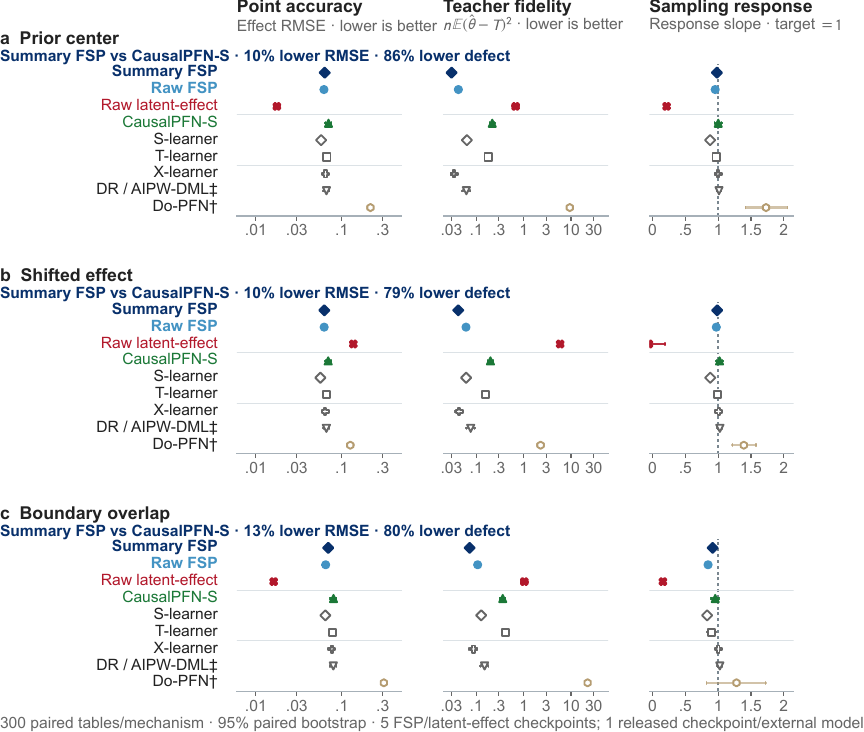}
\caption{\textbf{Point accuracy and sampling-law fidelity separate.} Within each panel, all methods use the same 300 tables at $n=256$. Blue headlines compare Summary FSP with CausalPFN-S throughout. Raw FSP and Raw latent-effect share architecture, $M=32768$, table streams, optimizer and five seeds, differing only in supervision. The first two axes are logarithmic and lower is better; the dashed response target is one. Classical estimators refit per table; Do-PFN is an out-of-domain check. Appendix~\ref{app:experiments} gives protocols and paired intervals.}\label{fig:paired-baseline-profiles}
\end{figure}

\smallhead{The amortized frontier.}
The four-stratum problem is deliberately well specified for per-table fitting: S-learner has lower point RMSE than FSP, while Summary FSP lowers its teacher defect by $32$--$53\%$ and moves the response closer to one. Against CausalPFN-S, Summary FSP lowers both RMSE ($10$--$13\%$) and defect ($79$--$86\%$) in every mechanism (Figure~\ref{fig:paired-baseline-profiles}). Its deployment gain comes from moving estimator construction offline: four lossless stratum tokens feed one fixed network, whereas classical routes refit table-specific models and the CausalPFN-S pipeline refits a gradient-boosting weak learner per table. Thus $Q$ deployments cost $C_{\mathrm{load}}+Q C_{\mathrm{fwd}}$ for FSP rather than $Q C_{\mathrm{fit}}$. Summary FSP required $.138$ ms per table on one CPU thread---\textbf{$11.6\times$ faster} than S-learner and \textbf{$1{,}491\times$ faster} than CausalPFN-S (Appendix Figure~\ref{fig:main-mechanism}). On the harder nonlinear continuous grid, FSP reduced macro RMSE by \textbf{$7.0\%$} versus S-learner across 24 trained-length cells, winning \textbf{16 cells} and all 12 weak-overlap cells (\textbf{$14.6\%$}; Appendix Figure~\ref{fig:amortization-evidence}).

\smallhead{Coverage needs both centering and scale.}
At the typical mechanism and $n=256$, five-seed $M=8192$ summary/raw FSP had native coverage $.952/.953$ but Kolmogorov distances $.081/.352$: \textbf{similar coverage concealed different distributional shapes}. Under weak overlap, raw FSP's mean $\widehat V/V$ was $.328$ and native/oracle-$V$ coverage was $.706/.998$, isolating scale error. Appendix Figure~\ref{fig:variance-full} and Table~\ref{tab:inference-diagnostics} report the joint bias, scale and shape diagnostics for Theorem~\ref{thm:main}.

\begin{figure}[t]\centering
\includegraphics[width=.94\linewidth]{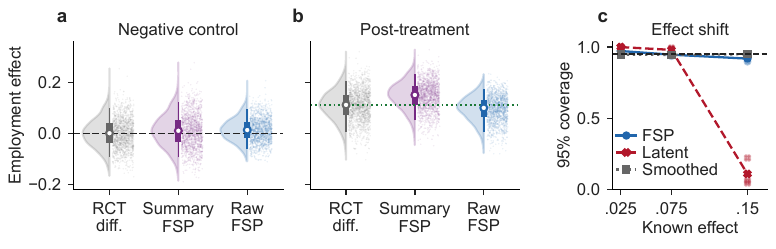}
\caption{\textbf{Frozen transfer trades variance against residual attenuation.} (a--b) National Supported Work at $n=256$: 1,500 re-randomized-null and post-treatment estimates; white marks are medians and thick/thin bars are interquartile/5th--95th percentiles. (c) Known-effect semisynthesis at $n=256$. Appendix~\ref{app:experiments} gives checkpoint budgets and protocols.}\label{fig:main-validation}
\end{figure}

\smallhead{Real populations and known-effect transfer.}
For National Supported Work~\citep{lalonde1986,dehejia1999}, the raw-FSP five-checkpoint ($M=8192$) ensemble reduced post-treatment RMSE from $.0596$ to $.0522$ relative to the arm difference, and pre-treatment null RMSE from $.0570$ to $.0479$ (Figure~\ref{fig:main-validation}). Its post-treatment mean was $.0957$ against the full-sample benchmark $.1106$. At $n=256$ on the independent social-pressure turnout trial~\citep{gerber2008}, five-checkpoint $M=8192$ summary/raw ensembles achieved RMSE $.0571/.0628$ versus $.0775$ for the arm difference, with biases $-.0077/-.0241$. Against these empirical trial contrasts, \textbf{lower ensemble RMSE coexisted with residual attenuation}. Known-effect semisynthesis supplies a complementary coverage test: as the effect increased from $.025$ to $.15$, three-seed $M=32768$ summary-FSP coverage changed from $.972$ to $.919$, while coverage under latent-effect supervision fell from $1.000$ to $.112$.

\smallhead{Ablations, sensitivity and the two data budgets.}
Target and shifted-teacher ablations isolate supervision; native/oracle variance replacement isolates scale learning. Sweeps over $M$, $n$, effect shift, overlap, dimension and smoothness map operating regimes; the $M\times n$ grid and 105,000 Bernoulli-family repetitions probe the two-budget scaling. These checks keep \textbf{training and deployment budgets empirically distinct}. Appendix~\ref{app:audit} gives the proof-to-code map.

\section{Discussion}\label{sec:discussion}
Full fluctuation removes first-order label ambiguity; finite-pretraining bounds quantify when frozen reuse preserves it despite the $n^{-1}$ deployment burden. Effect shift exposes prior-centered shrinkage (Figure~\ref{fig:paired-baseline-profiles}); continuous FSP improves trained weak-overlap risk; frozen reuse advances the accuracy--cost frontier. Together, the synthetic and randomized evidence separates point RMSE from sampling fidelity, makes centering, variation, shape, cost and accuracy joint design criteria, and opens a route from amortized ATE inference to other regular causal targets.

\clearpage
\section*{Reproducibility statement}
The supplement contains complete proofs; simulator, architecture, optimizer, and seed specifications; a compact numerical replay package with inventoried replicate outputs and checkpoints; data provenance; environment locks; executable analysis scripts; and a theorem--evidence map. Released CausalPFN inference is pinned to a repository commit and checkpoint digest, and matched-table identity is audited by reconstructed raw-array hashes. The paired-baseline, randomized-validation, and deployment-cost figures have compact source manifests plus aligned prediction, timing, and render audits. The README distinguishes stored-result replay from regeneration of the full training suite. Lean~4.32 source, a statement-level map for all 21 numbered results, and a kernel-verification record are described in Appendix~\ref{app:audit}.

\section*{AI use statement}
Generative AI systems assisted with conceptual brainstorming, literature discovery, mathematical claim formulation, proof drafting and algebraic checking, synthetic-data and experimental-design iteration, code implementation and debugging, experiment orchestration, result interpretation, figure generation, and manuscript drafting and editing. The authors independently checked every citation against its primary source, reviewed every theorem statement and proof, reran the reported computations from versioned scripts and stored data, inspected the raw outputs behind every numerical claim, and take full responsibility for the paper, code, and artifacts.
\bibliography{references}
\bibliographystyle{plainnat}
\clearpage
\appendix
\section{Related work through three structural distinctions}\label{sec:related}
\smallhead{Predictive targets versus inferential targets.}
PFNs amortize prediction over synthetic data sets \citep{muller2022,hollmann2025}. Their theory links variance to sample sensitivity and bias to localization \citep{nagler2023}, and analyzes optimal in-context adaptation under specified task-distribution shift \citep{ma2025}. CausalPFN amortizes treatment-effect estimation with Bayesian uncertainty, Do-PFN predicts interventional outcomes, and CausalFM constructs SCM-based priors across identification regimes \citep{causalpfn,dopfn,causalfm}. FSP builds on this interface but studies how the synthetic supervision target shapes the first-order sampling behavior of a frozen estimator at a fixed mechanism. Its focus is causal-label learnability and finite-pretraining inference, rather than another causal prior, identification regime, or uncertainty output. The simulation study of \citet{mourao2026} further motivates assessing point-estimation error and interval coverage separately.

\smallhead{Training-side versus deployment-side correction.}
Neyman-orthogonal scores and cross-fitting support debiased estimation with flexible nuisance learners \citep{dml}, while Dragonnet's targeted regularization incorporates treatment-effect structure into neural training on the data set being analyzed \citep{dragonnet}. FSP therefore differs in the unit and purpose of supervision, not simply in acting during training: it labels whole synthetic tables to learn a reusable estimator. A close PFN-specific comparison is \citet{ospc}, who identify prior-induced confounding bias and establish a semiparametric Bernstein--von Mises result for the OSPC ATE posterior under stated concentration conditions. Their MP-OSPC implementation recovers outcome and propensity nuisance posteriors from a PFN through martingale posteriors and applies an EIF-based one-step correction. FSP instead changes supervision before deployment and analyzes the native frozen output, explicitly tracking finite-pretraining, approximation, and mechanism-transfer errors. The distinction is between constructing a corrected posterior from nuisance posteriors and learning the sampling response of a reusable table-to-effect map.

\smallhead{Simulator privilege versus deployment information.}
Simulator-accessible joint scores and joint likelihood ratios already provide privileged labels for likelihood-free inference \citep{brehmer2020}; simulator privilege itself is not our contribution. FSP uses simulator-known causal nuisances for a different target: each synthetic table's efficient fluctuation. WALDO wraps a learned predictor or posterior estimator in simulation-calibrated Neyman inversion to construct confidence regions \citep{waldo}. FSP instead learns the estimator's first-order sampling response and studentizes the native frozen output with a separate variance head. The new results are the finite-stratum label-learnability transition and its connection to finite-pretraining inference. The variance head targets $V(P)=\E_P\psi_P^2$, the coefficient in the sampling variance $V(P)/n$, rather than conditional label uncertainty. Table~\ref{tab:routes} in Appendix~\ref{app:scope} summarizes what is learned once and what a new data set still requires.

\section{Assumptions, probability levels, and implementation contract}\label{app:scope}
The probability structure is
\[
 (P_j,D_j,T_j,V_j)_{j=1}^M
 \longrightarrow(\widehat f,\widehat v)\ \text{(freeze)};
 \qquad P^\star\longrightarrow D_n^{(1)},D_n^{(2)},\ldots
 \longrightarrow
 \bigl(\widehat f(D_n^{(b)}),\widehat v(D_n^{(b)})\bigr).
\]
Here $D_j=D_n^{(j)}$, $T_j=T_{P_j}(D_j)$, and $V_j=V(P_j)$. The first arrow is training-task learning; the second is repeated sampling at one mechanism.

\begin{table}[h]
\centering\scriptsize
\caption{Assumptions are paired with their role, an observable diagnostic, and the mathematical consequence of failure.}\label{tab:assumptions}
\begin{tabularx}{\linewidth}{@{}p{.18\linewidth}p{.25\linewidth}p{.25\linewidth}Y@{}}\toprule
Assumption & Why it is needed & Observable diagnostic & Failure consequence\\\midrule
Consistency and exchangeability & Identify the ATE in~\eqref{eq:target} & Scientific design and sensitivity analysis; not testable from one table & The same algebra targets an observed standardized contrast, not necessarily a causal effect\\
$p_s\ge p_*$ and $e_s\in[\epsilon,1-\epsilon]$ & Control empty cells and efficient-score moments & Stratum and arm counts; estimated propensity tails & Constants or variance diverge; uniform root-$n$ claim is unavailable\\
Bounded outcomes and outputs & Concentrate finite-episode squared losses & Outcome support; enforced output clipping & Replace by robust losses and explicit tail conditions\\
Independent synthetic episodes & Make $M$ the effective task sample size & Generator seeds and episode provenance & Dependent queries cannot be counted as independent tasks\\
Accurate simulator labels & Match the intended influence geometry & Teacher-corruption and nuisance diagnostics & Proposition~\ref{prop:teachererror} adds bias and variance to the native defect\\
Architecture and optimization control & Separate representability, estimation, and achieved fit & Norm/width audit; held-out target loss; optimization gap & The corresponding terms in~\eqref{eq:mainbound} persist\\
Mechanism coverage & Transfer training risk to deployment & Prior-density ratio, held-out shift, or panel resolution & No guarantee for unsupported mechanisms\\
$V(P)>0$ and consistent $\widehat V$ & Studentization and Wald coverage & $\widehat V/V$, oracle/native coverage, QQ diagnostics & Scale error appears explicitly in~\eqref{eq:studentized-decomp}\\\bottomrule
\end{tabularx}
\end{table}

\begin{algorithm}[h]
\caption{FSP changes supervision and keeps deployment frozen}\label{alg:fsp}
\begin{algorithmic}[1]
\REQUIRE Causal mechanism generator; table network $f_\omega$; variance head $v_\zeta$.
\FOR{each independent synthetic episode}
\STATE Sample $P$, generate $D_n\sim P^n$, and compute $T_{\lambda,P}(D_n)$ and $V(P)$ from the simulator law.
\STATE Update $\omega$ with squared label loss; update $\zeta$ with variance-coefficient loss. Set $\lambda=1$ for FSP.
\ENDFOR
\STATE Freeze all weights. On a new observed table, return $f_\omega(D_n)$ and $v_\zeta(D_n)$; use~\eqref{eq:deployed} under the stated inference conditions.
\end{algorithmic}
\end{algorithm}

\begin{table}[t]
\centering\scriptsize
\caption{\textbf{Where inferential structure enters.} Representative, non-exclusive workflows differ in what is learned once and what a new data set still requires. ``Native'' denotes the uncorrected output of FSP's fixed mean and variance heads.}\label{tab:routes}
\setlength{\tabcolsep}{3pt}
\renewcommand{\arraystretch}{1.08}
\begin{tabularx}{\linewidth}{@{}p{.19\linewidth}p{.18\linewidth}p{.31\linewidth}Y@{}}\toprule
Pattern & Representative examples & Where inferential structure enters & Deployment object or operation\\\midrule
Amortized causal or interventional prediction
& CausalPFN; Do-PFN; CausalFM \citep{causalpfn,dopfn,causalfm}
& Synthetic causal or interventional target
& Frozen forward prediction of an effect, uncertainty object, or interventional outcome\\
Per-data-set orthogonal or targeted estimation
& DML; targeted regularization \citep{dml,dragonnet}
& Nuisance fits plus an orthogonal score and cross-fitting, or a targeted neural objective
& Construct an estimate anew on the analyzed data set\\
Downstream inferential construction
& MP-OSPC; WALDO \citep{ospc,waldo}
& After nuisance posteriors, a predictor, or a posterior estimator has been learned
& EIF-based one-step ATE posterior, or simulation-calibrated critical values and Neyman inversion\\
FSP
& This work
& Efficient table fluctuation and $V(P)$ in synthetic supervision
& Reuse native fixed heads; guarantees require stated mechanism-transfer and variance conditions, with a finite-pretraining upper bound and separate hard-family lower bounds\\\bottomrule
\end{tabularx}
\end{table}

\section{Proof conventions and the efficient-label representation}\label{app:conventions}
All expectations are Lebesgue integrals. Conditional expectations are understood up to almost-sure equality. All table functions are measurable; in the finite-class results a fixed deterministic ordering resolves ties, so the empirical selector is measurable. Independence across synthetic episodes is distinct from independence of the rows within a table. A trained network is conditioned on throughout a deployment calculation. The results remain valid for independent external training randomness by enlarging the training sigma-field. A statement uniform over mechanisms is explicitly labeled as such.

The following standard probability tools are used in their usual precise forms \citep{boucheron2013,vdv1998}: (i) for independent centered $Z_i$ with $|Z_i|\le L$ and $\sum_i\E Z_i^2\le v$, Bernstein's inequality gives
\[
 P\left(\left|\sum_iZ_i\right|\ge\sqrt{2vt}+\tfrac23Lt\right)\le2e^{-t};
\]
(ii) for iid centered $Z_i$ with variance $\sigma^2>0$ and finite third absolute moment, the Kolmogorov distance of $\sum_i Z_i/(\sigma\sqrt n)$ from $N(0,1)$ is at most $C_{\rm BE}\E|Z_i|^3/(\sigma^3\sqrt n)$; (iii) a measurable statistic cannot increase total variation; (iv) $\TV(P,Q)\le\sqrt{\KL(P,Q)/2}$. These are invoked as classical theorems, not claimed as new results. All problem-specific assertions below are proved.

\begin{lemma}[Canonical gradient and exact label moments]\label{lem:score}
Under the conditions of Section~\ref{sec:setup}, $|\phi_P|\le H$, $\E_P\phi_P=\theta(P)$, and
\begin{equation}\label{eq:Vformula}
 V(P)=\sum_s p_s\left[(m_{1s}-m_{0s}-\theta)^2
       +\frac{\sigma_{1s}^2}{e_s}+\frac{\sigma_{0s}^2}{1-e_s}\right].
\end{equation}
Consequently $\E_{P^n}T_P=\theta$ and $\Var_{P^n}(T_P)=V(P)/n$. In the interior binary-outcome observed-law model, $\psi_P$ is the canonical gradient of the ATE functional and $V(P)$ is its semiparametric efficiency bound per observation.
\end{lemma}
\begin{proof}
For fixed $X=s$, exchangeability and consistency identify the observed conditional means with the potential-outcome means. The two residual terms in~\eqref{eq:phi} each have conditional expectation zero. Therefore $\E(\phi_P\mid X=s)=m_{1s}-m_{0s}$, proving the mean identity after averaging over $X$. Since $|m_{1s}-m_{0s}|\le1$ and exactly one of the two residual terms is nonzero, $|\phi_P|\le1+1/\epsilon=H$.

Let $U=A(Y-m_{1X})/e_X-(1-A)(Y-m_{0X})/(1-e_X)$. Conditional on $X$, $U$ has mean zero. Its two summands have zero product, hence
\[
 \E(U^2\mid X=s)=\sigma_{1s}^2/e_s+\sigma_{0s}^2/(1-e_s).
\]
It is uncorrelated with the $X$-measurable contrast $m_{1X}-m_{0X}-\theta$. Expanding $\psi_P^2$ proves~\eqref{eq:Vformula}. Independence of the rows gives the variance of their mean.

For the gradient claim, put $\mathcal O=\{1,\ldots,K\}\times\{0,1\}\times\{0,1\}$ and write $q(o)>0$ for the observed mass function. Every mean-zero function $h$ on $\mathcal O$ is the score of the local path $q_t(o)=q(o)(1+th(o))$ for sufficiently small $|t|$. Differentiating the finite sums and ratios gives
\[
 \dot p_s=\E[\ind\{X=s\}h(O)],\qquad
 \dot m_{as}=\frac{\E[\ind\{X=s,A=a\}(Y-m_{as})h(O)]}{P(X=s,A=a)}.
\]
The product rule applied to $\theta=\sum_sp_s(m_{1s}-m_{0s})$ then yields $\dot\theta=\E[\phi_Ph]=\E[\psi_Ph]$. The observed-law tangent space is the entire finite-dimensional space of mean-zero functions; an observed law in this space can be realized by a causal world with independent assignment given $X$ and the stated conditional outcome marginals. Thus the unique gradient in that tangent space is $\psi_P$, completing the claim.
\end{proof}

\section{Exact Gaussian objective mismatch}\label{app:gaussian}
In this illustration, the minimization ranges over square-integrable functions of $Z$, separately from the bounded causal class $\F$. At fixed $\theta$, $\E_\theta$ and $\Var_\theta$ average repeated draws of $Z\mid\theta$, and $\operatorname{Bias}_\theta(f)=\E_\theta f(Z)-\theta$ and $\operatorname{MSE}_\theta(f)=\E_\theta\{f(Z)-\theta\}^2$. For fixed $\theta,s^2,v,\lambda$, Proposition~\ref{prop:gauss} gives $1-a_{\lambda,n}=O(n^{-1})$. Thus every fixed partial-fluctuation predictor approaches first-order efficiency as $n$ grows; the endpoint distinction is the order of synthetic-label prediction error, not a universal necessity for asymptotic efficiency.
\begin{proof}[Proof of Proposition~\ref{prop:gauss}]
Multiplication of the likelihood $\exp\{-n(Z-\theta)^2/(2v)\}$ by the prior density $\exp\{-\theta^2/(2s^2)\}$ and completion of the square gives
\[
 \theta\mid Z\sim N\left(\frac{ns^2}{ns^2+v}Z,
                 \frac{s^2v}{ns^2+v}\right).
\]
For any square-integrable random target $U$, conditioning on the input $Z$ gives
\[
 \E[(a-U)^2\mid Z]=(a-\E[U\mid Z])^2+\Var(U\mid Z),
\]
so the conditional mean uniquely minimizes squared loss up to null sets. In the Gaussian location experiment, the efficient influence average is $Z-\theta$, so the full-fluctuation label is $T_1=Z$.
Because $T_\lambda=(1-\lambda)\theta+\lambda Z$, posterior linearity gives
\[
 \E(T_\lambda\mid Z)=\{(1-\lambda)\kappa_n+\lambda\}Z
 =a_{\lambda,n}Z,
 \qquad
 \Var(T_\lambda\mid Z)=(1-\lambda)^2\frac{s^2v}{ns^2+v}.
\]
At fixed $\theta$, $a_{\lambda,n}Z$ has mean $a_{\lambda,n}\theta$ and variance $a_{\lambda,n}^2v/n$, which proves Proposition~\ref{prop:gauss}. If $\lambda<1$ is fixed, the conditional variance multiplied by $n$ converges to $(1-\lambda)^2v$; at $\lambda=1$ it vanishes identically.
\end{proof}

\begin{lemma}[Population $\lambda$-supervision is an $L_2$ projection]\label{lem:projection}
Let the joint training law be $P\sim\Pi$, $D_n\sim P^n$, with $T_\lambda=T_{\lambda,P}(D_n)\in L_2$. Set $g_\lambda(D_n)=\E[T_\lambda\mid D_n]$. For every measurable $f(D_n)\in L_2$,
\begin{equation}\label{eq:projection}
 \E(f-T_\lambda)^2=\E\Var(T_\lambda\mid D_n)+\E(f-g_\lambda)^2.
\end{equation}
At $\lambda=1$, the minimax finite-stratum table-prediction risk is $\Theta(n^{-2})$. This conditional label ambiguity is distinct from the deployment sampling variance $V(P)/n$.
\end{lemma}
\begin{proof}
Expand $f-T_\lambda=(f-g_\lambda)+(g_\lambda-T_\lambda)$. The cross term integrates to zero because $f-g_\lambda$ is measurable with respect to $D_n$ and $\E[g_\lambda-T_\lambda\mid D_n]=0$. Conditional expectation of the second squared term is $\Var(T_\lambda\mid D_n)$. The $\lambda=1$ upper and lower orders follow from Theorem~\ref{thm:learnability}; Proposition~\ref{prop:gauss} gives the stated separation from sampling variance.
\end{proof}

\section{The finite-stratum lambda-supervision phase transition}\label{app:learnability}
Define
\[
 N_s=\sum_i\ind\{X_i=s\},\quad N_{as}=\sum_i\ind\{X_i=s,A_i=a\},\quad
 Z_{as}=\sum_i\ind\{X_i=s,A_i=a\}Y_i.
\]
The denominator-truncated observable comparator is
\begin{equation}\label{eq:Sn}
 S_n(D_n)=\sum_s\frac{N_s}{n}
  \left\{\frac{Z_{1s}/n}{\max(N_{1s}/n,q_*/2)}-
          \frac{Z_{0s}/n}{\max(N_{0s}/n,q_*/2)}\right\}.
\end{equation}
Each ratio lies in $[0,1]$, so $|S_n|\le1$. Let
\[
 G_n=\bigcap_{s,a}\{N_{as}\ge nq_*/2\}.
\]
On $G_n$, $S_n$ is the ordinary empirical-stratum-weighted difference of outcome means.

\begin{lemma}[Rare-cell probability]\label{lem:counts}
Uniformly over $\PP_{K,p_*,\epsilon}$,
$P(G_n^c)\le2K\exp(-nq_*/8)$.
\end{lemma}
\begin{proof}
If $B\sim\operatorname{Binomial}(n,q)$, then Markov's inequality and the binomial moment generating function give
\[
 P(B\le nq/2)\le e^{tnq/2}(1-q+qe^{-t})^n
 \le \exp\{nq(t/2+e^{-t}-1)\}.
\]
Taking $t=\log2$ makes the exponent at most $-nq/8$ because $(\log2)/2-1/2<-1/8$. For each cell, $q=P(X=s,A=a)\ge q_*$. The event $N_{as}<nq_*/2$ is contained in $N_{as}<nq/2$, so its probability is at most $e^{-nq_*/8}$. A union bound over $2K$ cells completes the proof.
\end{proof}

\begin{proof}[Upper bound in Theorem~\ref{thm:learnability}]
Let $e_{1s}=e_s$ and $e_{0s}=1-e_s$. On $G_n$, define $\bar Y_{as}=Z_{as}/N_{as}$. Subtracting~\eqref{eq:teacher} from the empirical contrast gives the exact identity
\begin{equation}\label{eq:count-product}
 S_n-T_P=\sum_{s,a}(2a-1)
       \frac{N_s-N_{as}/e_{as}}n(\bar Y_{as}-m_{as}).
\end{equation}
Condition on all $(X_i,A_i)$. Conditional outcome residuals in distinct cells are independent and have mean zero. The event $G_n$ is measurable under this conditioning. Therefore the conditional squared error of the right side is the sum, not the square of the sum, of its conditional variances:
\[
 \E_{P^n}[(S_n-T_P)^2\ind_{G_n}]
 =\E_{P^n}\left[\ind_{G_n}\sum_{s,a}
    \frac{(N_s-N_{as}/e_{as})^2}{n^2}\frac{\sigma_{as}^2}{N_{as}}\right].
\]
A bounded $[0,1]$ variable has variance at most $1/4$. On $G_n$, $1/N_{as}\le2/(nq_*)$. For fixed $(s,a)$, the centered count difference is a sum of iid mean-zero variables $\ind\{X_i=s\}(1-\ind\{A_i=a\}/e_{as})$. Its second moment equals
\[
 \E_{P^n}(N_s-N_{as}/e_{as})^2
 =np_s(1-e_{as})/e_{as}\le np_s/\epsilon.
\]
Combining these identities and summing $\sum_{s,a}p_s=2$ yields
\[
 \E_{P^n}[(S_n-T_P)^2\ind_{G_n}]\le\frac{1}{n^2q_*\epsilon}
 =\frac1{n^2p_*\epsilon^2}.
\]
On $G_n^c$, $|S_n-T_P|\le H+1$. Lemma~\ref{lem:counts} bounds this contribution by $2K(H+1)^2e^{-nq_*/8}$. Adding the two parts gives~\eqref{eq:comparator}.

For the complete $\lambda$ path, use the same observable rule $S_n$ and write
\[
 S_n-T_{\lambda,P}
 =\lambda(S_n-T_P)+(1-\lambda)(S_n-\theta).
\]
The squared triangle inequality and~\eqref{eq:comparator} give
\[
 \E_{P^n}(S_n-T_{\lambda,P})^2
 \le2\lambda^2A_n+2(1-\lambda)^2\E_{P^n}(S_n-\theta)^2.
\]
Moreover, $S_n-\theta=(S_n-T_P)+(T_P-\theta)$, so Lemma~\ref{lem:score}, another squared triangle inequality, and $V(P)\le(H+1)^2\le4H^2$ imply
\[
 \E_{P^n}(S_n-\theta)^2\le2A_n+2V(P)/n
 \le2A_n+8H^2/n.
\]
Substitution proves~\eqref{eq:lambda-upper}.
\end{proof}

\begin{proof}[Lower bound in Theorem~\ref{thm:learnability}]
Fix any admissible stratum vector $(p_s)_{s=1}^K$ with $p_s\ge p_*$, draw $X\sim p$ independently of $(A,Y)$, and use the same treatment and outcome law in every stratum. The law of $X$ is parameter-free and ancillary, so conditioning on the complete $X$ sequence leaves every Bayes calculation below unchanged. For the first term in~\eqref{eq:label-lower}, take $A\sim\operatorname{Bernoulli}(1/2)$ and $Y\sim\operatorname{Bernoulli}(\mu)$ independently, and put $S_i=2A_i-1$, $\bar S=n^{-1}\sum_iS_i$, and $N_Y=\sum_iY_i$. The ATE is zero, while its full-model efficient label is
\[
 T_\mu(D_n)=\frac2n\sum_i S_i(Y_i-\mu)
           =\frac2n\sum_iS_iY_i-2\bar S\mu.
\]
Hence $T_{\lambda,\mu}=\lambda T_\mu$.  Place the uniform prior on $\mu\in[0,1]$. The treatment signs are independent of $\mu$ and the outcome table. Given $Y_1,\ldots,Y_n$, the posterior is $\operatorname{Beta}(N_Y+1,n-N_Y+1)$, with variance
\[
 \frac{(N_Y+1)(n-N_Y+1)}{(n+2)^2(n+3)}.
\]
The prior-predictive distribution of $N_Y$ is uniform on $\{0,\ldots,n\}$ because
$\binom nk\int_0^1\mu^k(1-\mu)^{n-k}\,d\mu=1/(n+1)$.
Summing the quadratic numerator over $k$ gives
\[
 \sum_{k=0}^n(k+1)(n-k+1)
 =\frac{(n+1)(n+2)(n+3)}6,
 \qquad
 \E\Var(\mu\mid Y_1,\ldots,Y_n)=\frac1{6(n+2)}.
\]
Since $\E\bar S^2=1/n$, the Bayes prediction risk is
\[
 \E\Var(T_{\lambda,\mu}\mid D_n)
 =4\lambda^2\E\bar S^2\,\E\Var(\mu\mid Y)
 =\frac{2\lambda^2}{3n(n+2)}.
\]

For the first-order term, keep $A\sim\operatorname{Bernoulli}(1/2)$, set
$Y\mid A=0\sim\operatorname{Bernoulli}(1/2)$ and
$Y\mid A=1\sim\operatorname{Bernoulli}(\mu)$, and again place the uniform prior on $\mu$.  Let $N=\sum_iA_i$.  The ATE is $\theta_\mu=\mu-1/2$, and direct substitution in the efficient score gives
\[
 T_{\lambda,\mu}(D_n)
 =\left(1-\frac{2\lambda N}{n}\right)\mu+C_\lambda(D_n),
\]
where $C_\lambda(D_n)$ is observable and contains no $\mu$. Conditional on $N=k$ and the treated outcomes, the posterior is
$\operatorname{Beta}(Z_1+1,k-Z_1+1)$, where $Z_1=\sum_iA_iY_i$.  Its prior-predictive average variance is $1/\{6(k+2)\}$ by the same beta-integral calculation.  Control outcomes contain no information about $\mu$. Therefore
\[
 \E\Var(T_{\lambda,\mu}\mid D_n)
 =\E_{N\sim\operatorname{Binomial}(n,1/2)}
   \left[\frac{(1-2\lambda N/n)^2}{6(N+2)}\right]
 =r_{n,\lambda}.
\]
Every supremum risk dominates the Bayes risk under either prior, which proves the two lower bounds.

Because $N+2\le n+2$ and $\E(2N/n)=1$ with
$\Var(2N/n)=1/n$,
\[
 r_{n,\lambda}
 \ge\frac{\E(1-2\lambda N/n)^2}{6(n+2)}
 =\frac{(1-\lambda)^2+\lambda^2/n}{6(n+2)}.
\]

It remains to establish the limit. On $\{N\ge n/4\}$,
\[
 \frac{n(1-2\lambda N/n)^2}{6(N+2)}
 \longrightarrow \frac{(1-\lambda)^2}{3}
\]
in probability, since $N/n\to1/2$. For $\lambda\in[0,1]$ the integrand on this event is bounded by $2/3$, so convergence in probability also gives convergence of its expectation. On $\{N<n/4\}$ it is at most $n/12$, while the binomial Chernoff bound $P(N<n/4)\le e^{-n/16}$ makes that event's expected contribution at most $ne^{-n/16}/12$. Together these prove
$nr_{n,\lambda}\to(1-\lambda)^2/3$.
\end{proof}

\section{Finite pretraining: complete concentration argument}\label{app:training}
For the covering argument, the observed-table domain is $\mathcal O^n$ with $\mathcal O=\{1,\ldots,K\}\times\{0,1\}\times[0,1]$, and $\|f-g\|_\infty=\sup_{D\in\mathcal O^n}|f(D)-g(D)|$. The class $\F$ and its $\xi$-cover are fixed before pretraining; $a_n$ is the uniform approximation error of a member $f_0\in\F$ to the observable comparator $S_n$, whereas $\xi$ is a resolution used in the concentration argument. The confidence parameter $\delta$ selects a training event. Conditional on that event, $\gamma$ determines the excluded task mass in the mechanism-transfer bound, and $C_{\rm sh}$ bounds a task-law density ratio. A balanced panel uses a separate maximum-over-mechanisms loss and $M=Jm$ independent episodes, not the pooled empirical loss~\eqref{eq:loss}.
\begin{lemma}[A bounded nonnegative-loss oracle inequality]\label{lem:erm}
Let $\mathcal G$ be a nonempty finite class fixed before $m$ independent, identically distributed training observations. Suppose $0\le\ell_g\le L$. Write $R_g=\E\ell_g$ and $\widehat R_g=m^{-1}\sum_i\ell_g(Z_i)$. If $\widehat R_{\widehat g}\le\min_g\widehat R_g+\eta$, then with probability at least $1-\delta$,
\[
 R_{\widehat g}\le3\min_gR_g+\frac{20L\log(2|\mathcal G|/\delta)}{3m}+2\eta.
\]
The same statement holds for minimizing the maximum of $J$ empirical risks, each based on $m$ independent observations, with $\min_gR_g$ replaced by $\min_g\max_jR_{jg}$ and the logarithm replaced by $\log(2J|\mathcal G|/\delta)$.
\end{lemma}
\begin{proof}
Since $0\le\ell_g\le L$, $\Var(\ell_g)\le\E\ell_g^2\le LR_g$. Bernstein's inequality and a union bound imply, for $t=\log(2|\mathcal G|/\delta)$, simultaneously for all $g$,
\[
 |\widehat R_g-R_g|\le\sqrt{2LR_gt/m}+2Lt/(3m)
 \le R_g/2+5Lt/(3m).
\]
The second inequality follows from $\sqrt{2xy}\le x/2+y$ with $x=R_g$, $y=Lt/m$. Set $c=5Lt/(3m)$. Then $R_{\widehat g}\le2\widehat R_{\widehat g}+2c$. For a population minimizer $g_*$, which exists because the class is finite,
\[
 R_{\widehat g}\le2\widehat R_{g_*}+2\eta+2c
 \le3R_{g_*}+2\eta+4c.
\]
This proves the scalar statement. For the panel version, apply the same event to every pair $(j,g)$, then use
\[
 \max_jR_{j\widehat g}\le2\max_j\widehat R_{j\widehat g}+2c
 \le2\min_g\max_j\widehat R_{jg}+2\eta+2c
 \le3\min_g\max_jR_{jg}+2\eta+4c.
\]
Independence between panels is not needed for the union bound, although the stated construction supplies it. Independence and identical distribution within each panel are needed for its concentration bound.
\end{proof}

\subsection{A lower bound in the number of pretraining episodes}\label{app:meta-lower}
\begin{proof}[Proof of Theorem~\ref{thm:meta-lower}]
Put $\tau^2=1/n$ and $\delta=(4\sqrt M)^{-1}$. In meta-world $\sigma\in\{-1,1\}$, generate
\[
 \theta\sim N(0,s_\sigma^2),\qquad
 \varepsilon\sim N(0,r_\sigma^2),\qquad Z=\theta+\varepsilon,
\]
independently, where
\[
 s_\sigma^2=\frac{\tau^2}{2}(1+\sigma\delta),\qquad
 r_\sigma^2=\frac{\tau^2}{2}(1-\sigma\delta).
\]
Every labeled episode reveals $(Z,T_\lambda)$ with
$T_\lambda=(1-\lambda)\theta+\lambda Z$. Both worlds have the same deployment marginal $Z\sim N(0,\tau^2)$, while Gaussian conditioning gives
\[
 g_\sigma(z)=\E_\sigma(T_\lambda\mid Z=z)
 =\left\{\lambda+(1-\lambda)\frac{1+\sigma\delta}{2}\right\}z.
\]
Thus, under their common $Z$ law $\nu$,
\begin{equation}\label{eq:meta-separation}
 \norm{g_+-g_-}_{L_2(\nu)}^2
 =\tau^2(1-\lambda)^2\delta^2
 =\frac{(1-\lambda)^2}{16nM}.
\end{equation}

For the lower bound, allow the learner to observe the more informative latent pairs $(\theta,\varepsilon)$; every labeled-sample learner is obtained by applying the measurable map $(\theta,\varepsilon)\mapsto(Z,T_\lambda)$ first. Let $p_\sigma$ denote the density of one latent pair and put $A_M=\int\sqrt{p_+^{\otimes M}p_-^{\otimes M}}$. Direct Gaussian integration and product factorization give
\[
 A_M^2=\left\{\frac{4s_+^2r_+^2}{(s_+^2+r_+^2)^2}\right\}^{M}
 =(1-\delta^2)^M\ge1-M\delta^2=\frac{15}{16}.
\]
Writing $\alpha_M=\int\min(p_+^{\otimes M},p_-^{\otimes M})$, Cauchy--Schwarz yields
\[
 A_M^2\le
 \left\{\int\min(p_+^{\otimes M},p_-^{\otimes M})\right\}
 \left\{\int\max(p_+^{\otimes M},p_-^{\otimes M})\right\}
 =\alpha_M(2-\alpha_M),
\]
so $\alpha_M\ge3/4$.

Let an arbitrary algorithm map the training sample $x$ to $\widehat f_x$. Conditional-expectation orthogonality identifies its excess episode risk in world $\sigma$ with $\norm{\widehat f_x-g_\sigma}_{L_2(\nu)}^2$, also in the extended-real sense for predictors of infinite risk. The parallelogram inequality gives
\begin{align*}
 &\E_+\norm{\widehat f-g_+}_{L_2(\nu)}^2
  +\E_-\norm{\widehat f-g_-}_{L_2(\nu)}^2\\
 &\quad\ge \int \min(p_+^{\otimes M},p_-^{\otimes M})
   \left\{\norm{\widehat f_x-g_+}_2^2+\norm{\widehat f_x-g_-}_2^2\right\}dx\\
 &\quad\ge\frac{\alpha_M}{2}\norm{g_+-g_-}_{L_2(\nu)}^2.
\end{align*}
The larger excess risk is at least half their sum. Combining this with~\eqref{eq:meta-separation}, and restricting back to labeled-sample learners, proves
\[
 \sup_\sigma\E_\sigma[\text{excess risk}]
 \ge\frac{3}{256}\frac{(1-\lambda)^2}{nM}
 \ge\frac{(1-\lambda)^2}{100nM}.
\]
At $\lambda=1$, $T_1=Z$ and $g_+=g_-=Z$.
Independent algorithmic randomization is included by conditioning on its seed and then integrating the same inequality.
\end{proof}

\begin{proof}[Proof of Theorem~\ref{thm:dictionary-lower}]
Let $d=d_N=\lfloor\log_2N\rfloor$ and let $J$ be uniform on $\{1,\ldots,d\}$. For $\sigma\in\{-1,1\}^d$ and $\varepsilon\in(0,1/2]$, let the episode label $L\in\{-1,1\}$ satisfy
\[
 Q_\sigma(L=1\mid J=j)=\frac{1+\varepsilon\sigma_j}{2}.
\]
The regression function is $g_\sigma(j)=\E_\sigma(L\mid J=j)=\varepsilon\sigma_j$; take $\F_N=\{g_\sigma:\sigma\in\{-1,1\}^d\}$, whose cardinality is $2^d\le N$. For every predictor $f$, conditional-mean orthogonality gives
\begin{equation}\label{eq:dictionary-orthogonality}
 R_\sigma(f)-\inf_gR_\sigma(g)=\frac1d\sum_{j=1}^d\{f(j)-\varepsilon\sigma_j\}^2.
\end{equation}

Choose $\varepsilon^2=\min\{1/4,d/(8M)\}$. If $\sigma^{(j)}$ is obtained by flipping coordinate $j$, then
\[
 \KL(Q_\sigma,Q_{\sigma^{(j)}})
 =\frac{\varepsilon}{d}\log\frac{1+\varepsilon}{1-\varepsilon}
 \le\frac{4\varepsilon^2}{d},
\]
because $\log\{(1+x)/(1-x)\}\le4x$ on $[0,1/2]$. Product additivity and Pinsker's inequality imply
\[
 \TV(Q_\sigma^M,Q_{\sigma^{(j)}}^M)
 \le\sqrt{2M\varepsilon^2/d}\le\frac12.
\]
For the learner's output set $\widehat\sigma_j=1$ if $\widehat f(j)\ge0$ and $-1$ otherwise. Pairing each $\sigma$ with $\sigma^{(j)}$ and applying the binary testing inequality yields
\[
 2^{-d}\sum_\sigma P_\sigma(\widehat\sigma_j\ne\sigma_j)
 \ge\frac12\{1-\sup_\sigma\TV(Q_\sigma^M,Q_{\sigma^{(j)}}^M)\}\ge\frac14.
\]
The uniform-prior average expected Hamming loss is therefore at least $d/4$. A sign error in coordinate $j$ makes the corresponding squared error in~\eqref{eq:dictionary-orthogonality} at least $\varepsilon^2$. Lower-bounding a supremum by the hypercube average proves
\[
 \sup_\sigma\E_\sigma\{R_\sigma(\widehat f)-R_\sigma(g_\sigma)\}
 \ge\varepsilon^2/4
 =\min\!\left\{\frac1{16},\frac d{32M}\right\}.
\]
The argument covers randomized learners by conditioning on an independent seed. The resulting excess-risk lower bound is specific to this episode-regression experiment; it does not by itself lower-bound the causal sampling defect of~\eqref{eq:defect}.
\end{proof}

\subsection{Covering-number upper bound and transfer}
\begin{proof}[Training and transfer part of Theorem~\ref{thm:main}]
For each $f\in\F$, the episode loss $(f(D_n)-T_P(D_n))^2$ is nonnegative and bounded by $4H^2$. The approximation assumption and $(u+v)^2\le2u^2+2v^2$ give
\[
 \risk_\Pi(f_0)\le2a_n^2+2\E_\Pi\E_{P^n}(S_n-T_P)^2\le2a_n^2+2A_n.
\]
Let $\mathcal G_\xi$ be a uniform $\xi$-net of $\F$ of minimum cardinality $N_\xi$. Squared loss is $4H$-Lipschitz in its prediction on $[-H,H]$: if $\norm{f-g}_\infty\le\xi$, then both empirical and population risks differ by at most $4H\xi$. Choose $\widetilde f\in\mathcal G_\xi$ within $\xi$ of $\widehat f$. Then
\[
 \widehat\risk_{1,\Pi}(\widetilde f)
 \le\inf_{g\in\mathcal G_\xi}\widehat\risk_{1,\Pi}(g)+\eta_M+4H\xi.
\]
A net point within $\xi$ of $f_0$ has population risk at most
$\risk_\Pi(f_0)+4H\xi$. Apply Lemma~\ref{lem:erm} to $\mathcal G_\xi$ with $L=4H^2$ and tolerance $\eta_M+4H\xi$, and transfer back from $\widetilde f$ to $\widehat f$. Since $80/3\le32$,
\[
 \risk_\Pi(\widehat f)\le6a_n^2+6A_n+24H\xi
 +\frac{32H^2\log(2N_\xi/\delta)}M+2\eta_M.
\]
Multiply by $n$. On the resulting training event, nonnegativity and $d\Pi'/d\Pi\le C_{\rm sh}$ imply
$\E_{\Pi'}\defect\le C_{\rm sh}\E_\Pi\defect$. If $P_j$ has mass $\pi_j$, then $\pi_j\defect(P_j)\le\E_\Pi\defect$. These are deterministic consequences on the same training event. They do not require resampling the trained model.
\end{proof}

\smallhead{Variable context lengths.}
For training lengths in a finite set $\mathcal N$ with probability $\nu(n)>0$, one may apply the theorem conditionally to each length using its independent episode count, with a union budget. Alternatively, the nonnegative joint risk bounds the risk at length $n$ after division by $\nu(n)$. Neither argument yields a theorem for a length assigned zero training probability. Reusing many masked queries from a single mechanism-table pair does not multiply its number of independent episodes.

\section{A constructive table-attention approximation}\label{app:architecture}
The next two lemmas instantiate both complexity terms in Theorem~\ref{thm:main}: norm constraints yield a covering number for continuous-weight Transformers, and an explicit attention circuit approximates the finite-stratum comparator.

\begin{lemma}[Metric entropy of a norm-constrained Transformer]\label{lem:transformer-cover}
Fix a table-Transformer computation graph with $W$ scalar parameters and bounded row inputs. Constrain every parameter to $[-B,B]$, every matrix operator norm and hidden representation to a bounded set, and every layer-normalization denominator by a fixed positive stabilizer. Clip the scalar output to $[-H,H]$. Then there is a finite, recursively computable constant $L_{\rm tr}$ such that
\[
 \sup_{D_n}|f_w(D_n)-f_{w'}(D_n)|
 \le L_{\rm tr}\norm{w-w'}_\infty.
\]
Consequently, for every $\xi>0$,
\[
 \log\mathcal N_\infty(\xi,\F_{\rm tr})
 \le W\log\!\left(1+\frac{2BL_{\rm tr}}{\xi}\right).
\]
This bound controls continuous weights directly in Theorem~\ref{thm:main}.
\end{lemma}
\begin{proof}
On the declared compact domains, an affine map is jointly Lipschitz in its input and parameters. ReLU and output clipping are 1-Lipschitz, and GELU has bounded derivative. If $p=\operatorname{softmax}(u)$, its Jacobian is $\operatorname{diag}(p)-pp^\top$, so the mean value theorem gives
$\norm{\operatorname{softmax}(u)-\operatorname{softmax}(v)}_1\le2\norm{u-v}_\infty$. The positive layer-normalization stabilizer bounds every derivative of that operation on the compact hidden-state set. Dot-product attention is a finite composition of these maps and bounded bilinear products. Induction through the fixed computation graph therefore gives a finite parameter-to-output Lipschitz constant $L_{\rm tr}$, recursively in the layer-norm bounds, activation bounds, graph depth, and input bound.

If $L_{\rm tr}=0$ or $B=0$, all admissible parameters give the same function, so a single representative suffices. Otherwise partition each coordinate interval $[-B,B]$ into $\lceil2BL_{\rm tr}/\xi\rceil$ intervals of length at most $\xi/L_{\rm tr}$. In each product cell that intersects the admissible parameter set, select one admissible representative. Points in the same cell are within $\xi/L_{\rm tr}$ in $\ell_\infty$, so parameter Lipschitzness maps these representatives to an internal uniform $\xi$-net of $\F_{\rm tr}$. This construction retains all norm and hidden-state constraints and uses at most $(1+2BL_{\rm tr}/\xi)^W$ representatives. Taking logarithms proves the claim.
\end{proof}

\begin{lemma}[Explicit comparator approximation by a table network]\label{lem:architecture}
For $Y\in[0,1]$ and fixed $K,q_*>0$, there exists a permutation-invariant table-attention network with a ReLU row embedding, a uniform-attention aggregation, and a fixed-depth ReLU readout, whose bounded output $f_J$ satisfies
\[
 \sup_{D_n}|f_J(D_n)-S_n(D_n)|\le C_{K,q_*}J^{-2}.
\]
The readout has $O(KJ)$ hidden units and finite bounded weights for each $J$. Bounded quantization with sufficiently fine mesh preserves the approximation up to any prescribed additional error. A network family with $W_0$ trainable weights on a mesh of spacing $2^{-b}$ in $[-B_0,B_0]$ has cardinality at most $(1+2B_0\,2^b)^{W_0}$.
\end{lemma}
\begin{proof}
Encode each stratum by its one-hot vector $r_s=\ind\{X=s\}$. Because $r_s,A$ are binary and $0\le Y\le1$, the following identities remain exact:
\[
 \operatorname{ReLU}(r_s+A-1)=\ind\{X=s,A=1\},
\]
\[
 \operatorname{ReLU}(r_s+A+Y-2)=\ind\{X=s,A=1\}Y,
\]
\[
 \operatorname{ReLU}(r_s+1-A+Y-2)=\ind\{X=s,A=0\}Y.
\]
Use these $4K$ coordinates as a row embedding. A query whose attention logits are all zero assigns weight $1/n$ to each row, so its value is exactly the vector of empirical frequencies
\[
 u_s=N_s/n,\quad h_{1s}=N_{1s}/n,\quad z_{1s}=Z_{1s}/n,\quad z_{0s}=Z_{0s}/n.
\]
Set $h_{0s}=u_s-h_{1s}$. The operation $d_{as}=\max(h_{as},q_*/2)$ is represented exactly by a ReLU and an affine map. If $q_*\ge2$, every denominator equals $q_*/2$, so $S_n=(2/q_*)\sum_su_s(z_{1s}-z_{0s})$; only the product approximators below are needed. For $0<q_*<2$, let $r_J$ be piecewise-linear interpolation of $x\mapsto1/x$ on $[q_*/2,1]$ at $J+1$ equally spaced knots. On each interval, the elementary interpolation error bound follows by two applications of Rolle's theorem to the error minus a multiple of $(x-a)(x-b)$: it is at most $\sup|f''|(b-a)^2/8$. Since $\sup|f''|\le16/q_*^3$, we have
\[
 \sup_{x\in[q_*/2,1]}|r_J(x)-1/x|\le2q_*^{-3}J^{-2}.
\]
Every continuous piecewise-linear function with knots $t_j$ is exactly an affine term plus a sum $\sum_jc_j\operatorname{ReLU}(x-t_j)$, by matching successive slope increments. Thus $r_J$ has a one-hidden-layer ReLU representation with $O(J)$ units.

For products, use $uv=((u+v)^2-(u-v)^2)/4$. Piecewise-linear interpolation of $x^2$ on $[-B,B]$ at $J+1$ equally spaced knots has error at most $B^2/J^2$, by the same bound. It therefore approximates $uv$, for $|u|+|v|\le B$, with error at most $B^2/(2J^2)$. Here all intermediate exact reciprocals are at most $2/q_*$. First approximate $z_{as}r_J(d_{as})$, then multiply by $u_s$, and sum the two arms over $s$. The number of layers is a fixed constant, since these operations have fixed composition depth. All inputs and intermediate values lie in a compact interval depending only on $q_*$. The triangle inequality at each multiplication consequently bounds the total error by $C_{K,q_*}J^{-2}$, for example a sufficiently large constant $100K(1+q_*^{-3})$. Clip the scalar output to $[-1,1]$ using ReLUs; projection onto an interval cannot increase its error relative to $S_n\in[-1,1]$.

All constructed weights and biases are finite. On a compact input domain, the output of a fixed finite ReLU computation is uniformly continuous in its weights; equivalently, induction through its affine maps and 1-Lipschitz ReLUs gives a finite uniform weight-to-output Lipschitz bound on every bounded weight box. Thus sufficiently fine coordinatewise quantization adds at most any chosen tolerance. There are at most $1+2B_0 2^b$ possible values per weight, giving the displayed cardinality bound by multiplication. A zero-logit attention layer uses fixed weights and therefore introduces no approximation or additional quantized parameters.
\end{proof}

\smallhead{Architecture scope.}
The row aggregation is an exact sufficient-statistic reduction for the declared categorical experiment. Continuous tables use the histogram-attention construction of Theorem~\ref{thm:continuous}.  The implemented GELU checkpoint and the constructive ReLU circuit share the table-to-estimate interface; their approximation and optimization quantities are kept separate as $a_n$ and $\eta_M$.

\begin{corollary}[An explicit attention/pretraining schedule]\label{cor:schedule}
For fixed $K,p_*,\epsilon$ and binary outcomes, there are bounded table-attention classes with $O(n^{1/3})$ readout units and $O(\log n)$-bit quantization such that, with $M=\lceil n^{3/2}\rceil$, $\eta_M\le n^{-2}$, and $\delta_n=n^{-2}$,
\[
 \E_\Pi\defect(P;\widehat f)=O(n^{-1/6}\log n)
\]
on the training event and after integrating over training. Dominated task laws and fixed prior atoms inherit the corresponding density- and mass-weighted rates.
\end{corollary}

\begin{proof}[Proof of Corollary~\ref{cor:schedule}]
Take $J=\lceil n^{1/3}\rceil$ in Lemma~\ref{lem:architecture}. A sparse realization has readout depth at most seven and at most $20KJ+10K+2$ ReLU units, with shared wires retaining inputs across successive product modules. It uses $W_J=O(KJ)$ scalar weights; identical interpolation coefficients may be tied across strata. A single bounded box suffices for these trainable coefficients for all $J$: the knots lie in fixed intervals, reciprocal slopes are bounded by $4/q_*^2$, and square-interpolant slopes and their differences are uniformly bounded on the fixed interpolation domain.

Propagating magnitude and parameter-perturbation bounds through this specific circuit gives $L_J\le C_{K,q_*}(J+1)^3$ on that box. Choose an integer $b_0$ with $2^{b_0}\ge C_{K,q_*}$ and set $b_J=b_0+5\lceil\log_2(J+1)\rceil$. Coordinatewise dyadic rounding within the box then adds at most $L_J2^{-b_J}\le(J+1)^{-2}$ uniformly over tables. The required number of bits is $O(\log J)=O(\log n)$, the quantized comparator error is $a_n=O(J^{-2})=O(n^{-2/3})$, and
\[
 \log|\F_n|=O(W_J\log J)=O(n^{1/3}\log n).
\]
All outputs are clipped, so the class remains uniformly bounded. Substitution into Theorem~\ref{thm:main} gives $nA_n=O(n^{-1})$, $na_n^2=O(n^{-1/3})$, $n\log(2|\F_n|/\delta_n)/M=O(n^{-1/6}\log n)$, and $n\eta_M=O(n^{-1})$. These prove the high-probability rate. On the complementary training event, $\defect\le4nH^2$; multiplying by $\delta_n=n^{-2}$ adds only $O(n^{-1})$ to the unconditional task-average defect. The transfer conclusions follow directly from nonnegativity and the last part of Theorem~\ref{thm:main}; achieved optimization accuracy is represented explicitly by $\eta_M$.
\end{proof}

\section{Bias, finite-sample deviations, and Gaussian approximation}\label{app:inference}
\begin{proof}[Inference and studentization part of Theorem~\ref{thm:main}]
Fix the training realization. Set $R=f(D_n)-T_P(D_n)$. By definition $\E R^2=d/n$. Lemma~\ref{lem:score} gives $\E(T_P-\theta)=0$ and $\norm{T_P-\theta}_2=\sqrt{V/n}$. Hence Cauchy--Schwarz implies
$|\E f-\theta|=|\E R|\le\sqrt{d/n}$.
The ordinary and reverse triangle inequalities in $L_2$ give
\[
 \left|\norm{f-\theta}_2-\norm{T_P-\theta}_2\right|\le\norm R_2.
\]
Multiply by $\sqrt n$ to prove Theorem~\ref{thm:inference}. For variance, center both summands:
\[
 f-\E f=(T_P-\theta)+(R-\E R),\qquad
 \norm{R-\E R}_2\le\norm R_2.
\]
Applying the same two triangle inequalities proves the result in Theorem~\ref{thm:inference}; independence between $R$ and $T_P$ is neither assumed nor needed.

For the distribution bound, put
\[
 Z_n=\frac{1}{\sqrt{nV}}\sum_i\psi_P(O_i),\qquad U_n=\frac{\sqrt nR}{\sqrt V}.
\]
The classical Berry--Esseen inequality bounds $\sup_x|P(Z_n\le x)-\Phi(x)|$ by $C_{\rm BE}\rho(P)/\sqrt n$. Markov's inequality gives $P(|U_n|>t)\le d/(Vt^2)$. Without any independence of $U_n,Z_n$,
\[
 P(Z_n\le x-t)-P(|U_n|>t)\le P(Z_n+U_n\le x)
 \le P(Z_n\le x+t)+P(|U_n|>t).
\]
The normal density is bounded by $(2\pi)^{-1/2}$, so $|\Phi(x\pm t)-\Phi(x)|\le t/\sqrt{2\pi}$. These inequalities prove~\eqref{eq:be}. If $d>0$, choosing $t=(d/V)^{1/3}$ gives the stated order; if $d=0$, let $t\downarrow0$.

For the quantitative studentization bound, set
$X=\sqrt n(f-\theta)/\sqrt V$ and $Q=\widehat V/V$. Interpret the nonnegative moment $e_V=\E(Q-1)^2$ as an extended expectation. If $e_V=\infty$, the claimed upper bound is immediate; otherwise $Q-1\in L_2$, and the following argument applies. On the event
$E_r=\{|Q-1|\le r\}$, the threshold $x\sqrt Q$ lies between
$x\sqrt{1-r}$ and $x\sqrt{1+r}$, with the order reversed when $x<0$. Hence the distribution function of $X/\sqrt Q$ is bracketed by the corresponding two distribution functions of $X$, up to $P(E_r^c)$. For every $a>0$,
\[
 \sup_x|\Phi(ax)-\Phi(x)|
 \le\frac{|\log a|}{\sqrt{2\pi e}},
\]
because the derivative of $\Phi(e^u x)$ in $u$ is $e^ux\phi(e^ux)$ and $\sup_y|y|\phi(y)=1/\sqrt{2\pi e}$. For $r<1/2$, both
$|\log\sqrt{1-r}|$ and $|\log\sqrt{1+r}|$ are at most $r$.  Markov's inequality gives $P(E_r^c)\le e_V/r^2$. Combining these facts with~\eqref{eq:be} proves~\eqref{eq:studentized-decomp} without any independence assumption between the two heads.

For a sequence of mechanisms, bounded $\rho$, $V\ge v_*>0$, $d\to0$, and $e_V\to0$ make the right side converge to zero. An explicit choice that also covers identically zero errors is $t_n=(d_n/V_n)^{1/3}+n^{-1/6}$ and $r_n=\min\{1/4,e_{V,n}^{1/3}+n^{-1/6}\}$. Both are positive and tend to zero; eventually $d_n/(V_nt_n^2)\le(d_n/V_n)^{1/3}$ and $e_{V,n}/r_n^2\le e_{V,n}^{1/3}$. The normal distribution has no mass at $\pm z_{1-\alpha/2}$, so the interval coverage converges to $1-\alpha$.
\end{proof}

\begin{corollary}[A population finite-sample deviation bound]\label{cor:deviation}
Under the conditions of Theorem~\ref{thm:inference}, for $t>0$ and $\gamma\in(0,1)$,
\[
 P\left\{|f-\theta|>\sqrt{2Vt/n}+\frac{4Ht}{3n}
                    +\sqrt{\frac d{n\gamma}}\right\}
 \le2e^{-t}+\gamma.
\]
\end{corollary}
\begin{proof}
Since $|\psi_P|\le H+1\le2H$, Bernstein's inequality controls $|T_P-\theta|$ by the first two terms with error at most $2e^{-t}$. Markov's inequality controls $|f-T_P|$ by the third term with error at most $\gamma$. On the intersection of the two events, the triangle inequality gives the conclusion. The union bound requires no independence between the events.
\end{proof}

\smallhead{Integrating over training.}
The main theorem holds on a training event of probability $1-\delta_n$. To deduce unconditional distributional convergence it suffices that $\delta_n\to0$ together with the displayed defect bound. For unconditional \emph{mean squared} efficiency, choose $n\delta_n\to0$: the contribution of the complement is at most $n(H+1)^2\delta_n$. A fixed confidence level for a training-risk theorem must not silently be used to claim unconditional MSE convergence. An atomic training catalogue establishes pointwise claims at those atoms, not local regularity over the full causal model; Theorem~\ref{thm:panels} is the stated route to the latter.

\section{Variance supervision has its own learnability condition}\label{app:variance}
Assume here binary outcomes and $V(P)\ge v_*>0$. Let $\widehat p_s=N_s/n$, let $\widehat m_{as}=Z_{as}/N_{as}$ with value $1/2$ if $N_{as}=0$, and let $\widehat e_s$ be $N_{1s}/N_s$ clipped to $[\epsilon/2,1-\epsilon/2]$ (value $1/2$ if $N_s=0$). Set $\widehat\theta_S=\sum_s\widehat p_s(\widehat m_{1s}-\widehat m_{0s})$ and
\[
 \widehat V_S=\sum_s\widehat p_s\left[
 (\widehat m_{1s}-\widehat m_{0s}-\widehat\theta_S)^2
 +\frac{\widehat m_{1s}(1-\widehat m_{1s})}{\widehat e_s}
 +\frac{\widehat m_{0s}(1-\widehat m_{0s})}{1-\widehat e_s}\right].
\]
Clip this value to $[v_*/2,H^2]$.

\begin{proposition}[A learnable sampling-variance head]\label{prop:variancehead}
There is a finite constant $C=C(K,p_*,\epsilon,v_*)$ such that
\[
 \sup_P\E_{P^n}(\widehat V_S-V(P))^2\le C/n+C e^{-nq_*/8}.
\]
Let $\mathcal G_V$ be a fixed class of $[v_*/2,H^2]$-valued heads, let $N^V_\zeta$ be its uniform $\zeta$-covering number, and suppose $g_0\in\mathcal G_V$ satisfies $\sup_D|g_0(D)-\widehat V_S(D)|\le a_V$. If $\widehat V$ minimizes empirical squared variance-label loss to tolerance $\eta_V$, then with training probability at least $1-\delta$,
\begin{equation}\label{eq:variance-finiteM}
 \E_\Pi(\widehat V-V)^2\le B^V_{M,n}:=
 6\{C/n+Ce^{-nq_*/8}\}+6a_V^2+12H^2\zeta
 +\frac{8H^4\log(2N^V_\zeta/\delta)}M+2\eta_V.
\end{equation}
The same task-quantile, dominated-shift, atomic, and panel routes used in Theorem~\ref{thm:main} turn this average bound into $\E_{P^n}(\widehat V-V)^2\le b_V(P)$. In the panel route the variance head likewise minimizes the maximum mechanismwise empirical variance loss. A mean-head good set of task mass $1-\gamma_\mu$ and a variance-head good set of mass $1-\gamma_V$ intersect in mass at least $1-\gamma_\mu-\gamma_V$; the analogous density-shift and panel statements follow by the same union bound. Then the last term in~\eqref{eq:studentized-decomp} obeys $e_V\le b_V(P)/v_*^2$. The same conclusion follows from squared log-variance risk when target and output are clipped to this positive compact interval.
\end{proposition}
\begin{proof}
The map
\[
\begin{aligned}
 (p,e,m)\mapsto{}&\sum_sp_s(m_{1s}-m_{0s})^2
 -\left\{\sum_sp_s(m_{1s}-m_{0s})\right\}^2\\
 &+\sum_sp_s\left[\frac{m_{1s}(1-m_{1s})}{e_s}
                    +\frac{m_{0s}(1-m_{0s})}{1-e_s}\right]
\end{aligned}
\]
is continuously differentiable on the compact box with $p$ in the simplex, $e_s\in[\epsilon/2,1-\epsilon/2]$, and $m_{as}\in[0,1]$. Its first derivatives are bounded by a finite constant depending only on $K,\epsilon$. The mean value theorem therefore bounds its error by this constant times the $\ell_1$ error of $(\widehat p,\widehat e,\widehat m)$.

The empirical ratios admit a global bound that also covers empty cells. Define $r(u,v)=u/v$ for $v>0$ and $r(0,0)=1/2$. If $0\le u\le v$, $0\le a\le b$, and $b\ge q>0$, then
\[
 |r(u,v)-a/b|\le\frac4q\{|u-a|+|v-b|\}.
\]
For $v\ge q/2$, this follows by subtracting the two ratios. For $v<q/2$, both ratios lie in $[0,1]$ and $|v-b|\ge q/2$, which proves the same bound. Every numerator and denominator here is an empirical average of a $[0,1]$-valued row feature, with mean squared error at most $1/n$. Squaring the display therefore bounds each ratio error by $64/(nq^2)$. Use $q=q_*$ for the arm means and treatment fractions; the latter have population denominator $p_s\ge p_*\ge q_*$. Clipping cannot increase their distance from the true propensity. Also $\E(\widehat p_s-p_s)^2=p_s(1-p_s)/n$. Combining these bounds with $(\sum_{j=1}^{4K}|u_j|)^2\le4K\sum_j u_j^2$ and the preceding Lipschitz bound gives the stronger uniform risk bound $C/n$. Final projection onto an interval containing $V(P)$ cannot enlarge squared error, establishing the stated bound as well.

For~\eqref{eq:variance-finiteM}, choose a $\zeta$-net of $\mathcal G_V$. The squared loss is bounded by $H^4$ and is $2H^2$-Lipschitz in its prediction on the clipped interval. The two net transfers contribute at most $12H^2\zeta$. Moreover, $(u+v)^2\le2u^2+2v^2$ and the first part of the proposition bound the comparator risk by $2a_V^2+2\{C/n+Ce^{-nq_*/8}\}$. Lemma~\ref{lem:erm}, with $20/3\le8$, gives the display after transferring back from the net. Markov's inequality, density shift, nonnegativity at an atom, or the panel union bound gives the stated mechanismwise routes. Finally $V\ge v_*$ implies relative squared error at most $b_V/v_*^2$. On $[v_*/2,H^2]$, both $\log$ and $\exp$ have bounded derivatives, so squared log and level errors control one another up to constants.
\end{proof}

\section{Balanced panels and continuous deployment coverage}\label{app:uniform}
\begin{theorem}[Balanced mechanism panels]\label{thm:panels}
Under the bounded-class, covering, and common-comparator assumptions of Theorem~\ref{thm:main}, let $P_1,\ldots,P_J$ be a catalogue in $\PP_{K,p_*,\epsilon}$, with $m$ independent training tables per mechanism, and minimize the maximum empirical mechanismwise FSP risk to tolerance $\eta$. With probability at least $1-\delta$,
\begin{equation}\label{eq:panel}
 \max_{j\le J}\defect(P_j;\widehat f)\le B^{\max}_{m,n,J}(\xi):=
 6nA_n+6na_n^2+24nH\xi+\frac{32nH^2\log(2JN_\xi/\delta)}m+2n\eta.
\end{equation}
If the catalogue is an $r$-net of $\PP^{\rm bin}_{K,p_*,\epsilon}$ in total variation, with $L_\phi=8/(p_*\epsilon^2)$, then
\begin{equation}\label{eq:cover}
 \sup_{P\in\PP^{\rm bin}_{K,p_*,\epsilon}}\defect(P;\widehat f)
 \le2B^{\max}_{m,n,J}(\xi)+8H^2n^2r+2nL_\phi^2r^2.
\end{equation}
A net exists with $J\le(1+8K/r)^{4K-1}$. For example, $r_n=n^{-3}$, $a_n=o(n^{-1/2})$, $n\xi_n\to0$, $n\eta_n\to0$, and
\[
 m_n\gg n\{\log J_n+\log N_{\xi_n}+\log(1/\delta_n)\}
\]
make the uniform defect vanish.
\end{theorem}

\begin{lemma}[Continuity of the label and product-experiment transfer]\label{lem:transfer}
For binary outcomes, let $P,Q\in\PP^{\rm bin}_{K,p_*,\epsilon}$ and $\TV(P,Q)\le r$. Then
\[
 \sup_{o\in\mathcal O}|\phi_P(o)-\phi_Q(o)|\le L_\phi r,
 \qquad L_\phi=8/(p_*\epsilon^2).
\]
For every $|f|\le H$,
\[
 \defect(P;f)\le2\defect(Q;f)+8H^2n^2r+2nL_\phi^2r^2.
\]
\end{lemma}
\begin{proof}
Total variation bounds the difference of the probability of each event by $r$. Both stratum probabilities are at least $p_*$ and both arm-stratum probabilities are at least $q_*$. For ratios of a numerator $u$ and denominator $v$, with $0\le u\le v$ and analogous $u',v'$, the identity
\[
 \left|\frac uv-\frac{u'}{v'}\right|
 \le\frac{|u-u'|}{v}+\frac{u'}{v'}\frac{|v-v'|}{v}
\]
gives $|e_P-e_Q|\le2r/p_*$ and $|m_{a,P}-m_{a,Q}|\le2r/q_*$. The contrast term in $\phi$ therefore changes by at most $4r/q_*$. For the one active residual term,
\[
 \left|\frac{Y-m_P}{e_P}-\frac{Y-m_Q}{e_Q}\right|
 \le\frac{2r}{q_*\epsilon}+\frac{2r}{p_*\epsilon^2},
\]
with the same calculation for the control arm. Thus the total is at most
$4r/(p_*\epsilon)+4r/(p_*\epsilon^2)\le L_\phi r$. Averaging this pointwise inequality over the rows gives $|T_P(D_n)-T_Q(D_n)|\le L_\phi r$ for every table.

A product coupling, or telescoping the signed product measures, gives $\TV(P^n,Q^n)\le n\TV(P,Q)\le nr$. For a table function $0\le h\le C$, $|\E_{P^n}h-\E_{Q^n}h|\le C\TV(P^n,Q^n)$ by integration of the level sets $\{h>t\}$. Apply this with $h=(f-T_Q)^2\le4H^2$. The squared triangle inequality gives
\[
 \E_{P^n}(f-T_P)^2\le2\E_{P^n}(f-T_Q)^2+2L_\phi^2r^2
 \le2\E_{Q^n}(f-T_Q)^2+8H^2nr+2L_\phi^2r^2.
\]
Multiplication by $n$ proves the claim.
\end{proof}

\begin{proof}[Proof of Theorem~\ref{thm:panels}]
Apply the $\xi$-net reduction from the proof of Theorem~\ref{thm:main} and then the panel version of Lemma~\ref{lem:erm} to losses $(f(D)-T_{P_j}(D))^2\le4H^2$. The common approximation comparator has risk at most $2a_n^2+2A_n$ at every mechanism, by Theorem~\ref{thm:learnability}. The two net transfers and approximate minimization contribute $24H\xi$, while the union bound ranges over $JN_\xi$ mechanism--net pairs. Multiplication by $n$ yields~\eqref{eq:panel}.

For the net claim, a binary observed law is a probability vector on $4K$ atoms. Partition its first $4K-1$ coordinates into intervals of length at most $r/(4K)$. For every box intersecting $\PP_{K,p_*,\epsilon}$, select one law from that intersection. Two probability vectors in the same box differ by at most $r$ in the sum of the first-coordinate absolute differences; their last-coordinate difference is at most that sum. Therefore their total variation is at most $r$. The number of boxes is at most $(1+8K/r)^{4K-1}$ after harmless rounding. All selected laws satisfy the model restrictions because they are selected within the class.

For any $P$, choose its net point $Q=P_j$ and apply Lemma~\ref{lem:transfer}, then take the supremum over $P$. This gives~\eqref{eq:cover}. The asserted convergence follows by substituting $A_n$, observing $nA_n\to0$ for fixed positive $q_*$, and applying the stated conditions to each remaining term. For unconditional MSE conclusions take $\delta=\delta_n=o(n^{-1})$ as explained in Appendix~\ref{app:inference}.
\end{proof}

\smallhead{Local-model interpretation of the panel.}
The net radius in Theorem~\ref{thm:panels} decreases with $n$ and therefore covers the contiguous alternatives required by the binary observed-law local minimax experiment. A fixed separated catalogue instead defines a classification problem that can be exponentially easier. The complete pretraining cost is $M=Jm$: $J$ pays for mechanism resolution and $m$ for replication within each mechanism.

\section{ATE minimax risk and the efficiency constant}\label{app:minimax}
\begin{proof}[Finite-sample part of Theorem~\ref{thm:minimax}]
Fix any admissible $(p_s)_{s=1}^K$, draw $X\sim p$, and use the same conditional law in every stratum: $P(A=1\mid X)=1/2$, $P(Y=1\mid A=0,X)=1/2$, and $P_\pm(Y=1\mid A=1,X)=1/2\pm h$, with $h=1/(8\sqrt n)$. These laws are realized by binary potential outcomes independent of $A$ given $X$ and belong to $\PP_{K,p_*,\epsilon}$; their ATEs are $\theta_\pm=\pm h$. The common ancillary $X$ factor contributes zero likelihood ratio.

Only treated outcomes differ. For $h\le1/4$,
\[
 \KL(P_+,P_-)=h\log\frac{1+2h}{1-2h}\le8h^2.
\]
To check the inequality, for $x\in[0,1/2]$ one has $\log(1+x)\le x$ and $-\log(1-x)\le x/(1-x)$; hence $\log((1+x)/(1-x))\le2x/(1-x)\le4x$, and take $x=2h$. Independence gives $\KL(P_+^n,P_-^n)\le8nh^2=1/8$, so Pinsker's inequality yields $\TV(P_+^n,P_-^n)\le1/4$.

For an arbitrary estimator $\widetilde\theta$, use the test $\varphi=\ind\{\widetilde\theta>0\}$. Under $P_+$, the event $\varphi=0$ implies squared estimation error at least $h^2$; under $P_-$, the event $\varphi=1$ has the same implication. Thus
\[
 \max_{j\in\{+,-\}}\E_j(\widetilde\theta-\theta_j)^2
 \ge\frac{h^2}{2}\{P_+(\varphi=0)+P_-(\varphi=1)\}
 \ge\frac{h^2}{2}(1-\TV(P_+^n,P_-^n))
 \ge\frac3{512n}.
\]
The middle testing inequality follows directly from
$P_+(\varphi=0)+P_-(\varphi=1)=1-(P_+(\varphi=1)-P_-(\varphi=1))$ and the definition of total variation. Appending auxiliary training data or algorithm randomness whose law is the same in both worlds does not change either the likelihood ratio or total variation. This proves the stated pretraining-robust lower bound.
\end{proof}

\begin{lemma}[A self-contained local Bayes information inequality]\label{lem:vantrees}
Let $\mathcal D$ be a finite sample space and let $q_t(d)>0$ be a continuously differentiable family of probability mass functions on a compact parameter interval, with score $\dot\ell_t(d)$ and information $I_n(t)=\E_t\dot\ell_t^2$. Let $w(t)$ be a continuously differentiable prior density on a compact interval, vanish at its endpoints, and have finite $I(w)=\int(w')^2/w$. Let $\vartheta(t)$ be continuously differentiable. Then every estimator with finite integrated squared risk satisfies
\[
 \int\E_t(\widehat\vartheta-\vartheta(t))^2w(t)dt
 \ge\frac{\{\int\vartheta'(t)w(t)dt\}^2}
           {\int I_n(t)w(t)dt+I(w)}.
\]
For the finite-categorical paths used below, all differentiations and integrations are justified by finite sums and compact interior support.
\end{lemma}
\begin{proof}
Under the joint law $w(t)q_t(d)$, the joint score is $S(t,d)=\dot\ell_t(d)+w'(t)/w(t)$. Integration by parts, with boundary term zero because $w$ vanishes, gives
\[
 \E[(\widehat\vartheta(d)-\vartheta(t))S(t,d)]
 =\int\vartheta'(t)w(t)dt.
\]
Indeed $\widehat\vartheta(d)$ does not depend on $t$ when differentiating the joint density, whereas the derivative of $-\vartheta(t)$ is $-\vartheta'(t)$. The model score has conditional mean zero, so its cross product with $w'/w$ integrates to zero, and $\E S^2=\int I_n(t)w(t)dt+I(w)$. Cauchy--Schwarz gives the result. If the integrated risk is infinite, the inequality is immediate. For a finite table sample space the integration by parts is a finite sum. Independent external randomness can be integrated afterward; its density contains no $t$ and adds no score.
\end{proof}

\begin{proof}[Local constant and attainability in Theorem~\ref{thm:minimax}]
Fix an interior binary observed law $P$ with $V=V(P)>0$. Here $\epsilon<e_s<1-\epsilon$ and $0<m_{as}<1$ for every cell; also $p_s>p_*$ when $K\ge2$, whereas $p_1=1$ when $K=1$. These are the full-dimensional interior conditions relative to the observed-law probability simplex; they ensure $q(o)>0$ and leave an open neighborhood of admissible observed laws. When $K\ge2$ and $p_*=1/K$, or when $\epsilon=1/2$, that interior is empty. For $K=1$, the equality $p_1=1$ is imposed by normalization and is preserved by every probability path below. Define $h(o)=\psi_P(o)/V$ and, for sufficiently small $t$,
\[
 q_t(o)=q(o)(1+t h(o)).
\]
The masses sum to one because $\E_Ph=0$, and remain positive for $|t|\norm h_\infty<1$. The functional derivative calculation in Lemma~\ref{lem:score} gives $\theta'(0)=\E_P[\psi_Ph]=1$. The one-observation Fisher information is
\[
 I(t)=\sum_o\frac{q(o)h(o)^2}{1+t h(o)}\longrightarrow\frac1V.
\]
Choose a smooth density $w$ on $[-1,1]$ that vanishes at the endpoints and has finite $I(w)$, for example a normalized $(1-u^2)^2$. For fixed $c>0$, put $w_{n,c}(t)=\sqrt n\,w(\sqrt n t/c)/c$, supported on $[-c/\sqrt n,c/\sqrt n]$. Its information is $nI(w)/c^2$. For all large $n$, this support stays inside the path's interior neighborhood. Applying Lemma~\ref{lem:vantrees} and multiplying by $n$ gives
\[
 n\int\E_t(\widetilde\theta-\theta(P_t))^2w_{n,c}(t)dt
 \ge\frac{\{\int\theta'(t)w_{n,c}(t)dt\}^2}
 {\int I(t)w_{n,c}(t)dt+I(w)/c^2}.
\]
Continuity of $\theta'$ and $I$ on this finite-dimensional interior path makes the right side converge to $1/(1/V+I(w)/c^2)$. A supremum risk over that neighborhood dominates the integrated risk. Taking $\liminf_{n\to\infty}$, then $c\to\infty$, proves the local lower bound $V$. This explicit path suffices for a lower bound over the larger observed-law model; no assumption of regularity of the competing estimator is imposed.

For the upper bound, Theorem~\ref{thm:panels} supplies a sequence with uniform defect tending to zero, on training events whose complements can be chosen to have probability $\delta_n=o(n^{-1})$. Theorem~\ref{thm:main}(iii), uniformly on a shrinking neighborhood of $P$, gives
\[
 \left|\sqrt{n\E_{P_t^n}(\widehat f-\theta(P_t))^2}-\sqrt{V(P_t)}\right|
 \le\sup_{Q}\sqrt{\defect(Q;\widehat f)}\to0.
\]
The finite-dimensional expression~\eqref{eq:Vformula} is continuous, so $V(P_t)\to V(P)$ uniformly for $|t|\le c/\sqrt n$. The bounded output controls the training-failure contribution by $n(H+1)^2\delta_n\to0$. After integrating training randomness, for every fixed $c<\infty$,
\[
 \sup_{|u|\le c}\left|n\E_{P_{u/\sqrt n}^n,\mathrm{train}}
 \{\widehat f_n-\theta(P_{u/\sqrt n})\}^2-V(P)\right|\to0.
\]
Hence the same estimator attains the local constant $V$. On the full binary-outcome finite-stratum class, the same argument gives a uniform $O(n^{-1})$ upper bound. The order is matching on classes that also contain the randomized lower-bound subexperiment (or an equivalent nondegenerate local subexperiment); no minimax lower bound is asserted for arbitrary restricted subclasses such as a singleton with known ATE. The total pretraining cost of the panel construction can be very large; attainability does not imply an optimal rate in that cost.
\end{proof}

\section{Approximate simulator mechanisms and other boundaries}\label{app:robustness}
\begin{proposition}[Teacher error has a visible inferential price]\label{prop:teachererror}
Let $\widetilde T$ be an approximate label and $\widetilde d=n\E_{P^n}(f-\widetilde T)^2$. Suppose the simulator approximations are fixed (or conditionally fixed), $\widetilde m_a\in[0,1]$, and $\widetilde e\in[\epsilon,1-\epsilon]$. Then
\[
 \defect(P;f)\le2\widetilde d+2n\E_{P^n}(\widetilde T-T_P)^2.
\]
If $\widetilde T=\Pn\phi_{\widetilde\eta}$, put $\delta_a=\widetilde m_a-m_a$ and $\Delta e=\widetilde e-e$, with all displayed $L_2$ norms taken under $P_X$. Its exact bias is
\begin{equation}\label{eq:teacherbias}
 B_P=\E_P\left[\Delta e\left\{\frac{\delta_1}{\widetilde e}+\frac{\delta_0}{1-\widetilde e}\right\}\right],
\end{equation}
and
\begin{equation}\label{eq:teacherdefect}
 n\E_{P^n}(\widetilde T-T_P)^2\le nB_P^2+16\epsilon^{-4}
 \{\norm{\delta_0}_2^2+\norm{\delta_1}_2^2+\norm{\Delta e}_2^2\}.
\end{equation}
\end{proposition}

\begin{proof}[Proof of Proposition~\ref{prop:teachererror}]
Write $f-T_P=(f-\widetilde T)+(\widetilde T-T_P)$ and use $(u+v)^2\le2u^2+2v^2$ before taking expectations. This proves the first inequality without any independence requirement.

For the score identity, condition on $X$. The expected treated residual under the true mechanism is
\[
 \E\left[\frac{A(Y-\widetilde m_1)}{\widetilde e}\mid X\right]
 =\frac e{\widetilde e}(m_1-\widetilde m_1),
\]
and the analogous control expectation is $(1-e)(m_0-\widetilde m_0)/(1-\widetilde e)$. Substitute $\widetilde m_a=m_a+\delta_a$ and subtract $m_1-m_0$. Combining coefficients yields exactly
\[
 (\widetilde e-e)\left\{\frac{\delta_1}{\widetilde e}
                         +\frac{\delta_0}{1-\widetilde e}\right\}.
\]
Averaging over $X$ proves~\eqref{eq:teacherbias}. In particular
$|B_P|\le\epsilon^{-1}\norm{\Delta e}_2(\norm{\delta_1}_2+\norm{\delta_0}_2)$ by Cauchy--Schwarz.

Let $W=\phi_{\widetilde\eta}(O)-\phi_P(O)$. The conditionally fixed case means that generator-side randomness is independent of the fresh table $D_n$ at fixed $P$; condition on that randomness throughout. Since $\widetilde\eta$ is then a fixed generator approximation for the mechanism, the rowwise $W_i$ are iid and
\[
 n\E_{P^n}(\widetilde T-T_P)^2=nB_P^2+\Var_P(W).
\]
Using bounded outcomes, bounded outcome regressions, and propensities in the stated interval,
\[
 |W|\le(1+\epsilon^{-1})(|\delta_0|+|\delta_1|)+\epsilon^{-2}|\Delta e|.
\]
Squaring, using $(u+v)^2\le2u^2+2v^2$ twice, and $\epsilon\le1/2$, gives
\[
 \E_P W^2\le16\epsilon^{-4}
    (\norm{\delta_0}_2^2+\norm{\delta_1}_2^2+\norm{\Delta e}_2^2).
\]
Because $\Var_P(W)\le\E_P W^2$, this proves~\eqref{eq:teacherdefect}. A randomized label approximation must be conditioned on its generator-side randomness or supplied with an additional error bound; the iid step is not valid for arbitrary table-dependent approximate nuisances without further analysis.
\end{proof}

\begin{proposition}[A repeated-mechanism diagnostic for fixed-task bias]\label{prop:paired}
For fixed $P$ and a frozen estimator $f$, let $D_n^{(1)},D_n^{(2)}$ be conditionally independent tables. Put $R_j=f(D_n^{(j)})-T_P(D_n^{(j)})$. Then
\[
 \E_{P^n\otimes P^n}(R_1R_2)=\{\E_{P^n} f-\theta(P)\}^2,
 \qquad
 \E_{P^n} R_1^2=\Var_{P^n}(f-T_P)+\{\E_{P^n} f-\theta(P)\}^2.
\]
\end{proposition}
\begin{proof}
Conditional independence factors the first expectation, and Lemma~\ref{lem:score} makes each residual mean equal to $\E_{P^n} f-\theta(P)$. The second equality is the definition of variance. The sample product can be negative even though its expectation is nonnegative. The diagnostic therefore does not by itself define a nonnegative per-table training loss or control the efficient variance. This is why it is secondary to FSP, rather than a replacement for the fluctuation label.
\end{proof}

\smallhead{Identification remains a maintained boundary.}
A simulator supplies a valid efficient label only relative to its declared identifying model. If real unmeasured confounding breaks exchangeability, the observed-law functional~\eqref{eq:target} can remain estimable while differing from the causal ATE. None of the label, approximation, or pretraining theorems changes that fact. Likewise, clustered sampling needs a cluster-level efficient experiment; continuous treatments need their own score and support conditions. These extensions are not included in the present guarantees.

\section{Proof of the continuous-confounder extension}\label{app:continuous}
\begin{theorem}[Continuous causal-label learnability]\label{thm:continuous}
Let $X\in[0,1]^d$ have density bounded below by $c_*>0$, retain exchangeability at the original $X$, and let $m_a,e$ be H\"older with exponents $\beta_m,\beta_e\in(0,1]$. For a regular cube partition with $h^{-1}\in\mathbb N$, there is a bounded observable histogram comparator $S_{n,h}$, approximable by table attention, such that
\begin{equation}\label{eq:continuous}
 n\E_{P^n}(S_{n,h}-T_P)^2\le C\left\{
 (nh^d)^{-1}+nh^{2(\beta_m+\beta_e)}+h^{2\beta_m}+h^{2\beta_e}
 +nh^{-d}e^{-c_*\epsilon nh^d/8}\right\}=:L_{n,h}.
\end{equation}
If the architecture class (which may grow with $h^{-d}$) contains $f_{0,h}$ with $\sup_D|f_{0,h}(D)-S_{n,h}(D)|\le a_{n,h}$, Theorem~\ref{thm:main} holds with $6nA_n$ replaced by $6L_{n,h}$ and $a_n$ by $a_{n,h}$. If $\beta_m+\beta_e>d/2$, a bandwidth $h_n\asymp n^{-\alpha}$ with $1/\{2(\beta_m+\beta_e)\}<\alpha<1/d$ makes $L_{n,h_n}\to0$.
\end{theorem}

\begin{proof}[Proof of Theorem~\ref{thm:continuous}]
Use Euclidean H\"older conditions
\[
 |m_a(x)-m_a(x')|\le L_m\norm{x-x'}^{\beta_m},\qquad
 |e(x)-e(x')|\le L_e\norm{x-x'}^{\beta_e}.
\]
Let $B_h(X)$ index the $h$-cubes, and let $p_b=P(B_h=b)$. The density lower bound gives $p_b\ge c_*h^d$. Define the actual \emph{observed-law} bin quantities
\[
 e_b=P(A=1\mid B_h=b),\qquad
 m_{ab}=\E(Y\mid A=a,B_h=b).
\]
Importantly, $m_{ab}$ is not assumed to equal $\E[Y(a)\mid B_h=b]$. Conditional exchangeability at the original $X$ yields
\[
 m_{1b}=\frac{\E[e(X)m_1(X)\mid B_h=b]}{\E[e(X)\mid B_h=b]},\qquad
 m_{0b}=\frac{\E[(1-e(X))m_0(X)\mid B_h=b]}{\E[1-e(X)\mid B_h=b]}.
\]
These are weighted averages of $m_a(X)$ within the bin. They lie between its infimum and supremum there. Writing $\widetilde m_a(x)=m_{a,B_h(x)}$ and $\widetilde e(x)=e_{B_h(x)}$, the bin diameter $\sqrt d h$ gives
\[
 \sup_x|\widetilde m_a(x)-m_a(x)|\le L_md^{\beta_m/2}h^{\beta_m}=:c_mh^{\beta_m},
 \quad \sup_x|\widetilde e(x)-e(x)|\le L_ed^{\beta_e/2}h^{\beta_e}=:c_eh^{\beta_e}.
\]
The binned propensity is still in $[\epsilon,1-\epsilon]$. Put $T_B=\Pn\phi_{\widetilde\eta}$. Proposition~\ref{prop:teachererror}, which is a score identity under the original $X$-identified model, gives
\begin{equation}\label{eq:continuous-labeldiff}
 n\E_{P^n}(T_B-T_P)^2\le
 4\epsilon^{-2}c_e^2c_m^2\,nh^{2(\beta_m+\beta_e)}
 +16\epsilon^{-4}(2c_m^2h^{2\beta_m}+c_e^2h^{2\beta_e}).
\end{equation}
This step explicitly pays for residual confounding within bins.

The purely algebraic finite-cell learnability proof in Appendix~\ref{app:learnability} uses the observed conditional means and probabilities, not the causal interpretation of its coarse functional. Apply it to $(B_h,A,Y)$, with $K=h^{-d}$ and $p_*=c_*h^d$. Let $S_{n,h}$ be~\eqref{eq:Sn} with denominator floor $c_*\epsilon h^d/2$. It is bounded by one and satisfies
\[
 n\E_{P^n}(S_{n,h}-T_B)^2\le
 \frac1{nc_*\epsilon^2h^d}
 +2nh^{-d}(H+1)^2e^{-c_*\epsilon nh^d/8}.
\]
Combining this bound with~\eqref{eq:continuous-labeldiff} by the squared triangle inequality proves~\eqref{eq:continuous}, for a constant depending only on the displayed model quantities. Deterministic bin labels can be fed to the comparator-attention construction of Lemma~\ref{lem:architecture}; no unknown causal parameter enters that row encoder. This deterministic bin encoder is part of the input map, so uniform network approximation is over the encoded rows. Repeating the ERM proof with this comparator yields the stated modification of Theorem~\ref{thm:main}.

For $h_n\asymp n^{-\alpha}$, the cell term vanishes when $\alpha d<1$, and the coarsening term vanishes when $2\alpha(\beta_m+\beta_e)>1$. The remaining powers of $h_n$ then vanish. Because $nh_n^d$ is a positive power of $n$, the exponentially decreasing last term dominates its polynomial factor and also vanishes. The interval of admissible $\alpha$ is nonempty exactly under the stated strict smoothness inequality. At or below its boundary this particular construction gives no efficient root-$n$ claim; no impossibility theorem for every possible estimator is inferred from that fact.
\end{proof}

\smallhead{An analytic example for the numerical diagnostic.}
In the separate continuous-covariate experiment, $X\sim\operatorname{Uniform}[0,1]$, $e(x)=0.15+0.7x$, $m_0(x)=0.12+0.5x$, and $m_1(x)=m_0(x)+0.07$. Thus the true ATE is exactly $0.07$. On a bin of width $h$, $\Cov(e(X),m_a(X)\mid B)=0.7\cdot0.5\,h^2/12$. If $e_b$ is the propensity at its center, then
\[
 m_{1b}-m_{0b}=0.07+\frac{0.7\cdot0.5\,h^2}{12e_b(1-e_b)}.
\]
This follows directly from the two weighted-average identities above. Averaging over the equally probable bins gives the exact population coarsening bias used in the figure. A large bin is biased even with infinitely many rows; too many small bins produce empty-cell instability. This diagnostic isolates the histogram comparator; Appendix~\ref{app:experiments} separately evaluates neural FSP on unbinned continuous rows.

\section{Complete experimental specification}\label{app:experiments}
The evaluation separates three questions: learning the realized efficient fluctuation, converting that response into calibrated inference, and reusing the frozen rule on new populations. Table~\ref{tab:traceability} connects these questions to theorems; the specifications below distinguish the original three-seed summary experiments from the five-seed extensions and the validation-selected aligned comparison.

\subsection{Metrics, aggregation and uncertainty}\label{app:metric-contract}
For a fixed mechanism, write $\widehat\theta_{br}=f_b(D_r)$ for checkpoint $b$ on common table $r$, $T_r=\theta+P_{n,r}\psi_P$, and $B,R$ for the numbers of checkpoints and tables. Point RMSE estimates $\{\E(\widehat\theta-\theta)^2\}^{1/2}$; native defect estimates $n\E(\widehat\theta-T)^2$. The seven-target experiment reports the through-origin fluctuation coefficient
\[
 \widehat\beta_b=\frac{\sum_{r=1}^R(\widehat\theta_{br}-\theta)(T_r-\theta)}{\sum_{r=1}^R(T_r-\theta)^2}.
\]
The continuous-neural experiment instead uses an ordinary least-squares slope with an intercept, centering both predictions and teacher labels at their empirical means. Both quantify response to the efficient fluctuation; their finite-sample centering conventions differ.

Throughout the experiments, \emph{latent-effect supervision} means the table-invariant target $T_{0,P}(D_n)=\theta(P)$, not a latent-variable model. ``Raw latent-effect,'' ``Latent-PFN,'' and ``Neural latent-effect'' apply this same $\lambda=0$ target to the raw-row/column, original summary, and continuous-row architectures, respectively; ``latent-effect'' is the shortened panel label when space is limited.

For the five-seed extension comparisons, including Figure~\ref{fig:main-mechanism}a--c, the reported RMSE is the mean of checkpoint-specific RMSEs, $B^{-1}\sum_b\{R^{-1}\sum_r(\widehat\theta_{br}-\theta)^2\}^{1/2}$. The randomized-data ensemble first averages predictions, $\bar\theta_r=B^{-1}\sum_b\widehat\theta_{br}$, then computes RMSE over $r$. Averaging squared losses first and taking one square root defines a third, pooled-loss RMSE. The original three-seed summary table (Table~\ref{tab:main}) uses pooled-loss RMSE. These averaging orders are kept separate. Coverage, variance ratios and Kolmogorov distances in the synthetic diagnostics are computed per checkpoint before averaging.

The number of deployment repetitions is experiment-specific: 2,000 in the original summary study and semisynthesis, 1,000 in the seven-target/raw extension, 300 common tables per scenario--$n$ cell in the aligned comparison, 400 per continuous mechanism--length cell, and 1,500 for each randomized-data endpoint and context length. Prediction rows from multiple checkpoints on one table share that table's sampling noise. Across-seed $95\%$ $t$ intervals use the independent training seeds; table-bootstrap intervals condition on the fitted checkpoints. Wilson coverage intervals use the actual number of deployment repetitions in their cell.

\subsection{Paired uncertainty for the main baseline comparison}
The comparison uses the same 300 locked-final tables for five Summary-FSP, five Raw-FSP, five Raw-latent-effect checkpoints and every baseline. We resampled table indices jointly across methods 2,500 times, keeping all checkpoints fixed. Against CausalPFN-S, Summary FSP reduced mean-checkpoint RMSE by $9.65\%$, $9.78\%$ and $13.16\%$ at the prior center, under effect shift and at the overlap boundary; paired intervals were $[6.28,13.02]\%$, $[6.02,13.36]\%$ and $[9.29,16.81]\%$. Its corresponding defect reductions were $86.14\%$, $79.02\%$ and $79.94\%$, with intervals $[83.79,88.37]\%$, $[74.70,82.71]\%$ and $[76.38,83.02]\%$.

The matched raw pair changes only supervision. At the prior center, Raw latent-effect attains lower point RMSE ($.0176$ versus $.0621$) by shrinking toward the training prior, but its response slope is $.214$ rather than Raw FSP's $.955$. Effect shift reverses that RMSE ordering: Raw FSP reduces RMSE by $54.21\%$ ($95\%$ paired interval $[51.03,57.57]\%$) and teacher defect by $98.96\%$ ($[98.81,99.10]\%$), with response slopes $.972$ versus $-.058$. Thus effect shift is the regime that most sharply exposes the target mismatch. Replicate predictions, common table hashes, and bootstrap draws are supplied with the aligned evaluation and plotting scripts.
On the same shifted stream, Raw FSP's RMSE of $.0624$ is $10.19\%$ below CausalPFN-S's $.0695$.

\subsection{Synthetic mechanisms and actual observed rows}
We use four fixed stratum descriptors $z_s\in\{(-1,-1),(-1,1),(1,-1),(1,1)\}$. For every independent training episode,
\[
 (\widetilde p_1,\ldots,\widetilde p_4)\sim\operatorname{Dirichlet}(8,8,8,8),\qquad
 p_s=0.07+0.72\widetilde p_s.
\]
Independently draw $c\sim\operatorname{Uniform}[-2.3,0]$, $\beta\sim N(0,0.4^2I_2)$, and $b\sim N(0,0.2^2)$. The control means are
\[
 m_{0s}=\operatorname{expit}(c+\beta^\top z_s+b z_{s1}z_{s2}).
\]
Draw $\Delta\sim N(0,0.035^2)$ and $\gamma\sim N(0,0.015^2I_2)$, and set
\[
 m_{1s}=\operatorname{clip}_{[0.015,0.985]}(m_{0s}+\Delta+\gamma^\top z_s).
\]
For propensity, draw $c_e\sim N(-0.35,0.5^2)$ and $\beta_e\sim N(0,0.55^2I_2)$, and let
$e_s=\operatorname{clip}_{[0.15,0.85]}\{\operatorname{expit}(c_e+\beta_e^\top z_s)\}$.
Rows are generated independently by $X_i\sim p$, $A_i\mid X_i=s\sim\operatorname{Bernoulli}(e_s)$, and $Y_i\mid X_i=s,A_i=a\sim\operatorname{Bernoulli}(m_{as})$. A compatible potential-outcome model draws Bernoulli potential outcomes independently of assignment given $X$; their unobserved joint dependence is irrelevant to the ATE. No oracle effect or artificial route bias is added after the rows are generated.

The heterogeneous evaluation uses 2,000 fresh mechanisms and one table per mechanism in each cell. Effect shift replaces $\Delta$ with a random sign times $\operatorname{Uniform}[0.13,0.20]$. Weak overlap uses an intercept $N(-1,0.4^2)$ and $1.8\beta_e$, clipped only to $[0.035,0.965]$. The null configuration sets $\Delta=0$ and $\gamma=0$; the final clipping convention is retained. Because $m_0$ can lie outside the treatment-mean clipping range, this is a zero-increment generator before clipping, not an assertion that every clipped task has exactly zero ATE. Actual truth always uses $\sum_sp_s(m_{1s}-m_{0s})$, not the requested increment. This distinction is represented in the raw outputs.

The original summary-backbone fixed-mechanism experiments instead use
\[
 p=(0.20,0.30,0.30,0.20),\quad e=(0.20,0.38,0.62,0.78),\quad
 m_0=(0.12,0.23,0.34,0.48),
\]
with $m_1-m_0=(0.025,0.035,0.015,0.025)$. The large-effect mechanism adds $0.15$ to every contrast, giving exactly $\theta=0.175$. The aligned main comparison uses the overlap boundary $e=(0.15,0.20,0.80,0.85)$; the more extreme $e=(0.04,0.12,0.78,0.95)$ design is retained as an out-of-domain stress test. Each configuration and $n\in\{64,128,256,512,1024\}$ has independent tables, and all methods within a cell receive identical rows.

\subsection{Network inputs, training and absence of deployment correction}
The six coordinates of stratum token $s$ are
\[
 (N_s/n,\ N_{1s}/n,\ Z_{1s}/n,\ Z_{0s}/n,\ \log(n)/6,\ n^{-1/2}).
\]
They are deterministic statistics of observed rows. There is no propensity estimate, outcome-regression estimate, teacher score, or true task label in an input token. The first four entries jointly determine the empirical categorical law; ordering of the four tokens is irrelevant to the aggregation. The network has a $6\to48\to48$ GELU embedding, one four-head Transformer encoder layer with width 48 and feed-forward width 96, no dropout, and a local $54\to64\to32\to1$ GELU readout. A $2\tanh(\cdot/2)$ transformation bounds each local contrast. Weighting by $N_s/n$ gives the final effect. A separate $48\to32\to1$ head reads the detached aggregate representation and predicts log variance.

Training uses AdamW, initial learning rate $0.0012$, cosine decay to $0.00006$, weight decay $10^{-5}$, batch size 256, 60 epochs, and gradient-norm clipping at 2. The scalar loss is mean squared target error plus $0.002$ times squared log-variance-label error, with $V$ floored at $0.03$ before its logarithm. The variance branch does not send gradients into the mean backbone, although both branches share the global optimization implementation. Targets are exactly $\theta$, $T_P$, $(\theta+T_P)/2$, and $T_P+0.025$ for the four recorded variants. There is no externally fitted nuisance model at test time.

For each $M$ and seed, the same generator stream supplies equal counts of the four training lengths: $M$ is the total number of episodes, and each length receives $M/4$. The fixed-$n$ theory applies to a length-conditioned training experiment with its own episode count. An independent 2,000-task, in-prior validation pool selects the epoch with the smallest loss for that method's own mean target. No shifted evaluation task or real observation participates in this selection. Seeds $0,1,2$ are used at all three primary training budgets; blended and biased-label ablations use seed zero and $M=32768$. The original suite contains all 20 checkpoints and epoch histories. For the aligned comparison, the lowest mean own-label validation loss over common seeds $0$--$2$ selected Summary FSP at $M=32768$; seeds $3$--$4$ were trained only after this budget was locked. Test metrics never enter budget or checkpoint selection.

\subsection{Partial fluctuation and the raw row/column backbone}\label{app:lambda-experiment}
The supervision-path study fixes the summary architecture, generator, optimizer and compute, then trains $\lambda\in\{0,.25,.5,.75,.9,.95,1\}$ with $M=8192$, 40 epochs, and seeds $0,\ldots,4$. Evaluation uses 1,000 common tables in every mechanism--$n$ cell. Besides target RMSE and FSP defect, we regress the frozen prediction on the realized efficient-label fluctuation within a fixed mechanism; its slope is the empirical sampling-response coefficient in Figure~\ref{fig:lambda-path}. The replicate-level outputs also report own-label risk, $n$- and $n^2$-scaled risk, fixed-mechanism bias, and sampling variance for every executed $(\lambda,n)$ cell.
At the typical mechanism and $n=256$, checkpoint-mean full-label defect falls from $.82815$ at $\lambda=0$ to $.09496$ at $\lambda=1$, an $8.72$-fold reduction; the five $\lambda=1$ checkpoints have mean native coverage $.952$.

\begin{figure}[t]\centering
\includegraphics[width=.94\linewidth]{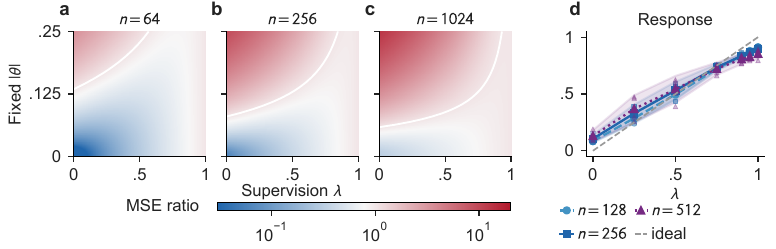}
\caption{\textbf{Changing supervision changes the sampling response.} (a--c) Exact Gaussian MSE relative to FSP ($s=.035$, $v=1$); the white contour marks equal MSE. (d) All 105 seedwise coefficients at the typical mechanism: seven targets, five seeds, three context lengths, $M=8192$. Lines are means; bands are $95\%$ across-seed $t$ intervals. Appendix~\ref{app:metric-contract} defines the coefficient.}\label{fig:lambda-path}
\end{figure}

\begin{figure}[p]\centering
\includegraphics[width=\linewidth]{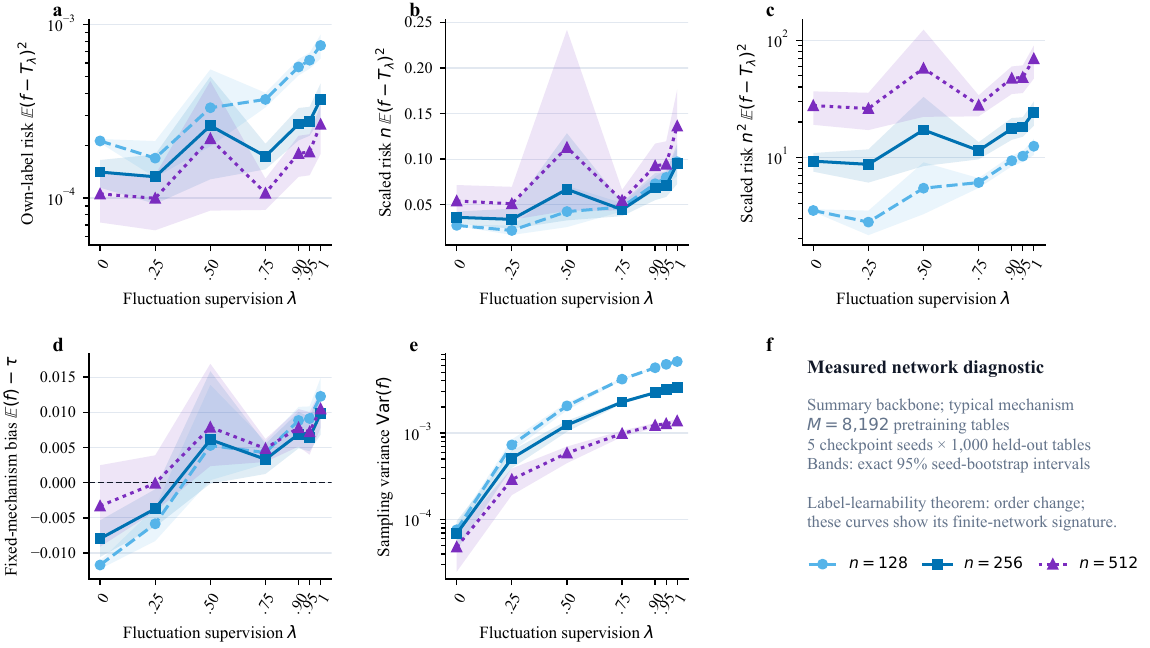}
\caption{\textbf{Seven-target finite-network diagnostics.} Panels a--c report own-label risk and its $n$- and $n^2$-scaled versions; panels d--e report fixed-mechanism bias and sampling variance. Every mark averages five independently trained checkpoints over 1,000 common held-out tables. Intervals are percentile intervals from all $5^5$ ordered seed-bootstrap resamples. These finite-network diagnostics complement the population order transition in Theorem~\ref{thm:learnability}.}\label{fig:lambda-scaling}
\end{figure}

The raw model receives only $n\times4$ arrays $(X_1,X_2,A,Y)$, with the two binary covariates coded in $\{-1,1\}$. A value map and learned role embedding feed a four-column, four-head Transformer inside each row. Eight inducing tokens cross-attend to all row representations, a second four-head block mixes those queries, and a direct raw-row residual is fused with the attention summary. Width is 32, feed-forward width is 64, dropout is zero, and the bounded effect head is $32\to64\to32\to1$. The detached variance head is $32\to32\to1$. This computation is permutation-invariant over rows; an explicit 16-table permutation test changed effects by at most $1.94\times10^{-7}$ and log variances by at most $2.69\times10^{-7}$. The original diagnostic suite uses five Raw FSP and five Raw latent-effect models with $M=8192$ and 25 epochs. The aligned comparison retrains both targets on the same five $M=32768$ table streams and seeds for 40 epochs; architecture, optimizer, batch size and own-label checkpoint selection are otherwise identical. AdamW starts at $1.5\times10^{-3}$, decays to $6\times10^{-5}$, uses weight decay $10^{-5}$, gradient clipping at 2 and batch size 128 on Apple MPS. The first raw architecture collapsed toward a constant at $M=2048$; its checkpoints and outputs are retained as the \texttt{raw\_v1\_pilot} artifacts, and the higher-capacity residual architecture is the declared V2 evaluation.

\subsection{Released baseline protocol}
We evaluate the official \texttt{vdblm/CausalPFN} \texttt{ATEEstimator} at repository commit \texttt{7da4afa3\allowbreak 5affea0f\allowbreak 1fb58893\allowbreak e9f41e87\allowbreak dafdc7eb} with released checkpoint SHA-256 \texttt{4f0f5371\ldots841d}. CausalPFN-S is run on the fresh locked stream of exactly 300 tables in every scenario--$n$ cell. This final stream is disjoint from the former development stream; the earlier S/RA evaluation on the first 300 of 1,000 tables remains a historical artifact and does not enter current figures, tables or headline claims. Single-thread CPU inference avoids a macOS OpenMP collision. The audit reconstructs every aligned cell independently and records its raw-array SHA-256 fingerprint. The official Do-PFN checkpoint~\citep{dopfn}, repository commit \texttt{90d67433\allowbreak b43c4d52\allowbreak d752dc33\allowbreak 6070f525\allowbreak ff856e0b}, is evaluated at $n=256$ on those aligned tables after exact hash joining, by averaging its predicted CATE over observed covariates. Its published implementation specifies continuous numerical covariates and outcomes; our binary table experiment is consequently an out-of-domain transfer stress test, not a like-for-like benchmark. Its frozen weights are reused, and the adapter performs no external outcome normalization or test-time fine-tuning. The full checkpoint hash and per-table predictions are in \path{results/dopfn_run_manifest.json}. CausalFM was pinned at commit \texttt{7808c62}, whose evaluation notebook refers to an unreleased \texttt{best\_model.pth}; the OSPC paper-facing sources expose no runnable checkpoint. The audit records these repository-level outcomes and reports numerical scores only for completed official inference paths.

\subsection{Baselines and interval semantics}
The stratified baseline uses empirical stratum weights and within-arm outcome means, assigning $1/2$ to an empty arm cell. Its smoothed counterpart replaces an arm mean by $(Z_{as}+1/2)/(N_{as}+1)$. Both use the same finite-stratum plug-in sampling-variance formula, with empirical propensities clipped to $[0.025,0.975]$. These are asymptotic plug-in intervals, not exact small-sample coverage statements. The unadjusted contrast uses the ordinary Bernoulli two-group sample-variance formula for its own mean difference. The oracle-label diagnostic uses the true simulator $V(P)$; it is never an implementable baseline.

For the current paired comparison, the classical meta-learners are fit on every table in the same fresh 300-table scenario--$n$ cells used by Figure~\ref{fig:main-mechanism}; Figure~\ref{fig:paired-baseline-profiles} uses their $n=256$ slice. S-, T-, and X-learners follow \citet{kunzel2019}; the DR-learner follows \citet{kennedy2023dr}. Their conditional-effect predictions are averaged over the observed covariates to estimate the ATE. Binary-outcome regressions use four observed-stratum indicators and fixed $L_2$-logistic fits with $C=1$; X- and DR-learner second stages use ridge regression with penalty one. Cross-fitted AIPW/DML uses five treatment-stratified folds and propensities clipped to $[.05,.95]$ \citep{dml}. Every fit uses only its current deployment table, with no simulator quantities or test-mechanism tuning. With an intercept in the DR ridge fit, its empirical-covariate average equals the mean AIPW score, so their ATE entries coincide by construction. The oracle teacher is used solely to score sampling defect and response slope.

Table~\ref{tab:paired-baselines} and Figure~\ref{fig:paired-baseline-profiles} separate point accuracy, efficient-teacher defect and fluctuation response. Summary FSP improves RMSE/defect over CausalPFN-S at the prior center ($9.6\%/86.1\%$), under effect shift ($9.8\%/79.0\%$), and at the overlap boundary ($13.2\%/79.9\%$). The controlled raw comparison is sharper: prior-centered shrinkage gives Raw latent-effect a low in-center RMSE, whereas effect shift raises its RMSE to $.1364$, defect to $5.876$, and flips its response slope to $-.058$; Raw FSP records $.0624$, $.061$, and $.972$. These joint readouts show why point accuracy, teacher fidelity and response must be assessed together.

\begin{table}[t]\centering\small
\caption{\textbf{Paired fixed-mechanism baselines.} All estimates use the same 300 held-out tables in each scenario ($n=256$). Summary FSP, Raw FSP and Raw latent-effect each use five $M=32768$ checkpoints; the raw pair shares architecture, table streams, optimizer and seeds and differs only in the label. CausalPFN-S and Do-PFN each use one released checkpoint, and classical rows refit once per table. Teacher defect $\defect=n\mathbb E(\widehat\theta-T_P)^2$ and the through-origin response slope use the oracle teacher only for evaluation. $^{\dagger}$Do-PFN is tested outside its published continuous-input/output domain. $^{\ddagger}$The DR empirical-$X$ average equals AIPW exactly under this intercept fit.}\label{tab:paired-baselines}
\resizebox{\linewidth}{!}{\begin{tabular}{@{}lrrrrr@{}}\toprule
Method & Typical RMSE & Shifted RMSE & Boundary RMSE & Shifted $\defect$ & Shifted slope \\ \midrule
Summary FSP & 0.0634 & 0.0627 & 0.0696 & 0.042 & 0.985 \\
Raw FSP & 0.0621 & 0.0624 & 0.0650 & 0.061 & 0.972 \\
Raw latent-effect & 0.0176 & 0.1364 & 0.0161 & 5.876 & -0.058 \\
CausalPFN-S & 0.0702 & 0.0695 & 0.0802 & 0.200 & 1.021 \\
S-learner & 0.0576 & 0.0565 & 0.0644 & 0.061 & 0.879 \\
T-learner & 0.0671 & 0.0664 & 0.0776 & 0.157 & 0.988 \\
X-learner & 0.0647 & 0.0642 & 0.0766 & 0.044 & 1.008 \\
DR-learner$^{\ddagger}$ & 0.0664 & 0.0663 & 0.0798 & 0.077 & 1.026 \\
AIPW/DML$^{\ddagger}$ & 0.0664 & 0.0663 & 0.0798 & 0.077 & 1.026 \\
Do-PFN$^{\dagger}$ & 0.2159 & 0.1258 & 0.3092 & 2.276 & 1.391 \\
\bottomrule\end{tabular}
}
\end{table}

\subsection{Measured deployment cost}
The deployment benchmark replays the same 300 fixed-mechanism tables per scenario on an Apple M4 Max CPU with PyTorch and every numerical library restricted to one thread. Quality is the mean of checkpoint-specific RMSEs, not ensemble RMSE. Warm singleton latency includes observed-row preprocessing and all method-specific work: summary aggregation for FSP, nuisance fitting for the classical estimators, and the released CausalPFN \texttt{fit}/\texttt{estimate\_ate} pipeline, including its per-table gradient-boosting weak learner. Table generation, metric calculation and file I/O are excluded. Cold model construction and checkpoint loading are recorded separately; the cumulative curves in Figure~\ref{fig:amortization-evidence}d include one cold load and then process $Q\in\{1,3,10,30,100,300\}$ new tables. FSP batch sizes are fixed at 512 for the summary map and 128 for the raw map; the other methods refit table by table.

Averaging scenario-specific median warm latencies, one deployed Summary-FSP checkpoint takes $.1381$ ms per table, versus $1.6045$ ms for S-learner and $205.9508$ ms for CausalPFN-S: speedups of $11.62\times$ and $1{,}491.27\times$. Raw FSP takes $.6759$ ms on average. At $Q=300$, including one load, median batched totals are $.00832$ s for Summary FSP and $.19677$ s for Raw FSP, versus $.41296$ s for the fastest refitted comparator, T-learner. These are measured wall times. The timing manifest records the environment, checkpoint/source hashes and all repeats.

\begin{figure}[t]\centering
\includegraphics[width=.94\linewidth]{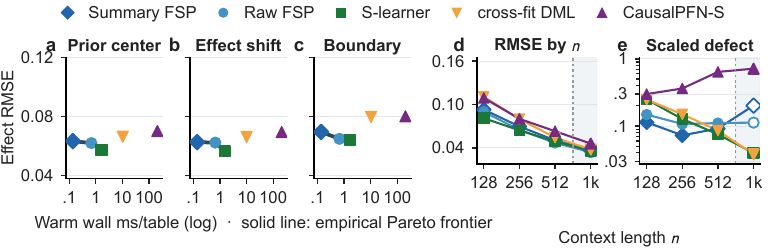}
\caption{\textbf{Statistical quality and deployment cost define an empirical Pareto frontier.} (a--c) Mean-checkpoint RMSE versus warm one-checkpoint wall latency on one CPU thread (300 common $n=256$ tables); ``warm'' excludes one-time model loading but includes preprocessing and all per-table work. Solid charcoal segments connect nondominated methods. (d--e) Effect RMSE and scaled teacher defect $n\E(\widehat\theta-T_P)^2$ on the fixed boundary-overlap mechanism. The factor $n$ makes $O(n^{-1})$ teacher MSE flat on panel e, so monotone decay is not implied. Summary and Raw FSP use $M=32768$; open FSP marks and grey shading denote frozen $n=1024$ extrapolation.}\label{fig:main-mechanism}
\end{figure}

All neural normal intervals use their native learned $\widehat V/n$ and are evaluated as frequentist intervals. Replacing $\widehat V$ by simulator-known $V$ while holding the mean map fixed isolates the effect of scale estimation on coverage. Bias, response coefficients and oracle-studentized Kolmogorov distance then diagnose the mean map. Deployment uses the native learned scale.

In the five-seed $M=8192$ diagnostics, for each architecture, training seed, mechanism, and context length, we also compute the exact empirical Kolmogorov distance between the 1,000 native-studentized outputs and $N(0,1)$, and repeat the calculation after replacing the learned scale by oracle $V$. At the typical mechanism and $n=256$, seed-averaged native/oracle distances are $.081/.082$ for summary FSP and $.352/.345$ for raw FSP. These 90 cell-level values are stored in \texttt{studentized\_kolmogorov.csv}; the QQ curves are visual summaries of the same distributional question, not substitutes for the metric.

\begin{table}[t]\centering\small
\caption{\textbf{Joint inference diagnostics at $n=256$, $M=8192$.} Means over five checkpoints, each evaluated on 1,000 common tables. Coverage and Kolmogorov columns list native/oracle-$V$ values; smaller Kolmogorov distance means closer agreement with $N(0,1)$.}\label{tab:inference-diagnostics}
\begin{tabular}{@{}llrrrrr@{}}\toprule
Mechanism & FSP map & Bias & $\defect$ & $\widehat V/V$ & Coverage & Kolmogorov\\\midrule
Typical & Summary & .0098 & .095 & .906 & .952/.965 & .081/.082\\
Typical & Raw & .0414 & .715 & .882 & .953/.972 & .352/.345\\
Large effect & Summary & $-.0046$ & .115 & .909 & .963/.971 & .064/.071\\
Large effect & Raw & $-.0039$ & .399 & .864 & .993/.995 & .186/.199\\
Weak overlap & Summary & .0388 & 1.444 & .398 & .955/1.000 & .287/.311\\
Weak overlap & Raw & .0828 & 3.286 & .328 & .706/.998 & .591/.514\\\bottomrule
\end{tabular}
\end{table}

\subsection{Monte Carlo reporting conventions}
Bias is averaged over repeated samples of a \emph{fixed} mechanism. Task-mixture averages are called average signed error, not fixed-mechanism bias. Monte Carlo standard errors use the variance of the estimation error, whereas fixed-mechanism sampling variance is reported separately. The original summary-study paired intervals resample 2,000 common table indices 2,500 times after averaging squared losses over its three checkpoints. The released-baseline comparison uses the 300-table protocol specified above. Both quantify deployment-sampling uncertainty conditional on the checkpoints.

Density ridges use Gaussian kernel smoothing solely for display and are peak-normalized, so ridge heights should not be compared as common-scale probability densities. The displayed horizontal quantile ranges are distribution summaries, not confidence intervals for a mean. The main rainclouds show all recorded points in their declared cells. The legacy birthweight panel in Figure~\ref{fig:real} overlays a seeded subset of 160 points on a density and summaries computed from all 2,000 outputs. Heatmaps use common color scales when comparing methods. No interpolation of unrun experimental cells is used.

\subsection{Continuous-covariate diagnostics}
The analytic mechanism in Appendix~\ref{app:continuous} is simulated for $n\in\{64,128,256,512,1024,2048,4096\}$ and $K\in\{2,4,8,16,32,64\}$, with 1,500 independently generated tables per cell. The comparator denominator is floored at $n(0.15)/(2K)$ in count units. We record squared distance to the continuous oracle label, defect, bias, MSE, and exact population coarsening bias.

A second experiment trains directly on unbinned continuous rows. Each token contains four padded covariate coordinates, a four-coordinate dimension mask, treatment, and outcome. Mechanisms cross dimension $d\in\{1,2,4\}$, smooth versus rough nonlinear response surfaces, strong versus weak overlap, and $n\in\{128,256\}$. Three Neural FSP and three Neural latent-effect checkpoints use the same permutation-invariant gated set encoder, $M=4096$ tables per target and seed, and 20 epochs. We evaluate the frozen checkpoints additionally at $n\in\{512,1024\}$ without retraining. The resulting 48 cells contain 400 common tables each: 19,200 unique tables and 307,200 recorded estimator/checkpoint rows, including neural ensembles and oracle-label rows.

The per-table comparators are difference in means, S-, T-, and X-learners, full-fit AIPW, two-fold DR-learner, and two-fold cross-fitted AIPW/DML. Their ridge/logistic nuisance sieve contains the simulator's feature families; penalties $(.35,.60,.35)$ and propensity clipping $[.05,.95]$ are fixed before evaluation. This is a deliberately well-specified classical comparison. At the trained lengths, S-learner is the lowest-RMSE classical method in all 24 cells. Mean-checkpoint FSP wins 16 cells and lowers equal-cell-weight macro RMSE from $.07695$ to $.07154$, a $7.03\%$ reduction. It wins all 12 weak-overlap cells, lowering macro RMSE by $14.56\%$. A 10,000-draw paired bootstrap resamples the same 400 table indices jointly across methods within every cell, computes checkpoint-specific cell RMSE before macro averaging, and conditions on the six fixed checkpoints; the respective intervals are $[6.10,7.97]\%$ and $[13.21,15.90]\%$.

The expanded grid also resolves the boundary. Under weak overlap, FSP macro RMSE remains below S-learner at all four lengths ($.03298$ versus $.03317$ at $n=1024$), although only three of six weak cells favor FSP at that extrapolation length. Across all 12 cells, S-learner overtakes at $n=1024$ ($.03085$ versus $.03239$). At $n=256$, mean response slopes are $.819$ for FSP, $.951$ for S-learner, $.988$ for full-fit AIPW, and $1.008$ for cross-fitted DML. Thus pretraining supplies its clearest point-risk gain in the trained weak-overlap regime and amortizes deployment; per-dataset semiparametric fitting remains closer to the efficient response. Figure~\ref{fig:amortization-evidence} displays all cells and fixed seeds, making the accuracy--reuse boundary directly inspectable.

\begin{figure}[t]\centering
\includegraphics[width=.485\linewidth]{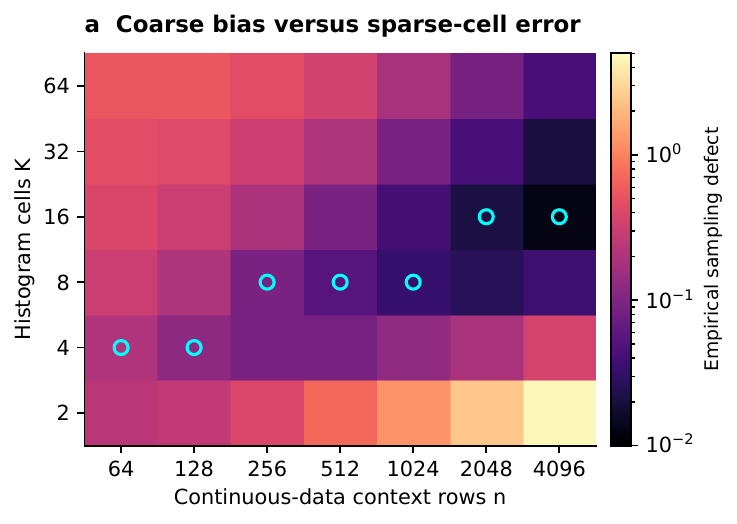}\hfill
\includegraphics[width=.485\linewidth]{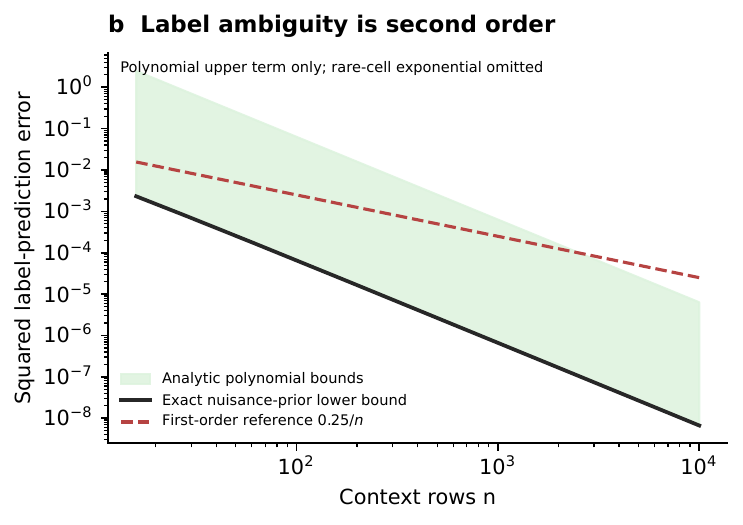}
\caption{\textbf{Two distinct scales, and the cost of forgetting confounding information.} Left: empirical comparator defect on continuous data; cyan rings identify the smallest observed defect at each $n$, without asserting optimality beyond the tested grid. Coarse bins retain confounding bias; excessively fine bins create sparse-cell error. Right: analytic polynomial bounds and the exact Bernoulli nuisance-prior label lower bound. The exponentially small rare-cell term is omitted only in this explicitly labeled scale diagram, not in the theorem. The $n^{-1}$ reference concerns causal parameter estimation, whereas the $n^{-2}$ curve concerns predicting the efficient training label.}\label{fig:continuous}
\end{figure}

\begin{figure}[p]\centering
\includegraphics[width=\linewidth]{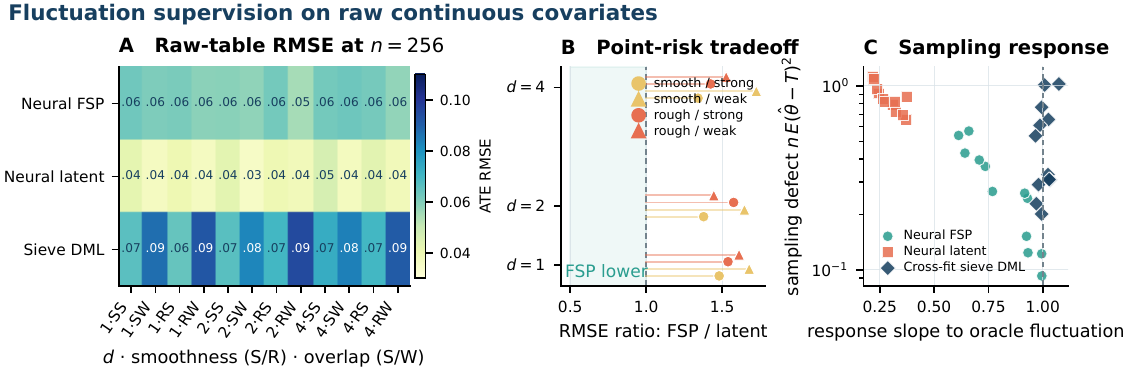}
\caption{\textbf{Neural FSP on unbinned continuous covariates.} The 12 mechanisms use 400 common tables each at $n=256$; neural metrics average three checkpoints. Panel A resolves dimension, smoothness and overlap; B shows the point-risk cost of moving toward the efficient sampling response; C compares native defect with response slope. Neural latent-effect supervision has lower point RMSE, whereas FSP recovers more efficient fluctuation.}\label{fig:continuous-neural}
\end{figure}

\begin{figure}[p]\centering
\includegraphics[width=\linewidth]{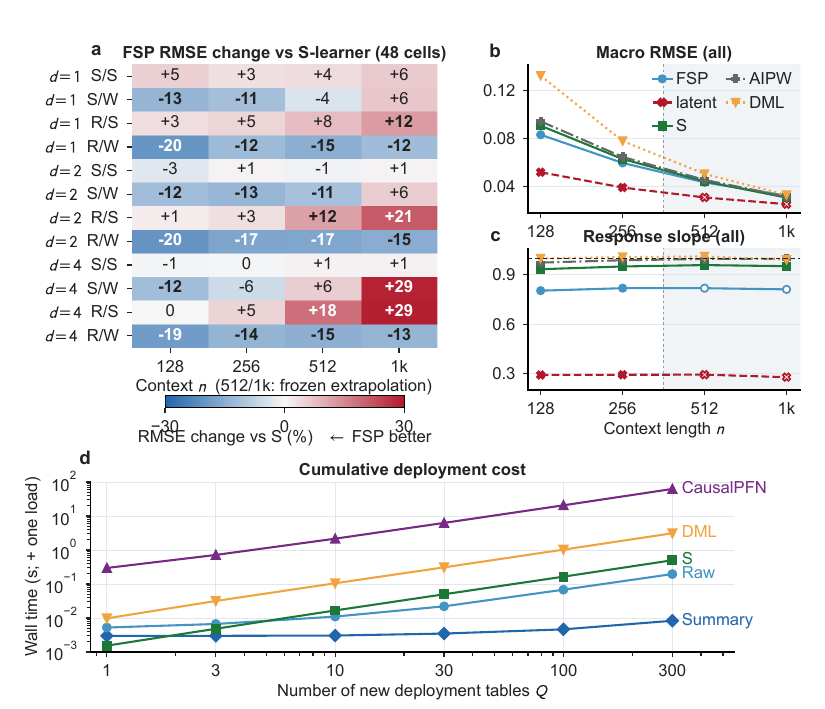}
\caption{\textbf{Complete nonlinear and deployment-cost evidence.} (a) FSP's checkpoint-mean RMSE change relative to S-learner in all 48 continuous cells; blue negative entries favor FSP. Row suffixes denote smooth/rough response (S/R) and strong/weak overlap (S/W). (b--c) Equal-cell-weight macro RMSE and sampling-response slope; the grey region and open neural marks denote frozen extrapolation. (d) Median measured cumulative one-thread CPU wall time, including one cold load; FSP curves use direct batches, whereas comparison curves accumulate observed per-table fits. No panel filters seeds or executed cells.}\label{fig:amortization-evidence}
\end{figure}

\subsection{Finite-dictionary pretraining lower-bound experiment}
We execute the Bernoulli hypercube used in the proof of Theorem~\ref{thm:dictionary-lower}, with $N\in\{4,8,16,32,64,128,256\}$ and $M\in\{32,64,128,256,512,1024\}$. Put $d=\lfloor\log_2N\rfloor$, $\varepsilon^2=\min\{1/4,d/(8M)\}$ and at each repetition draw $\sigma$ uniformly from $\{-1,1\}^d$, holding it fixed across the $M$ episodes. Each episode draws $J$ uniformly from $\{1,\ldots,d\}$ and $Y\in\{-1,1\}$ with $\Pr(Y=1\mid J=j)=(1+\varepsilon\sigma_j)/2$. Exact squared-loss ERM over the $2^d=N$ sign functions is coordinatewise empirical-sign selection. The simulator draws the sufficient counts exactly from their multinomial and conditional binomial laws; exhaustive dictionary enumeration independently checks one repetition in every cell--seed pair. Population excess risk is exactly $4\varepsilon^2d^{-1}\sum_j\ind\{\widehat\sigma_j\ne\sigma_j\}$.

Each of the 42 cells has five seeds and 500 repetitions, totaling 105,000. The log--log slope against $d/M$ is $1.000$ (descriptive regression interval $[.996,1.003]$), with $R^2=.9999$. Mean excess risk divided by the proof's lower bound lies between $5.64$ and $5.96$ across cells. This measures ERM in the theorem's shrinking-signal family: the signal itself changes with $M$, so the observed rate describes local episode-learning difficulty, rather than the learning curve of a fixed neural FSP problem. Raw repetitions and the seed/cell summaries accompany the executable validation script.

The earlier Gaussian dictionary companion uses the same $N\times M$ grid, $a^2=.20\log(N)/M$, five seeds and 500 repetitions. Its 105,000 repetitions yield slope $1.043$ (standard error $.018$, descriptive interval $[1.008,1.078]$) and $R^2=.988$ against $\log N/M$ (Figure~\ref{fig:dictionary-scaling}). It illustrates the same local complexity scale in a different observation model; the Bernoulli experiment above directly implements the proof family.

\begin{figure}[p]\centering
\includegraphics[width=\linewidth]{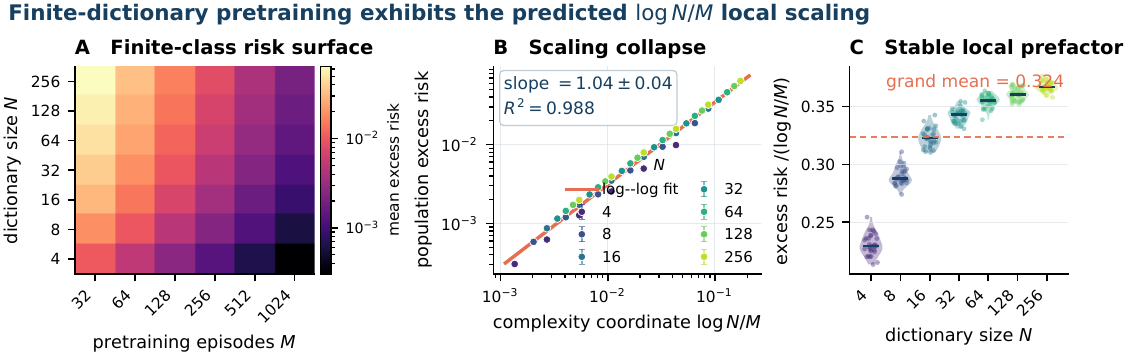}
\caption{\textbf{Gaussian dictionary companion to the Bernoulli proof-family experiment.} The executed $N\times M$ surface (left) collapses against $\log N/M$ with near-unit log--log slope (center). Error bars are $\pm1.96$ standard errors across five simulation seeds. Right: each violin pools 30 seed--budget mean risks per $N$, normalized by $\log N/M$.}\label{fig:dictionary-scaling}
\end{figure}

\subsection{Cattaneo data provenance and two different validation targets}
The raw file was obtained from the installed statsmodels distribution at \path{treatment/tests/results/cataneo2.csv} and copied without modifying its records. The package records its SHA-256 digest and source documentation \citep{cattaneo2010,statsmodelsdata}. Exposure is \texttt{mbsmoke\_}; outcome is \texttt{lbweight}. Strata are $2\ind\{\texttt{mage}\ge25\}+\ind\{\texttt{medu}\ge12\}$. No pregnancy-care mediator is adjusted for in this deliberately restricted example.

For actual-data resampling, the full empirical law fixes a descriptive contrast $0.06208267148538865$. We generate 2,000 bootstrap samples at each $n\in\{128,256,512\}$ with replacement. The benchmark is calculated from the same original records; it is explicitly \emph{not held out}. Its role is to define the target of a conditional empirical-distribution study, not provide ground truth for a scientific causal effect. Models remain frozen throughout.

For semisynthetic validation, only the real stratum proportions are retained. We choose
\[
 e=(0.20,0.12,0.27,0.18),\quad m_0=(0.07,0.04,0.11,0.075),\quad
 m_1=m_0+\Delta,
\]
with $\Delta\in\{0.025,0.075,0.15\}$. New assignment and outcomes are sampled from these mechanisms. True ATE and sampling variance are therefore known and do not depend on the original smoking or birthweight labels. We use 2,000 repetitions at each $n$ and $\Delta$. Figure~\ref{fig:main-validation}c reports the original summary backbone trained at $M=32768$, with three independent seeds; coverage is computed per checkpoint and then averaged. These models remain frozen, with no real-data fine-tuning.

\subsection{Randomized National Supported Work validation}
The Dehejia--Wahba experimental sample contains 185 treated and 260 control observations from the National Supported Work demonstration \citep{lalonde1986,dehejia1999}. The archived Stata file is downloaded from the authors' NBER data page and stored with SHA-256 \texttt{d1bd2680\ldots4e072}. Age and education are thresholded at 25 and 12 to form the two raw covariate columns. Post-treatment employment is $\ind\{\mathrm{RE78}>0\}$; its finite-sample randomized benchmark is the complete-data arm difference $0.1106029$. At each $n\in\{128,256\}$, 1,500 arm-stratified bootstrap samples preserve the experimental arms. The pre-treatment endpoint $\ind\{\mathrm{RE75}>0\}$ has causal effect zero by temporal ordering; each repetition takes a simple random subsample and completely re-randomizes treatment at the trial allocation fraction, eliminating the chance baseline imbalance in the single realized assignment. All five $M=8192$ summary-FSP and raw-FSP extension checkpoints see identical rows, and the design-based difference in means is computed from those rows.

\begin{figure}[p]\centering
\includegraphics[width=\linewidth]{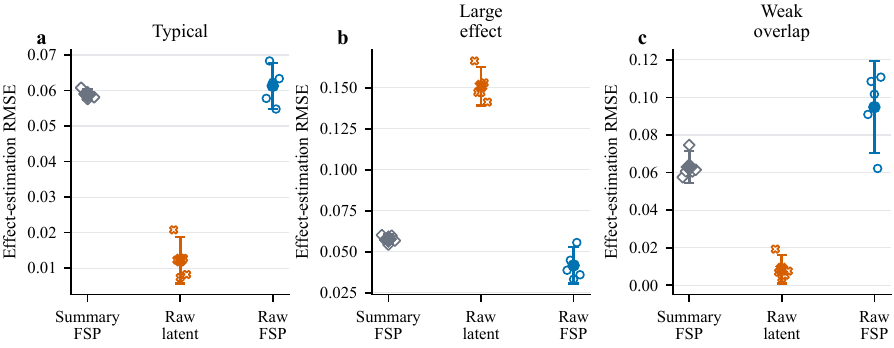}
\caption{\textbf{Seed-level raw-backbone stress test.} RMSE at $n=256$, $M=8192$, using all 1,000 common tables per mechanism; the aligned main comparison instead uses its disjoint locked 300-table stream. Open marks show five training seeds, filled marks their mean, and intervals are $95\%$ across-seed $t$ intervals.}\label{fig:raw-backbone-full}
\end{figure}

\begin{figure}[p]\centering
\includegraphics[width=\linewidth]{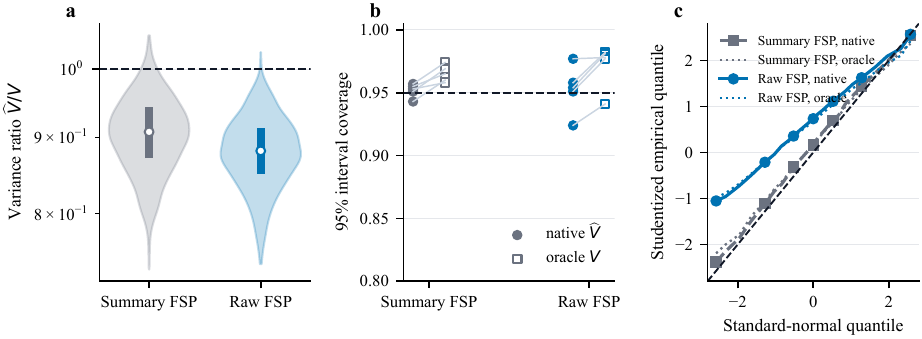}
\caption{\textbf{Mean-head and variance-head diagnostics.} Typical mechanism, $n=256$, 1,000 common tables. (a) Tablewise variance ratios averaged over five checkpoints, with median/interquartile marks. (b) Paired checkpoint-specific native/oracle coverage. (c) Quantiles averaged over checkpoints, rather than quantiles of ensemble predictions. QQ shape distinguishes centering and scale even when marginal coverage is close to nominal.}\label{fig:variance-full}
\end{figure}

\begin{figure}[p]\centering
\includegraphics[width=\linewidth]{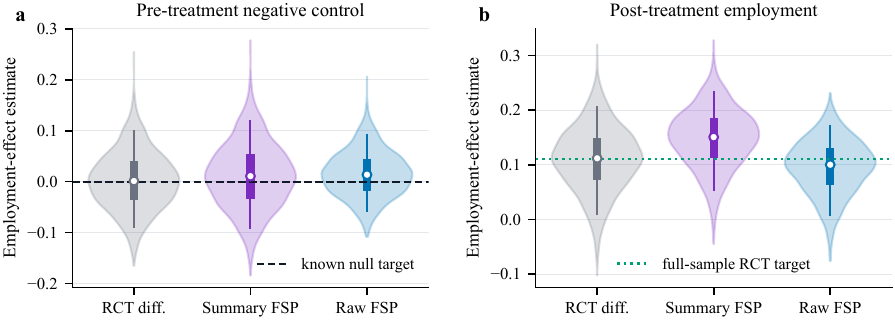}
\caption{\textbf{Full randomized-data distributions.} National Supported Work at $n=256$, with 1,500 resamples for each endpoint: the re-randomized pre-treatment null and the post-treatment empirical benchmark. Neural predictions average five checkpoints. White marks show medians, thick bars interquartile ranges, and thin bars 5th--95th percentile ranges.}\label{fig:nsw-full}
\end{figure}

For the NSW ensemble summaries, predictions are averaged over the five fixed checkpoints within each repetition before computing bias and RMSE. At $n=256$, raw-FSP post-treatment ensemble RMSE is $.0522$, while the mean of its individual-checkpoint RMSEs is $.0578$; the design-based difference has RMSE $.0596$. The ensemble therefore has its own explicitly defined estimand-performance summary, alongside the seed-level diagnostics.

\subsection{Second randomized benchmark: social-pressure turnout trial}
We independently evaluate the frozen checkpoints on the Gerber--Green--Larimer household-randomized social-pressure field experiment \citep{gerber2008,gerberdata}. The pinned Yale Dataverse version contains 344,084 records. We compare the neighbors-mailing arm (38,201 individuals in 20,000 households) with the no-mail control (191,243 individuals in 99,999 households), using turnout in the August 2006 Michigan primary as the binary outcome. The authors' analysis file verifies the arm codes. Two pre-treatment voting indicators, general-election turnout in 2002 and primary-election turnout in 2004, form the raw binary covariates; all four resulting strata contain both arms. The complete-trial observed-arm difference, $0.3779482-0.2966383=0.0813099$, is the declared finite-trial randomized benchmark, not a superpopulation ground truth.

At each $n\in\{128,256\}$, 1,500 common arm-stratified samples are drawn without replacement within each repetition, using the trial allocation fraction rounded to integer arm counts. Every repetition is evaluated by the design-based difference and by all five frozen $M=8192$ summary-FSP and raw-FSP extension checkpoints, with no retraining or checkpoint selection. At $n=128$, the respective design-based and five-checkpoint neural-ensemble means are $.08018$, $.06732$, and $.05218$, with RMSE $.11750$, $.08467$, and $.08901$ against the full-trial difference. At $n=256$, the means are $.08502$, $.07362$, and $.05721$, with RMSE $.07751$, $.05710$, and $.06275$. Thus both learned estimators reduce repeated-subsample RMSE in this design, while their negative biases expose shrinkage that the RMSE comparison does not erase.

The resampling unit here is the individual record within each trial arm. This defines precision around the empirical arm contrast; the original experiment randomized households, so inference under its assignment design would instead require household-level resampling or randomization. We use this experiment for the declared point-risk and bias comparison. The raw-FSP mean individual-checkpoint RMSE at $n=256$ is $.0675$, versus ensemble RMSE $.0628$.

\begin{figure}[p]\centering
\includegraphics[width=\linewidth]{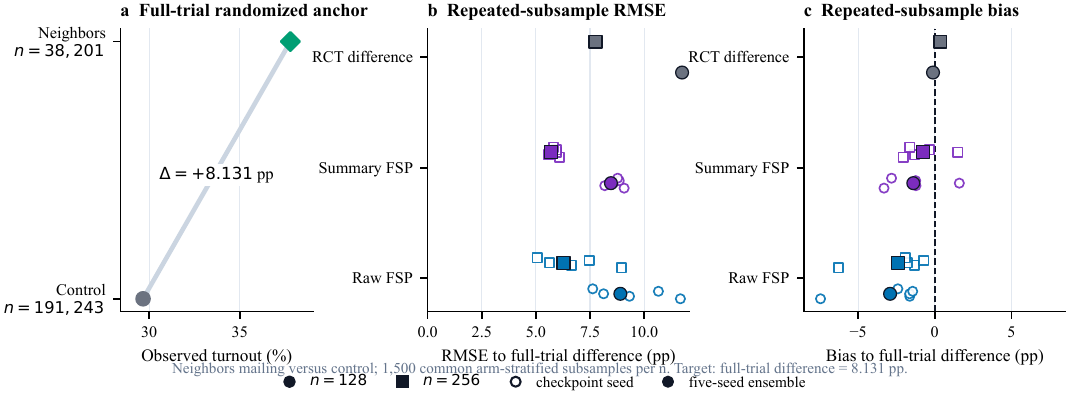}
\caption{\textbf{Independent randomized validation on voter turnout.} (a) Full-trial neighbors-mailing versus control contrast. (b--c) RMSE and bias over 1,500 common arm-stratified subsamples at each context length. Open symbols show five learned checkpoints; filled symbols show their prediction ensemble or the unaveraged design-based estimator. Every error targets the full-trial observed-arm difference.}\label{fig:second-rct}
\end{figure}

\subsection{Teacher and prediction ablations}
The blended-label and shifted-label experiments modify only synthetic target construction. The first measures the continuum between prior-conditioned effect prediction and fluctuation prediction; the second tests the signed bias term in Proposition~\ref{prop:teachererror}. Together with native/oracle variance replacement, they isolate target design and scale learning. Architecture comparisons test transfer of the phenomenon across encoders; they do not identify the necessity of every architectural block. The teacher perturbation study uses a controlled label shift, leaving table-dependent estimated-nuisance teachers as a separate dependence problem.

\begin{figure}[t]\centering
\includegraphics[width=.77\linewidth]{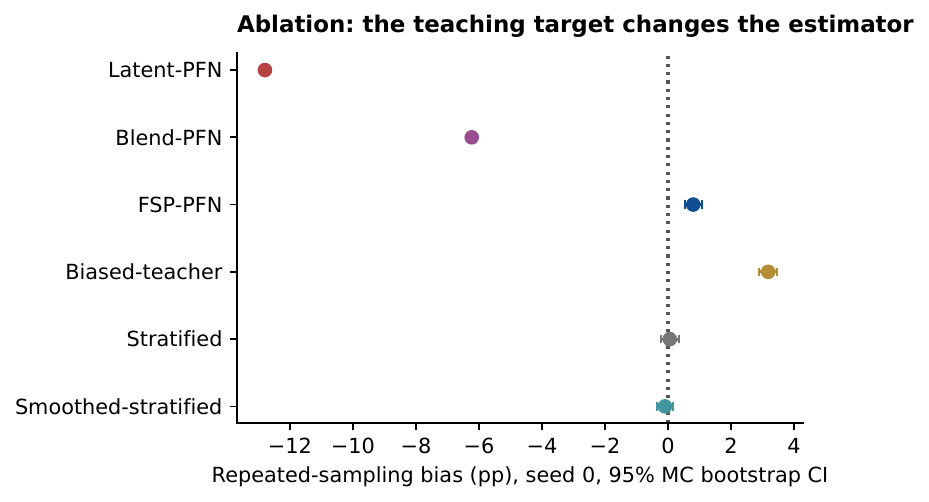}
\caption{\textbf{Changing the target changes the native repeated-sample bias.} Fixed large-effect mechanism, $n=256$, $M=32768$, seed zero, same deployment tables. Intervals bootstrap table indices and quantify Monte Carlo uncertainty of the signed bias. They do not quantify variation across training seeds. The blended target partially retains shrinkage, and the deliberately biased teacher visibly shifts the learned estimator.}\label{fig:ablation}
\end{figure}

\subsection{Complete evidence panels}
\begin{figure}[p]\centering
\includegraphics[width=.485\linewidth]{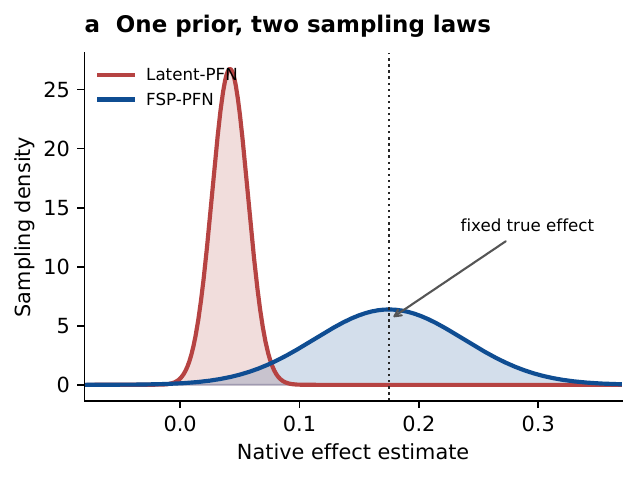}\hfill
\includegraphics[width=.485\linewidth]{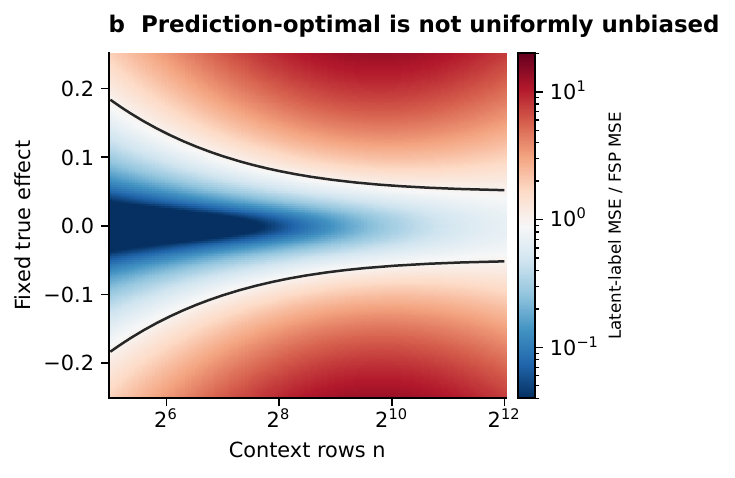}
\caption{\textbf{Exact Gaussian objective map.} Fixed-task sampling distributions and the analytic latent-effect-to-FSP MSE ratio from Proposition~\ref{prop:gauss}; the equality contour separates prior-center shrinkage from effect-shift attenuation.}\label{fig:gaussian}
\end{figure}

\begin{figure}[p]\centering
\includegraphics[width=.485\linewidth]{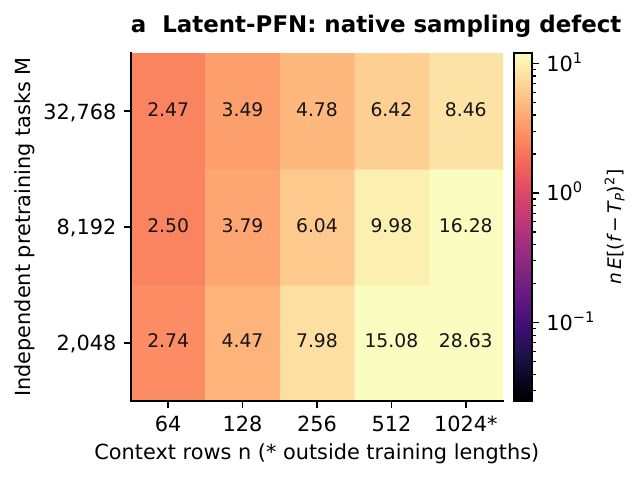}\hfill
\includegraphics[width=.485\linewidth]{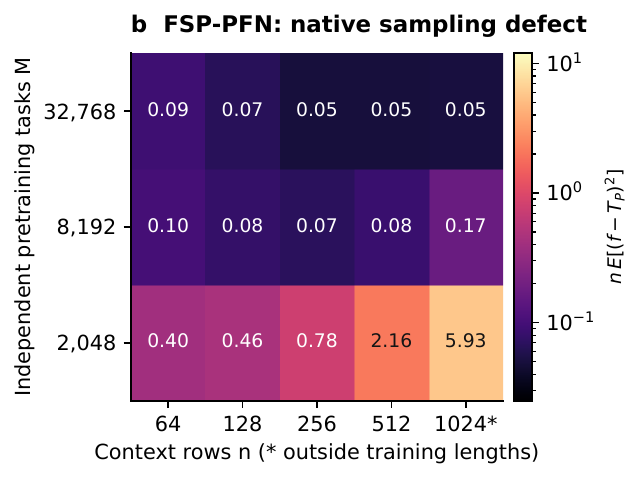}
\caption{\textbf{Deployment rows and pretraining episodes are separate resources.} Same-scale maps of $n\E(f-T_P)^2$ on the large-effect mechanism over the executed $M\times n$ grid. Asterisks identify the extrapolation length $n=1024$.}\label{fig:scaling}
\end{figure}

\begin{figure}[p]\centering
\includegraphics[width=\linewidth]{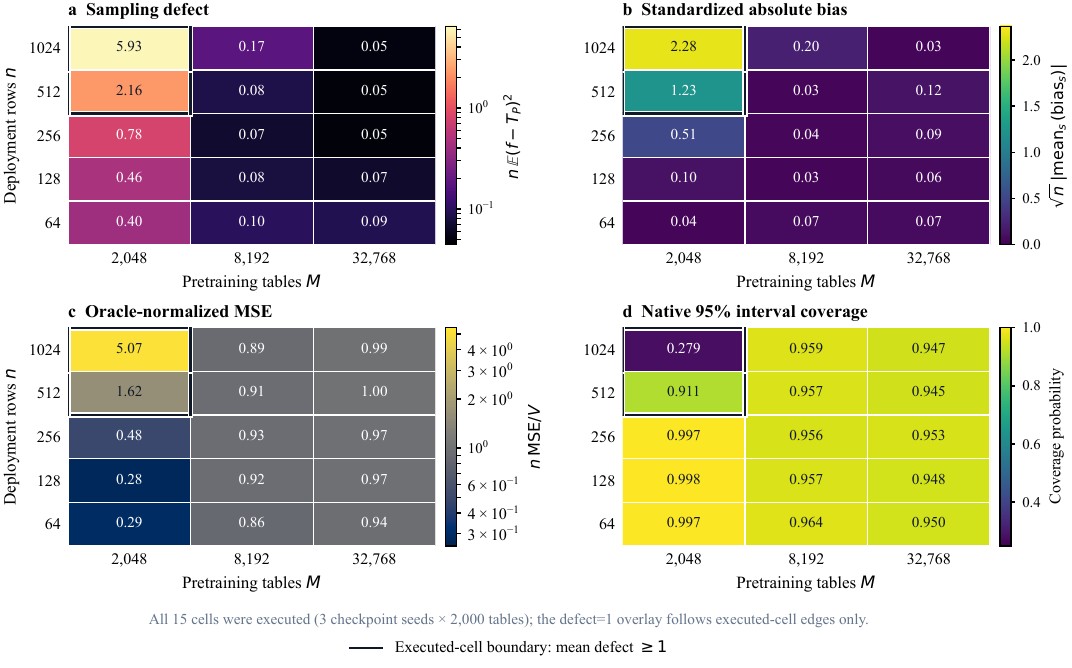}
\caption{\textbf{One executed $M\times n$ grid, four inferential consequences.} Summary-backbone FSP on the fixed large-effect mechanism, with three checkpoint seeds and 2,000 common tables per cell. Panels show native defect, $\sqrt n$ times the absolute seed-averaged bias, oracle-normalized $n\operatorname{MSE}$, and native $95\%$ coverage. The outlined executed-cell region is where mean defect is at least one; every displayed value is an executed cell.}\label{fig:pretraining-deployment-map}
\end{figure}

\begin{figure}[p]\centering
\includegraphics[width=.485\linewidth]{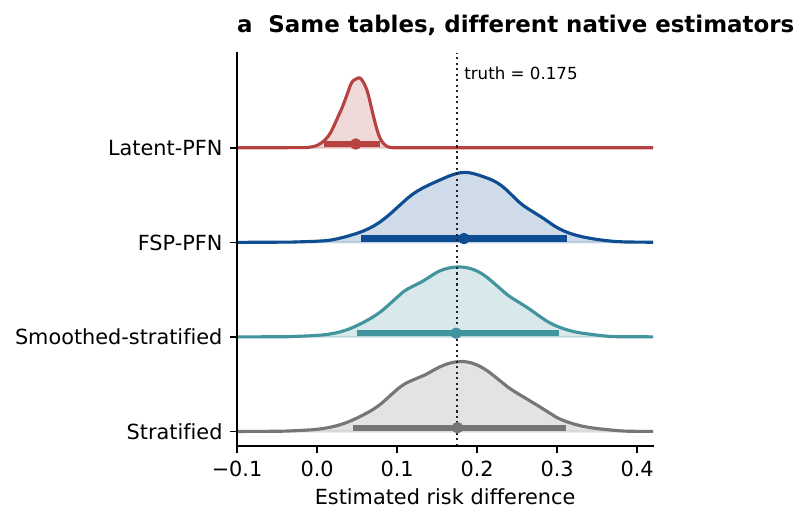}\hfill
\includegraphics[width=.485\linewidth]{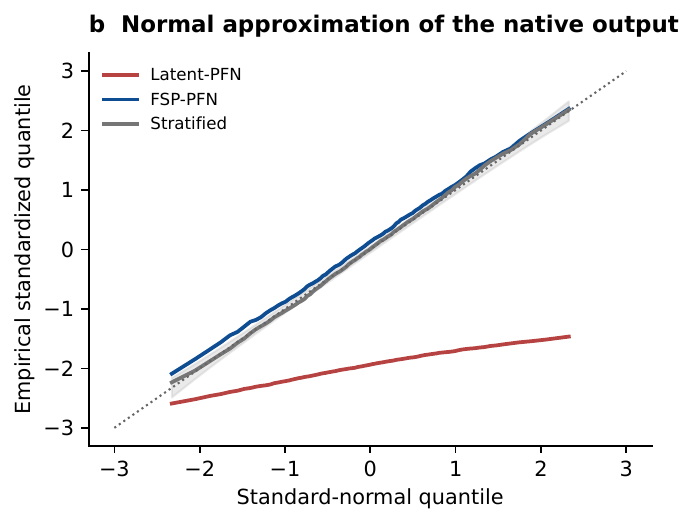}
\caption{\textbf{Repeated-sample shape at a fixed mechanism.} Peak-normalized ridges and QQ diagnostics use the large-effect mechanism, $n=256$, 2,000 common tables, and seed-zero neural models trained at $M=32768$. QQ outputs are standardized with oracle $V$; the grey band is an approximate pointwise $95\%$ normal-quantile envelope.}\label{fig:sampling}
\end{figure}

\begin{figure}[p]\centering
\includegraphics[width=.485\linewidth]{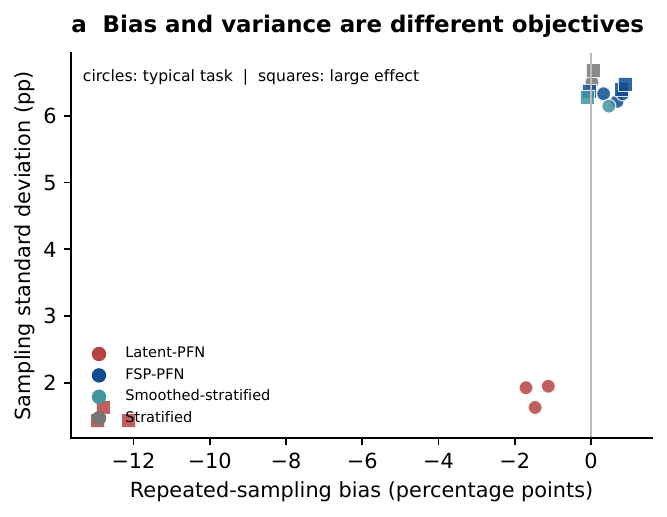}\hfill
\includegraphics[width=.485\linewidth]{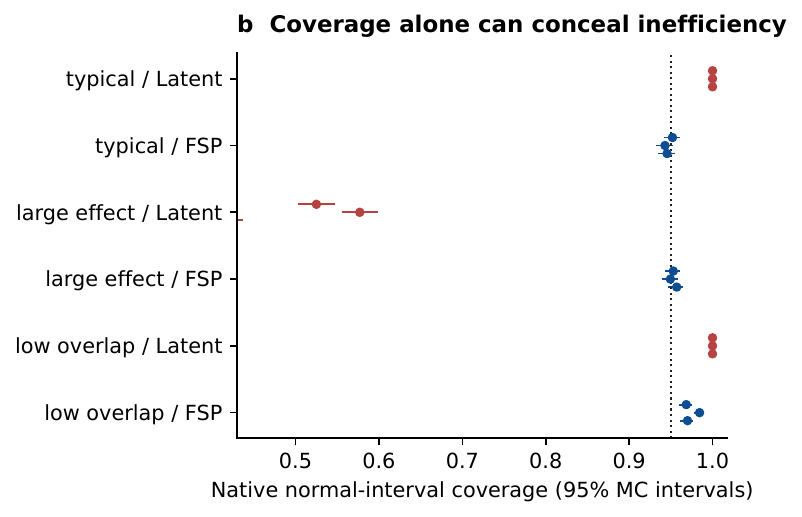}
\caption{\textbf{Bias, sampling spread, and native coverage.} At $n=256$, $M=32768$, learned methods show three checkpoint seeds and baselines one deterministic-model point. Right-panel $95\%$ Wilson intervals use 2,000 deployment repetitions per checkpoint. Centering, spread and coverage assess complementary aspects of inference.}\label{fig:tradeoff}
\end{figure}

\begin{figure}[p]\centering
\includegraphics[width=.77\linewidth]{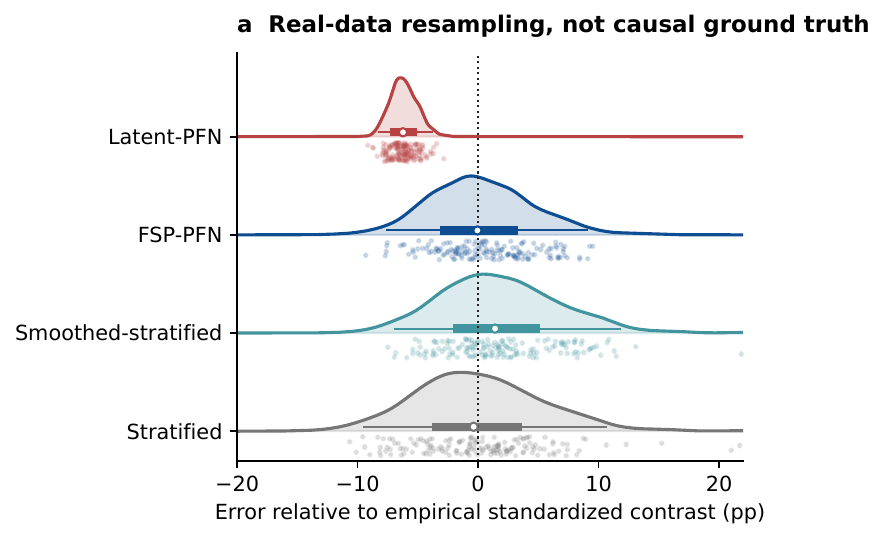}\\[3pt]
\includegraphics[width=.77\linewidth]{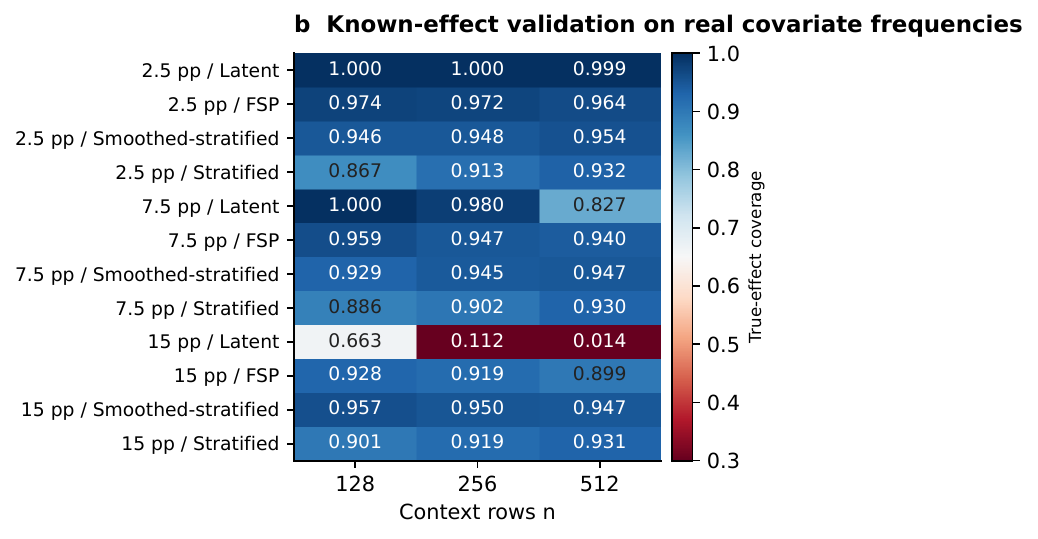}
\caption{\textbf{Observed-data stability and known-effect calibration.} Top: the empirical birthweight contrast, $n=256$, seed zero; 160 displayed dots accompany density and median/interquartile/2.5th--97.5th percentile summaries of all 2,000 bootstrap outputs. Bottom: known-effect semisynthesis uses 2,000 tables per context length and effect, with neural coverage averaged over three checkpoints trained at $M=32768$.}\label{fig:real}
\end{figure}

\begin{table}[p]\centering\small
\caption{\textbf{Accuracy and inferential behavior on the original summary architecture.} All fixed-mechanism rows use $n=256$ and learned methods use $M=32768$. The final block reports inclusion of an empirical descriptive contrast.}\label{tab:main}
\resizebox{\linewidth}{!}{\begin{tabular}{llrrr}
\toprule
Target / setting & Method & Bias & RMSE & Coverage / inclusion \\
\midrule
Fixed large effect & Latent-PFN & -0.1262 & 0.1271 & 0.506 \\
 & FSP-PFN & +0.0056 & 0.0645 & 0.953 \\
 & Smoothed-stratified & -0.0010 & 0.0628 & 0.956 \\
 & Stratified & +0.0006 & 0.0667 & 0.942 \\
\addlinespace
Fixed typical effect & Latent-PFN & -0.0143 & 0.0234 & 1.000 \\
 & FSP-PFN & +0.0062 & 0.0632 & 0.947 \\
 & Smoothed-stratified & +0.0047 & 0.0616 & 0.949 \\
 & Stratified & +0.0003 & 0.0650 & 0.929 \\
\addlinespace
Semi-synthetic, 7.5 pp & Latent-PFN & -0.0640 & 0.0655 & 0.980 \\
 & FSP-PFN & -0.0001 & 0.0488 & 0.947 \\
 & Smoothed-stratified & +0.0224 & 0.0589 & 0.945 \\
 & Stratified & +0.0003 & 0.0592 & 0.902 \\
\addlinespace
Real empirical contrast & Latent-PFN & -0.0557 & 0.0573 & 0.993 \\
 & FSP-PFN & +0.0026 & 0.0444 & 0.957 \\
 & Smoothed-stratified & +0.0178 & 0.0526 & 0.940 \\
 & Stratified & +0.0006 & 0.0526 & 0.911 \\
\bottomrule
\end{tabular}
}
\end{table}

\begin{table}[t]\centering\small
\caption{Theorem--evidence contract. Proof establishes each mathematical statement; the paired experiment probes its observable implication.}\label{tab:traceability}
\begin{tabularx}{\linewidth}{@{}p{.23\linewidth}YY@{}}\toprule
Theory & Implemented check & Diagnostic readout\\\midrule
$\lambda$-phase transition & Seven targets, five seeds, identical generator and compute & Scaled own-label risk, FSP defect, response slope\\
Finite pretraining defect & Summary-backbone $M\times n$ maps and aligned nine-method fixed-mechanism $n$-grid & Bias, RMSE, $n\E(f-T_P)^2$\\
Bias, variance, CLT & Fixed-mechanism repetitions; native output quantiles & Centering, spread, Kolmogorov/QQ shape\\
Studentization & Native/oracle-$V$ intervals and Kolmogorov distances & $\widehat V/V$, coverage, centering and shape\\
Pretraining lower bound & Exact Bernoulli proof family; separate Gaussian dictionary companion & Excess risk versus $\log N/M$\\
Teacher robustness & Blended and shifted labels on common tables & Signed bias response\\
Continuous extension & 48 nonlinear cells; seven fitted estimators; frozen length extrapolation & Point risk, defect, response slope\\
Amortized deployment & Aligned five-method timing on common tables & Warm latency and cumulative wall time\\
Real transfer & Two empirical trial benchmarks; known-effect semisynthesis & Benchmark RMSE/bias, re-randomized null; synthetic coverage\\\bottomrule
\end{tabularx}
\end{table}

\section{Proof, software, and claim audit}\label{app:audit}
\subsection{Mathematical verification ledger}
Every numbered manuscript result has a written proof in the preceding appendices. The proof chain separates classical inputs (Bernstein, Berry--Esseen, Pinsker, total-variation contraction, and conditional-expectation projection) from problem-specific identities and bounds. In particular, the count-product cancellation proves second-order FSP-label learnability; the finite-cover oracle inequality tracks pretraining, approximation, and optimization errors; norm perturbations yield the inferential consequences; Assouad's cube supplies the finite-dictionary $M$ lower bound; and a self-contained van Trees argument supplies the local $n$ lower bound. The continuous and panel extensions state their narrower scopes explicitly.

The primary label identities, Gaussian and $\lambda$-path algebra, count decompositions, beta-prior finite sums, and finite-class constants are independently checked by exact or symbolic computation in \path{code/verify_math.py}. A separate mathematical red-team audit rederived all constants and quantifiers after the final scope edits. These checks complement the analytic proofs and are recorded with commands and outputs.

\subsection{Lean source and verification status}
The \texttt{LeanProofs} project maps every one of the 21 numbered results in this manuscript to compiled declarations in \texttt{coverage\_manifest.json}. The map includes actual Gaussian and causal experiments, the bounded-outcome label minimax problem, finite-episode pretraining, normal approximation and studentization, the explicit quantized network and training schedule, variance learning, uniform panels, the deployment and local information bounds, and the continuous-covariate extension. It also records the original constants and the statistical domains of the formal statements. A separate count of Lean's auxiliary declarations is not a count of paper results.

The checker uses Lean~4.32.0 with Git-pinned Mathlib and the formalized Berry--Esseen theorem in \texttt{ProbabilityApproximation}. It compiles every source module with warnings treated as errors, audits all compiled project declarations (including generated and private ones), checks the upstream normal-approximation theorem, and confirms that the 21 manuscript labels agree exactly with the statement-level map. The transitive axiom audit permits only \texttt{propext}, \texttt{Classical.choice}, and \texttt{Quot.sound}. Source hashes, dependency revisions, commands, and the result of the immutable full-project run are in \path{LeanProofs/verification_status.json}; \path{research_audit/Lean_coverage_20260922.md} explains the correspondence. The earlier Lean~4.19 check is archived separately.

\subsection{Empirical integrity checks}
The original extension suite records 405,000 fixed-mechanism prediction rows: 45 fitted maps evaluated on 9,000 distinct common tables, rather than 405,000 independent deployment samples. Its 90,000 variance-diagnostic rows are derived from the same predictions. The aligned comparison adds 75,600 prediction rows from nine named methods on 3,600 common tables across 12 mechanism--length cells. The NSW and social-pressure evaluations add 66,000 and 33,000 prediction rows. Checks verify all seven $\lambda$ levels and five training seeds, MSE-decomposition closure to $1.39\times10^{-17}$, and raw-row permutation invariance below $2.7\times10^{-7}$. The MSE identity is a software-consistency check; sampling-law accuracy is evaluated separately by bias, variance and the 90 empirical Kolmogorov distances.

The Gaussian dictionary extension contains 105,000 replicate records. The original 76,800-row continuous-neural record is expanded to 307,200 paired estimator/checkpoint rows on 19,200 unique tables; all released cells are numerically replayed before the two frozen extrapolation lengths are added. The revision additionally executes 105,000 repetitions of the exact Bernoulli lower-bound family, 2,500 paired bootstrap resamples of the 300-table large-effect comparison, and 10,000 cell-stratified paired resamples for the continuous comparison. The official CausalPFN evaluation is paired by a shared seed--scenario--length--replicate rule; independent reconstruction records one raw-table SHA-256 fingerprint per scenario--length cell. The turnout benchmark checks treatment-code mapping, the official source MD5, common subsample-index hashes and ten checkpoint hashes.

The compact replay inventory, \texttt{results/artifact\_provenance.json}, identifies the supplied extension records and checkpoints; \texttt{code/audit\_compact\_evidence.py} verifies their hashes, finite values and shared-table counts. Earlier run manifests document the historical experiment suite, including large original-summary replicate files regenerated by the supplied training scripts. The aligned prediction, timing, table and Figure 4/7 pipeline is bound by \texttt{aligned\_pipeline\_audit.json}; its audit rechecks external identities, source arrays, model budgets, hashes and render inputs. New Bernoulli and bootstrap calculations are independently replayable from \texttt{code/validate\_experiment\_contract.py}.

Three accounting corrections were applied before the frozen artifact audit: the semisynthetic oracle-label interval uses oracle variance, the unadjusted contrast uses its own two-group variance, and task-average signed-error uncertainty uses the error variance. The raw V1 constant-prediction failure is retained as a pilot artifact; the reported V2 architecture is independently named and fully specified. Results for external systems are emitted only after a completed official checkpoint call. This rule yields a matched-table CausalPFN comparison and repository-level availability records for CausalFM and OSPC.

\subsection{Scope contract}
The finite-class theorem treats the architecture class as fixed before its concentration sample and exposes optimization tolerance rather than assuming an optimizer certificate. The binary mechanism panel provides the stated uniform route; continuous outcomes retain task-average, dominated-shift, or atomic routes. The continuous-$X$ theorem pays both sparse-cell and within-bin confounding terms. Studentization requires a separately consistent variance head. Real causal interpretation follows the identifying assumptions in Section~\ref{sec:setup}; randomized-study resampling evaluates declared empirical contrasts, the re-randomized NSW negative control has a known null, and semisynthesis supplies known causal effects. The birthweight bootstrap targets its declared observed-data functional. These contracts make each conclusion traceable to a theorem condition and an executable diagnostic.

\end{document}